\documentclass[a4paper,13pt,3p,twoside]{report}
\usepackage{scrextend}
\changefontsizes{13pt}
\usepackage[utf8]{vietnam}
\usepackage[top=2cm, bottom=2cm, left=3.5cm, right=2.5cm]{geometry}
\usepackage{xurl}
\usepackage{appendix}
\usepackage[english]{babel}
\usepackage{xcolor}
\usepackage{outlines}
\usepackage{graphicx} 
\usepackage{fancybox} 
\usepackage{indentfirst} 
\usepackage{amsthm} 
\usepackage{latexsym} 
\usepackage{amsmath} 
\usepackage{amssymb} 
\usepackage{amsbsy} 
\usepackage{times} 
\usepackage{array} 
\usepackage{enumitem} 
\usepackage{subfiles} 
\usepackage{titlesec} 
\usepackage{titletoc}
\usepackage{chngcntr} 
\usepackage{pdflscape} 
\usepackage{afterpage}
\usepackage[ruled,linesnumbered]{algorithm2e}  
\usepackage{capt-of} 
\usepackage{multirow} 
\usepackage{fancyhdr} 

\usepackage[font=small,labelfont=bf]{caption}

\usepackage{listings}
\usepackage{float}
\usepackage{subcaption}

\usepackage[nonumberlist, nopostdot, nogroupskip, acronym]{glossaries}
\usepackage{glossary-superragged}
\setglossarystyle{superraggedheaderborder}
\usepackage{setspace}
\usepackage{parskip}

\usepackage{tocbasic}

\usepackage{blindtext}

\usepackage[backend=bibtex,style=ieee]{biblatex}  

\include{lstlisting} 

\makenoidxglossaries

\newacronym{WSN}{WSN}{Wireless sensor network}

\newacronym{WRSN}{WRSN}{Wireless rechargeable sensor network}

\newacronym{MC}{MC}{Mobile charger}

\newacronym{BS}{BS}{Base station}

\newacronym{WET}{WET}{Wireless energy transfer}
\newacronym{CL}{CL}{Charging location}
\newacronym{MDP}{MDP}{Markov decision process}
\newacronym{AI}{AI}{Artificial Intelligence}
\newacronym{ML}{ML}{Machine Learning}
\newacronym{RL}{RL}{Reinforcement Learning}
\newacronym{DRL}{DRL}{Deep Reinforcement Learning}
\newacronym{DL}{DL}{Deep Learning}
\newacronym{CNN}{CNN}{Convolutional Neural Network}
\newacronym{AMAPPO}{AMAPPO}{Asynchronous Multi-Agent Proximal Policy Optimization algorithm(proposal)}

\newacronym{PPO}{PPO}{Proximal Policy Optimization algorithm}

\newacronym{IPPO}{IPPO}{Independent Proximal Policy Optimization algorithm}

\newacronym{Dec-POSMDP}{Dec-POSMDP}{Decentralized Partially Observable Semi-Markov Decision Process}

\newacronym{Dec-POMDP}{Dec-POMDP}{Decentralized Partially Observable Markov Decision Process}

\newacronym{NLM-CTC}{NLM-CTC}{Network Lifetime Maximization for Connected Target Coverage}

\newacronym{MCs}{MCs}{Mobile Chargers}

\newacronym{MARL}{MARL}{Multi-Agent Reinforcement Learning}
\makeatletter
\def\munderbar#1{\underline{\sbox\tw@{$#1$}\dp\tw@\z@\box\tw@}}
\makeatother

\newtheorem{theorem}{Theorem}[section]
\newtheorem{definition}[theorem]{Definition}

\newtheorem{proposition}[theorem]{Proposition}

\newcommand{\be}{\begin{equation}}
\newcommand{\ee}{\end{equation}}
\newcommand{\bea}{\begin{equation*}\begin{aligned}}
\newcommand{\eea}{\end{aligned}\end{equation*}}

\newcommand{\R}{\mathbb{R}}

\newcommand{\mc}{\mathcal}

\fancypagestyle{plain}{%
\fancyhf{} 
\fancyfoot[RO,RE]{\thepage} 

}

\def \TITLE{GRADUATION THESIS}
\def \AUTHOR{Trần Văn A}

\titleformat{\chapter}[hang]{\centering\bfseries}{CHAPTER \thechapter.\ }{0pt}{}[]

\titleformat 
    {\chapter} 
    [hang] 
    {\centering\bfseries} 
    {CHAPTER \thechapter.\ } 
    {0pt} 
    {} 
    [] 
\titlespacing*{\chapter}{0pt}{-20pt}{20pt}

\titleformat
    {\section} 
    [hang] 
    {\bfseries} 
    {\thechapter.\arabic{section}\ \ \ \ } 
    {0pt} 
    {} 
    [] 
\titlespacing{\section}{0pt}{\parskip}{0.5\parskip}

\titleformat
    {\subsection} 
    [hang] 
    {\bfseries} 
    {\thechapter.\arabic{section}.\arabic{subsection}\ \ \ \ } 
    {0pt} 
    {} 
    [] 
\titlespacing{\subsection}{30pt}{\parskip}{0.5\parskip}

\titleformat
    {\subsubsection} 
    [hang] 
    {\bfseries} 
    {\alph{subsubsection}, \ } 
    {0pt} 
    {} 
    [] 
\titlespacing{\subsubsection}{50pt}{\parskip}{0.5\parskip}

\usepackage{hyperref}
\hypersetup{pdfborder = {0 0 0}}
\hypersetup{pdftitle={\TITLE},
	pdfauthor={\AUTHOR}}
	
\usepackage[all]{hypcap} 

\graphicspath{{figures/}{../figures/}} 

\counterwithin{figure}{chapter} 

\title{\bf \TITLE}
\author{\AUTHOR}

\theoremstyle{definition}

\begin{document}

\pagestyle{empty} 
\begin{titlepage}
\thispagestyle{empty}
\begin{center}

{\textbf{\large{HANOI UNIVERSITY OF SCIENCE AND TECHNOLOGY}}}\\[4cm]

{\textbf{\huge{ GRADUATION THESIS}}}\\[1cm]
{\textbf{\Large{Multi-agent reinforcement learning strategy to maximize the lifetime of Wireless Rechargeable Sensor Networks}}\\[1cm]

{\textbf{\large{NGUYỄN NGỌC BẢO}}}\\
{\large{bao.nn193989@sis.hust.edu.vn}}\\[0.5cm]

{\textbf{\large{Major: Information Technology}}}\\
{\textbf{\large{Specialization: Computer Science}}}\\

\vspace{2cm}
\begin{table}[H]
\centering
\resizebox{\textwidth}{!}{%
\begin{tabular}{ll}
\multicolumn{1}{c}{\textbf{Supervisor:}} & Associate Professor Huynh Thi Thanh Binh \hspace{1.5cm} \underline{\hspace{3cm}}  \\[0.5cm]
  & \multicolumn{1}{r}{Signature}     \\[0.5cm]
\textbf{Department:}                                   & Computer Science                 \\[0.5cm]
\textbf{School:}                                     & School of Information and Communications Technology \\[3cm]
\multicolumn{2}{c}{\textbf{HANOI, 08/2023}}                                            
\end{tabular}%
}
\end{table}}

\end{center}

\end{titlepage}
\newpage
\pagenumbering{gobble} 

\pagenumbering{roman}
\begin{center}
    \Large{\textbf{ACKNOWLEDGMENT}}\\
\end{center}
\vspace{1cm}

I would like to express my heartfelt gratitude to all those who have contributed to the completion of this thesis. This work would not have been possible without the support and assistance of numerous individuals.

Foremost, I extend my deepest appreciation to my supervisor, Assoc. Prof. Huynh Thi Thanh Binh, whose unwavering support, timely reminders, and invaluable advice have been instrumental throughout both the thesis practice and my time in her laboratory. I am also thankful to Prof. Viet Anh Nguyen, Dr Ta Duy Hoang, and Dr. Binh Nguyen for their pressing and pushing me to become a more diligent research practitioner.

I am sincerely grateful to my research team for their technical support and tireless efforts in assisting me whenever needed. In particular, I am immensely indebted to MSc. Tran Thi Huong for her multiple revisions of the thesis and unwavering support and encouragement. Special thanks go to Nguyen Xuan Nam, Vu Quoc Dat, Nguyen Minh Quang, Pham Quan Nguyen Hoang, and Nguyen Minh Hieu for their invaluable assistance in designing images, offering technical suggestions, and providing experimental support for the thesis.

I would also like to express my respect and appreciation to my seniors, Le Van Cuong, Vuong Dinh An, Tran Cong Dao, Nguyen Duc Anh, Nguyen Khuong Duy, Nguyen Trung Hieu, Bui Hong Ngoc, and Dao Minh Quan, for their valuable consultations and guidance from the early stages of my research journey till now. My sincere gratitude extends to all members of MSOLab for their helpful suggestions and support that contributed to the completion of this thesis.

I want to give a big shout-out to my bros Nguyen Phuc Tan, Ta Huu Binh, Nguyen Duc Tam, Nguyen Chi Long, Nguyen Tran Nhat Quoc, and especially Vu Quang Truong. These guys have been there for me not just in research but in all aspects of life. Also, a huge thanks to the lecturers, students, and all the awesome people at Hanoi University of Science and Technology. You've been super encouraging and supportive over the past few years. And of course, big thanks to my family, and friends who've always had my back with their support and love.

\newpage
\pagenumbering{gobble} 
\begin{center}
    \Large{\textbf{ABSTRACT}}\\
\end{center}
\vspace{1cm}
The Wireless Sensor Network (WSN) plays a crucial role in the era of the Internet of Things post-COVID-19. For surveillance applications, it is essential to ensure continuous monitoring and data transmission of critical targets. However, this task presents challenges due to the limited energy capacity of sensor batteries. Recently, Wireless Rechargeable Sensor Networks (WRSN) have emerged, allowing mobile chargers (MCs) to replenish sensor energy using electromagnetic waves. 

Numerous studies have focused on optimizing charging schemes for MCs. However, existing charging algorithms only consider specific network topologies. Applying these works to different network architectures would require starting over. The primary limitations are the inability to reuse previously obtained information and the lack of scalability for large-scale networks. Additionally, since the complexity of the charging problem, these studies only focus on optimizing for a single MC or dividing the original network into smaller, separate regions, with each MC assigned to serve a particular area independently. Consequently, these approaches reduce charging performance when there is no cooperation between MCs.

The thesis proposes a generalized charging framework for multiple mobile chargers to maximize the network's lifetime, ensuring target coverage and connectivity in large-scale WRSNs. Moreover, a multi-point charging model is leveraged to enhance charging efficiency, where the MC can charge multiple sensors simultaneously at each charging location. The thesis proposes an effective Decentralized Partially Observable Semi-Markov Decision Process (Dec-POSMDP) model that promotes cooperation among Mobile Chargers (MCs) and detects optimal charging locations based on real-time network information. Furthermore, the proposal allows reinforcement algorithms to be applied to different networks without requiring extensive retraining. To solve the Dec-POSMDP model, the thesis proposes an Asynchronous Multi-Agent Reinforcement Learning algorithm (AMAPPO) based on the Proximal Policy Optimization algorithm (PPO). Experimental results demonstrate the superiority of AMAPPO over state-of-the-art approaches.
\begin{flushright}
\begin{tabular}{@{}c@{}}
Student\\
\textit{(Signature and full name)}
\end{tabular}
\end{flushright}

\newpage
\pagenumbering{gobble} 
\renewcommand*\contentsname{TABLE OF CONTENTS}
\titlecontents{chapter}
    [0.0cm]             
    {\bfseries\vspace{0.3cm}}                  
    {{\bfseries{\scshape} CHAPTER \thecontentslabel.\ }} 
    {}         
    {\titlerule*[0.3pc]{.}\contentspage}         
    
\titlecontents{section}
    [0.0cm]             
    {\vspace{0.3cm}}                  
    {\thecontentslabel \ } 
    {}         
    {\titlerule*[0.3pc]{.}\contentspage}         
    
\titlecontents{subsection}
    [1.0cm]             
    {\vspace{0.3cm}}                  
    {\thecontentslabel \ } 
    {}         
    {\titlerule*[0.3pc]{.}\contentspage}         

\addtocontents{toc}{\protect\thispagestyle{empty}}
\tableofcontents 
\thispagestyle{empty}
\cleardoublepage

\pagenumbering{roman}
\renewcommand{\listfigurename}{LIST OF FIGURES}
{\let\oldnumberline\numberline
\renewcommand{\numberline}{Figure~\oldnumberline}
\listoffigures} 
\newpage

\renewcommand{\listtablename}{LIST OF TABLES}
{\let\oldnumberline\numberline
\renewcommand{\numberline}{Table~\oldnumberline}
\listoftables}

\glsaddall 
\renewcommand*{\acronymname}{LIST OF ABBREVIATIONS}
\renewcommand*{\entryname}{Abbreviation}
Definition
\begin{center}
    \printnoidxglossaries
\end{center}


\renewcommand\appendixname{APPENDIX}
\renewcommand\appendixpagename{APPENDIX}
\renewcommand\appendixtocname{APPENDIX}

\renewcommand{\figurename}{Figure}
\renewcommand{\tablename}{Table}
\renewcommand{\chaptername}{CHAPTER}


\newpage
\pagenumbering{arabic}

\pagestyle{fancy}
\fancyhf{}
\fancyhead[RE, LO]{\leftmark}
\fancyfoot[RE, LO]{\thepage}

\chapter{INTRODUCTION}
\label{chap:intro}

In recent years, wireless sensor networks (WSNs) have emerged as greate solutions for addressing various large-scale monitoring and tracking problems in diverse sectors. However, these advancements have also raised concerns, particularly regarding the sustainability of network performance. In Chapter~\ref{chap:intro}, a detailed discussion on both the conventional WSN and the new-generation variant, wireless rechargeable sensor network (WRSN), is presented, along with an exploration of their wide-ranging applications. Subsequently, this section will highlight either the target coverage and connectivity problem or the existing attempts made to address it, thereby emphasizing the significant contributions of this thesis.

\section{Wireless sensor network}

Wireless sensor networks (WSNs), which are fundamental components of the Internet of Things (IoT) \cite{khalil2014wireless}, \cite{kocakulak2017overview}, had been recognized as valuable in various real-world applications. Moreover, in recent years, Micro-Electro-Mechanical Systems (MEMS) technology has allowed for the development of smart sensors that are smaller, less expensive, and more energy-efficient than conventional counterparts, while still possessing sufficient computing and processing capacity \cite{faudzi2020application}. This advancement in new-generation sensors has significantly reduced the implementation and maintenance costs of WSNs. Moreover, the ubiquity of wireless technologies in our daily lives has created a favorable environment for the widespread adoption of WSNs. The increasing availability and affordability of wireless connectivity have opened up opportunities for the seamless integration of sensor networks into various sectors, fostering innovative applications and transformative solutions. 

\begin{figure}[H]
    \centering
    \includegraphics[scale=0.4]{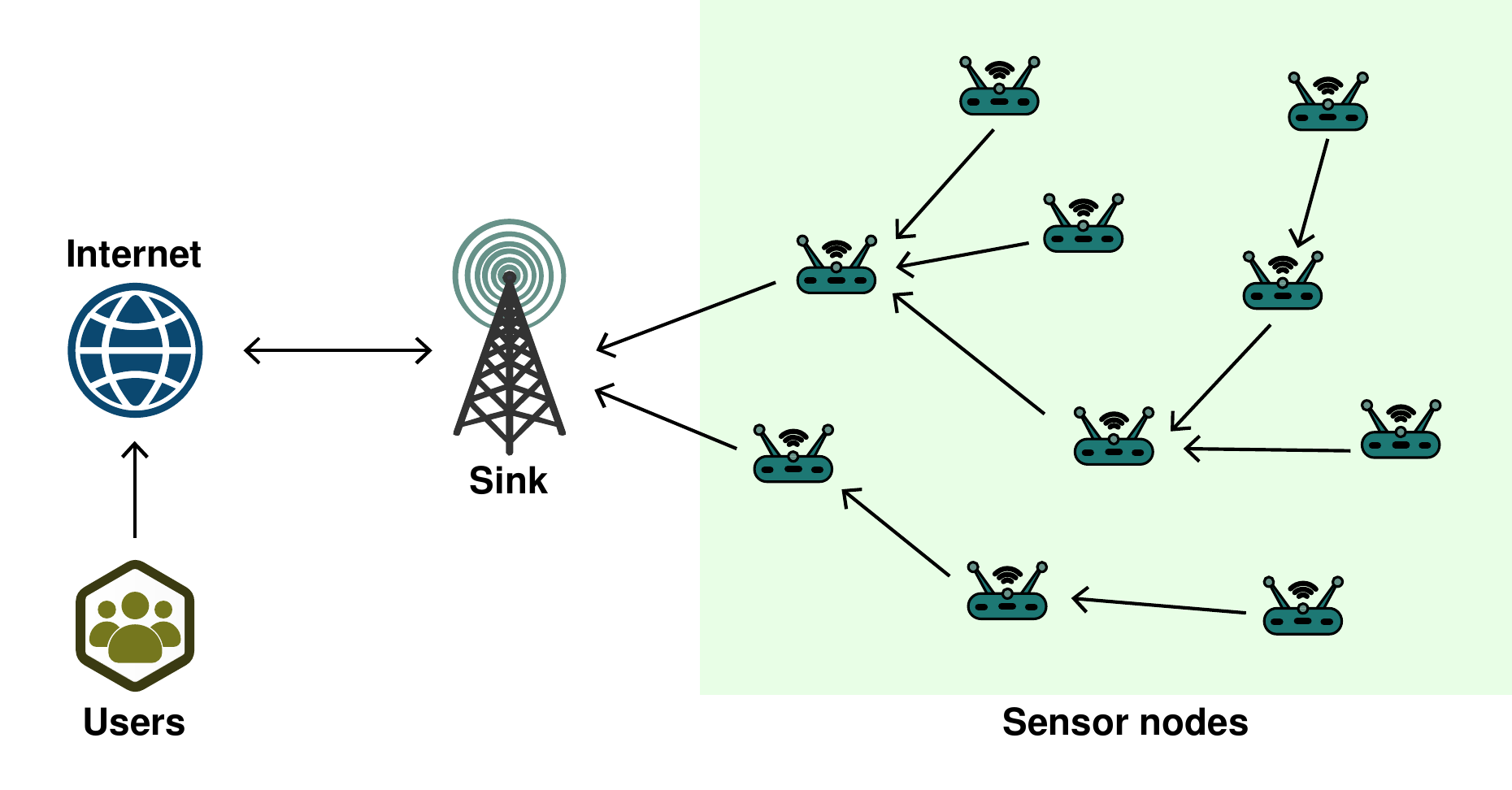}
    \caption{Architecture of a Wireless Sensor Network}
    \label{fig: wsn}
\end{figure}

As illustrated in Fig. \ref{fig: wsn}, a typical Wireless Sensor Network (WSN) consists of a collection of sensors strategically positioned across an area of interest. These sensors collect data about the surrounding environment, which is then relayed through other sensors before reaching a central sink station named Base Station (BS). The data transmission is executed through wireless connections established between sensors. The operation of WSNs necessitates that each sensor performs three essential tasks: sensing, communicating, and computing \cite{fakilidz2002wireless}. These tasks correspond to three integral components within the sensors:

\begin{itemize}
    \item The sensing component is responsible for gathering environmental data such as sound, light intensity, temperature, humidity, images of the surrounding area, and more.
    
    \item The computing unit converts analog signals, such as sound or ground vibration, into digital signals. Additionally, it may handle preprocessing or aggregating of the collected data.
    
    \item The communication unit which utilizes wireless technology plays a vital role in receiving and forwarding data packets to their next destination.

In addition to these three main components, each sensor is equipped with a power unit, typically a battery, which serves as the energy source for the sensor's operation. An illustration of the sensor architecture is presented in Figure~\ref{fig:sn}.
\end{itemize}
\begin{figure} [H]
\centering
\includegraphics[width=0.75\textwidth]{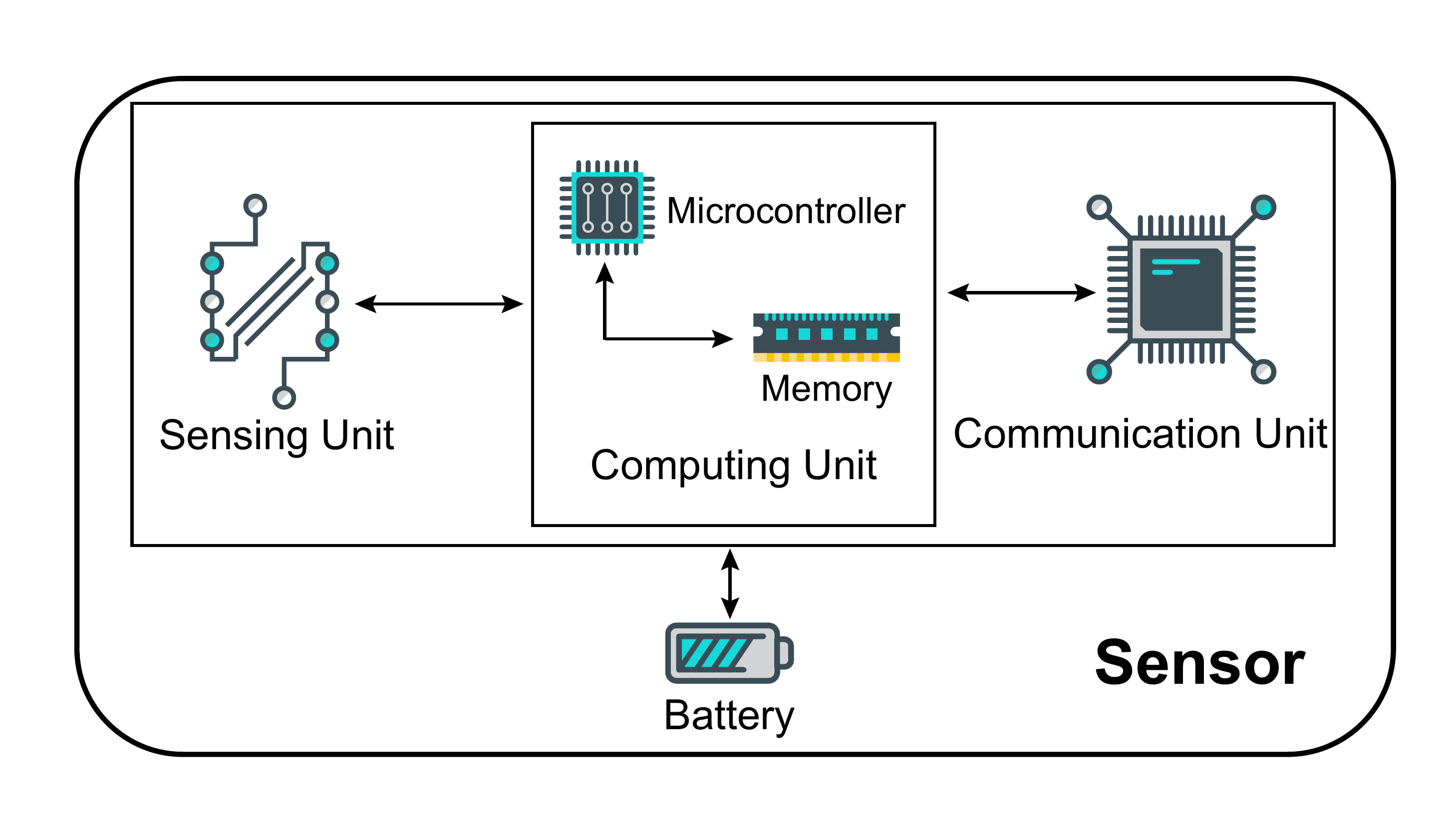}
\caption{Sensor node architecture}
\label{fig:sn}
\end{figure}
At the central station, the collected data serves as a source to facilitate tasks such as forecasting, tracking, monitoring, and other pertinent activities for end-users. The types of sensors and the collected data in a WSN can vary depending on the specific goals and sectors of application. This versatility has enabled WSNs to play a crucial role in diverse real-world domains.

\begin{itemize}
    \item In healthcare applications, Wireless Sensor Networks (WSNs) comprised of sensors of diverse types have the potential to empower hospital staff in effectively tracking the location and monitoring the physiological parameters of a group of patients within hospital environments \cite{lopez2010lobin, kahn1999next}. Consequently, this capability enables doctors and medical experts to actively engage in telemedicine, remote patient monitoring, and teleconsultation.
    
    \item In environmental applications, WSNs can be utilized in governing environmental and climate indicators across a large area. The data collected from temperature, humidity, and vibration sensors in various parts of the area of interest are essential for timely forecasting and safeguarding against natural disasters, including forest fire alarms \cite{chandrakasan1999design}, flood detection \cite{bonnet2000querying}, earthquake prediction 
    \cite{suzuki2007high}, and more.
    
    \item  In smart agriculture, WSNs play a pivotal role in optimizing agricultural practices and boosting productivity \cite{njoroge2018research}. These networks enable real-time monitoring and control of various parameters, facilitating the implementation of precision farming techniques, including automatic irrigation systems \cite{villarrubia2017combining, hamami2020application}, smart livestock monitoring solutions \cite{arshad2022deployment, sharma2018cattle}, and others. Additionally, sensor data obtained from WSNs serve as the primary control signals for efficient resource management systems \cite{gangwar2019conceptual}.

    \item In the civil fields, WSNs possess a fundamental function in the context of smart homes. By integrating embedded sensors, various electronic devices such as light bulbs, washing machines, dishwashers, air conditioners, and more can operate automatically or be remotely controlled based on user preferences \cite{petriu2000sensor, ghayvat2015wsn}.

    \item In commercial applications, WSNs are deployed for office control, inventory control \cite{rabaey2000picoradio}, car theft alarm systems \cite{pottie2000wireless}, and various other purposes.
    
    \item In the military, the applications of WSNs lie in resource management \cite{del2009darma, shah2007distributed}, military monitoring \cite{bekmezci2009energy, ismail2018establishing}, target surveillance \cite{wang2009distributed, biswas2006self}, and more. Specifically, WSNs serve as the backbone for the Command, Control, Communications, Computers, Intelligence, Surveillance, Reconnaissance, and Targeting (C4ISRT) military system \cite{fakilidz2002wireless}.

\end{itemize}

While WSNs hold immense potential for positive impact across various fields, their practical implementation encounters numerous challenges, including security, sensor placement, and network performance. Among these challenges, ensuring sustained network performance over an extended period takes precedence due to its direct influence on user missions. However, this objective is challenging to achieve, primarily due to the limited energy capacity of sensors. As mentioned in the sensor's structure, each sensor relies on a battery as its energy source. However, due to the physical specifications and compact size requirements of sensors, the battery capacity falls short for long-term operation. Once a sensor's battery is depleted, its data collection and transmission tasks come to an abrupt halt. The interruption of even a single sensor or a group of sensors can result in the disconnection of critical subjects being monitored, such as patients in hospitals or strategic military strongholds, potentially leading to severe consequences.

Researchers have proposed three main approaches to address energy depletion in wireless sensor networks (WSNs) \cite{anastasi2009energy, ba2013passive}. The first involves managing the sleep/wake status of sensors, but it has practical limitations as only a small subset of sensors can wake up others \cite{ba2013passive}. The second optimizes data collection from sensors but requires high-performance sensors and does not fully solve the energy depletion problem. The third approach allows sensors to move during operation to enhance energy efficiency but may not be applicable in many real-life scenarios \cite{anastasi2009energy}. Another approach utilizes renewable energy sources, but maintaining a stable source of renewable energy remains challenging due to their uncontrollable nature \cite{adu2018energy}. It was not until the significant advancements in Wireless Energy Transfer (WET) recently that a feasible solution for the energy depletion challenge emerged, giving rise to a new generation of networks known as Wireless Rechargeable Sensor Networks (WRSN).

\section{Wireless rechargeable sensor network}

\begin{figure}[H]
    \centering
    \includegraphics[width=0.7\linewidth]{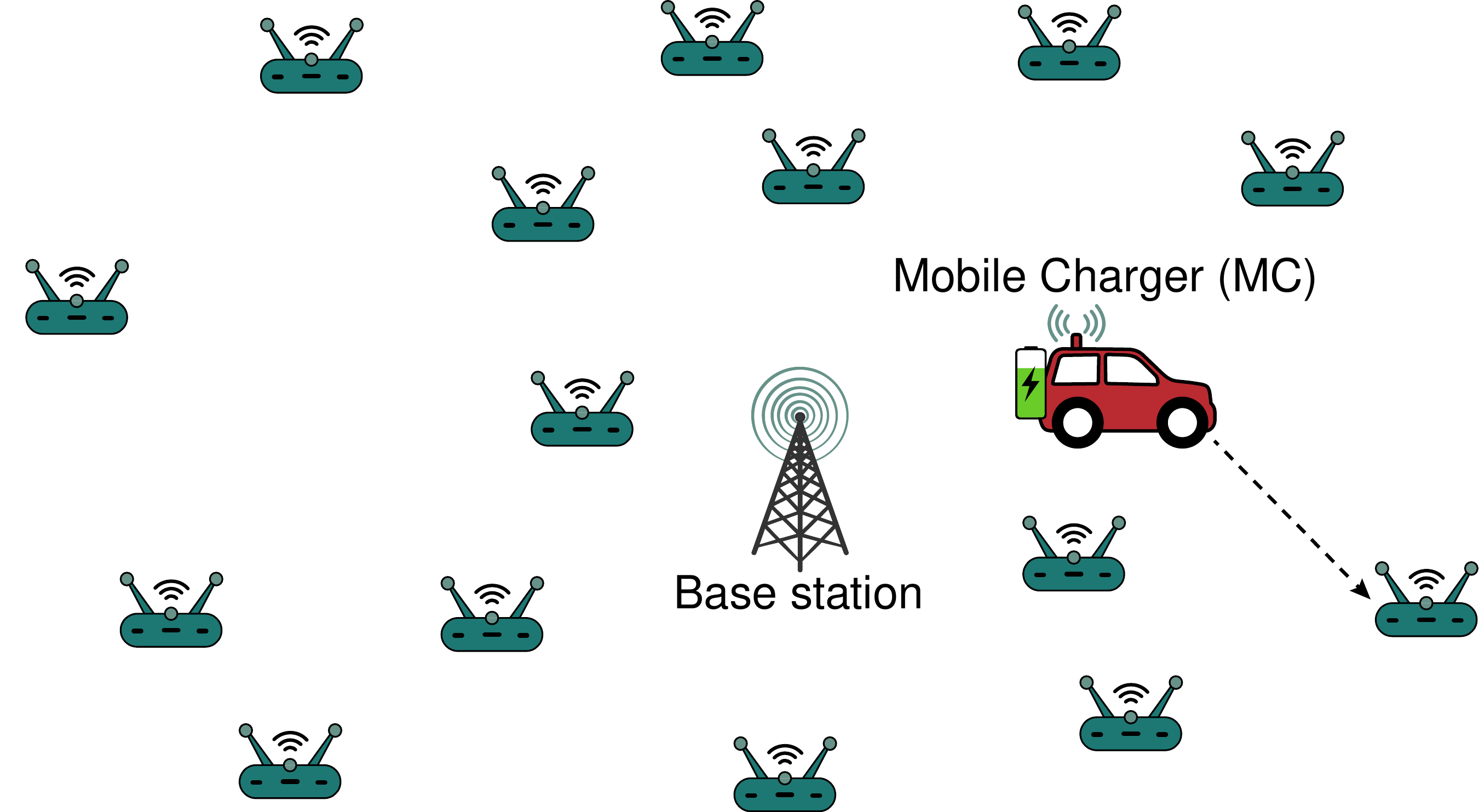}
    \caption{An example of a WRSN}
    \label{fig:wrsn}
\end{figure}

In WRSN, by integrating wireless energy receivers into sensors, it becomes possible to supplement sensors' energy wirelessly. Additionally, autonomous robots known as Mobile Chargers (MCs) are deployed to serve as wireless energy emitters for the sensors. An example of a WRSN is illustrated in Figure~\ref{fig:wrsn}. In a WRSN, the MCs depart from the base station and traverse the network to visit specific locations where they charge the surrounding sensors. The MCs repeat these charging steps continuously, ensuring a consistent supply of energy. Only when an MC's energy is depleted, does it return to the base station for recharging. By implementing a well-designed controller for the MCs, it is possible to effectively mitigate the energy deterioration of the sensors, thereby significantly enhancing the network's performance over an extended period. Consequently, the development of an algorithm for the MCs that directly affect the lifetime of the network has become a prominent research focus, attracting considerable attention from the research community.

\section{NLM-CTC: \textbf{N}etwork \textbf{L}ifetime \textbf{M}aximization for \textbf{C}onnected \textbf{T}arget \textbf{C}overage}

To design an effective algorithm for MCs in a WRSN, it is crucial to consider the specific objective of the network. Several algorithms proposed in the literature, such as those presented in \cite{8882248, en12020287, huong2021effective, huong2022bi}, focus on achieving area coverage as their primary objective. On the other hand, other algorithms aim to optimize data collection, as discussed in \cite{zhao2014framework, guo2014joint}. However, in this thesis, the main concern is ensuring both target coverage and connectivity within the WRSN. In this context, a set of target points is distributed throughout the network's monitored area. The sensors continuously collect information about these targets and transmit it to the base station. It is essential to avoid situations where a group of sensors suddenly shuts down, leading to the disconnection of a target from the base station or the absence of sensors to track the target. Therefore, a well-designed algorithm for MCs should guarantee the presence of at least one operational sensor continuously monitoring each target. Furthermore, this sensor must maintain a viable data transmission path to the base station. A WRSN designed with one MC for the NLM-CTC problem is illustrated in Figure~\ref{fig: target_coverage_problem}. The issue of target coverage and connectivity plays a crucial role in real-world applications, specifically in scenarios encompassing critical network points like strategic military strongholds, patients with severe diseases, nests of endangered animals, and similar contexts. The immense motivation derived from these practical contexts further amplifies the significance of this problem, leading to substantial attention and research efforts.

\begin{figure}[H]
    \centering
    \includegraphics[scale=0.30]{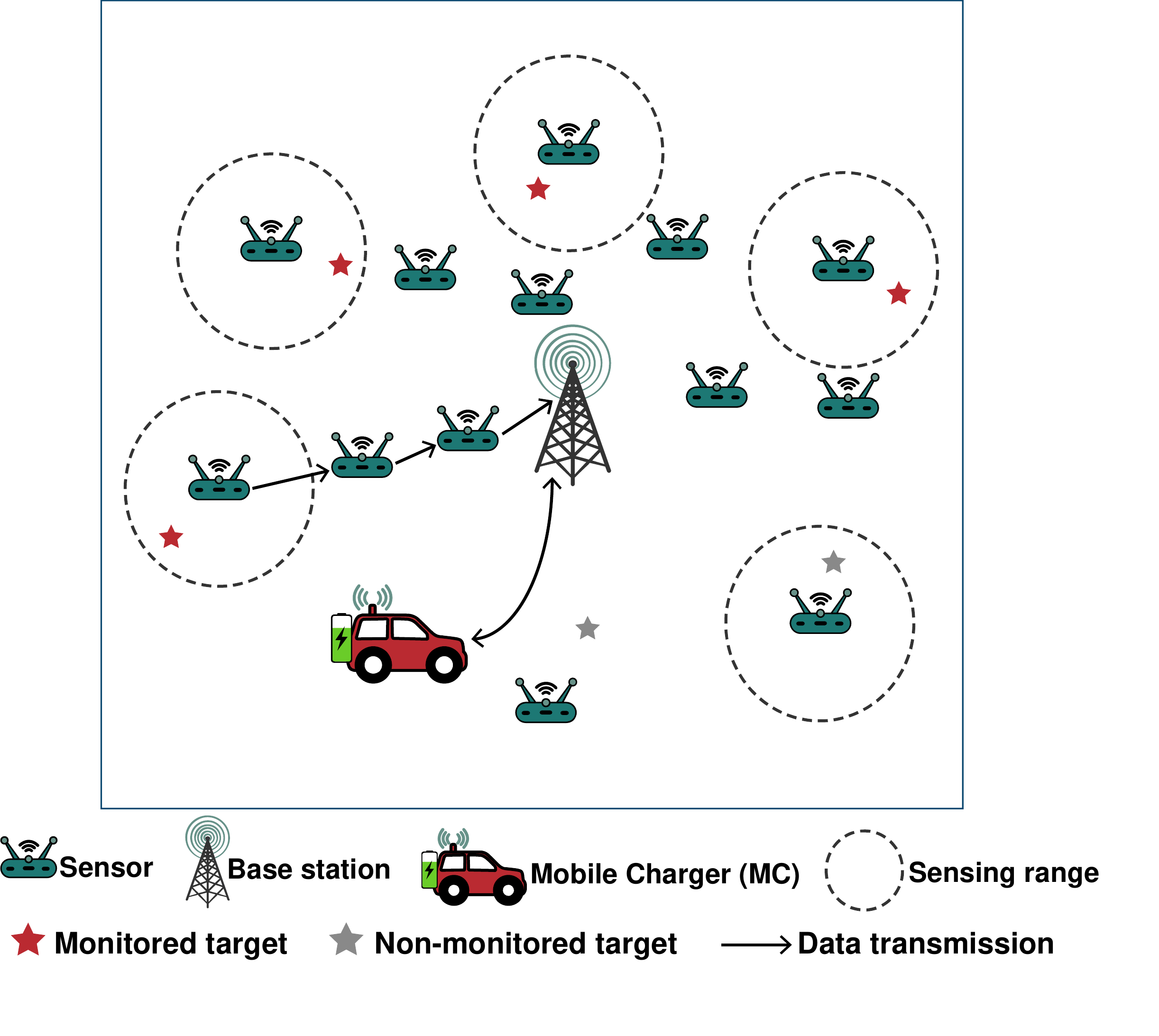}
    \caption{A WRSN with one MC for  NLMCTC problem}
    \label{fig: target_coverage_problem}
\end{figure}

\section{Related work}

Recently, substantial research efforts have been devoted to the development of charging algorithms for MCs. While the fundamental physical specifications of sensors and mobile chargers remain consistent, there are notable differences in the physical models employed in previous works. In this section, we would begin by introducing several related works that explore different physical models, specifically focusing on the distinction between using single MC or multiple MCs, as well as the comparison between multi-node charging and single-node charging. After stating the limitations and applicability of these models, this thesis will choose the most suitable and adaptable one among them. The selected model for this study is utilizing multiple Mobile Chargers (MCs) with a multi-node charging mechanism. Following the discussion of the physical model, the thesis proceeds to analyze two controller strategies: offline and online. Relevant works are examined to evaluate the strengths and weaknesses of each strategy. Based on this analysis, the charging algorithm in this thesis adopts the online strategy, which is more suitable for real-world applications. Of all online algorithms, those based on reinforcement learning are predominant and show much more effectiveness than other approaches. However, these former reinforcement learning algorithms for the multi-MC system with the multi-node charging model encounter lack of generality, scalability, and extensive cooperation between MCs. Those deficiencies of existing works highlight the thesis's contributions.

\textbf{Models using one MC and multiple MCs. }
Initially, pioneering studies in this field primarily addressed the scheduling problem for a single Mobile Charger (MC), as discussed in \cite{LYU2019388, 7889006, 8375982, 8885324, 8737589, huong2021effective, huong2022bi}. While this approach adequately suffices for small-scale WRSNs, it becomes less suitable for large-scale networks comprising thousands of nodes. As the scale of the WRSN increases, the demand for charging tasks assigned to the MC surges. In such scenarios, a single MC alone cannot effectively sustain the network's operation. Consequently, to alleviate the burden on a single MC, numerous research works have proposed models involving multiple mobile chargers. However, the inclusion of multiple MCs in the model necessitates more complex algorithms to facilitate cooperation and prevent collisions among them. A common strategy for addressing the presence of multiple mobile chargers involves dividing the main problem into two sub-problems. The first sub-problem entails clustering the sensors into groups, with each cluster being served by a single MC. The second sub-problem focuses on optimizing the charging schedule for each MC. In most studies, the clustering criteria for the first sub-problem aim to balance the workload among the MCs \cite{9078842,huong2021multi}. In a particular topology examined in \cite{9119859}, where sensors are uniformly deployed in a circular region, the authors proposed dividing the sensors into circular sectors. Another approach, as presented in \cite{9000624}, involves grouping sensors into on-time-covered sets, which can be charged by a single MC within a charging cycle. To tackle the second sub-problem, i.e., the optimization of charging schedules for each MC, various techniques have been employed, including fuzzy logic \cite{9078842}, genetic algorithms \cite{8882248,en12020287,huong2021effective,huong2022bi}, game theory \cite{LIN201688}, and linear programming \cite{8721449}. 

In general, the model using one MC is a special case of the multi-MC model. Therefore, algorithms built for the model using multiple MCs possess generalization and are adaptable for diverse network scenarios. However, the two-stage approach in virtually every research for the multi-MC model fails to capture the intricate coordination among MCs because their tasks are completely separated. In this approach, each MC is solely responsible for a specific set of sensors or an area, which means there is no assistance from other MCs if the workload in an MC's area suddenly increases. This drawback highlights the need for a more comprehensive mechanism that promotes cooperation and collaboration among MCs. 

\textbf{Single-node and Multi-node charging models. }The authors of \cite{huong2020genetic} classify charging algorithms for Mobile Chargers (MCs) into two main types: the single-node model and the multi-node charging model:
\begin{itemize}
    \item In the single-node model, the MCs visit the positions of individual sensor nodes to replenish their energy. This approach focuses on reducing the destination space for the charger and targeting specific sensor positions within the network. To address the issue of node failure avoidance, previous research such as \cite{feng2016starvation, zhu2018adaptive} has proposed heuristic algorithms. One such algorithm is the PA (Probability Assignment) algorithm, where a charging probability is assigned to each requesting node. The next node to be charged is selected based on this probability. Another algorithm is the INMA (Intelligent Node with Minimum Accumulated-charge) algorithm, which selects the next charged node in a way that minimizes the potential energy depletion for other requesting nodes.
    
    \item Recognizing that a single MC may not be sufficient to fulfill all charging requests in dense or large-scale networks, recent studies have explored multi-node charging schemes. In these schemes, an MC can charge multiple sensors simultaneously. For instance, in \cite{xu2019minimizing}, the authors demonstrated that overall energy efficiency can be significantly improved by properly turning a coupled resonator, allowing multiple receivers to be charged simultaneously instead of a single receiver. According to \cite{wang2018crcm}, in the multi-node charging model, a key physical parameter is the charging range. The charging range represents the radius of a circle centered at the location of an MC. Only sensor nodes located within this circle are eligible to receive energy replenishment from the MC. This range determines the spatial coverage of the charging process and influences which nodes can be charged simultaneously.
\end{itemize}

The single-node charging model is preferable in scenarios where sensors are widely distributed, with each sensor being far away from others. On the other hand, the multi-node charging model is more suitable for networks characterized by a high density of sensors, as discussed in \cite{long2023q}. It is worth noting that the single-node charging model can be viewed as a special case of the multi-node charging model, where the charging range is set to 0. One of the advantages of the multi-node charging model is that it allows mobile chargers to freely navigate and reach the exact positions of any sensor nodes same as the way of the single-node model. This flexibility makes algorithms designed for the multi-node charging model applicable to a wide range of scenarios. However, it is noteworthy that previous works often restrict the mobility of mobile chargers by enforcing a fixed set of charging locations, thereby limiting their flexibility and adaptability. For instance, Xu et al. explored the determination of charging schedules for multiple MCs in their studies \cite{xu2019minimizing, xu2020minimizing}. In their network system, the charging locations align with the precise positions of sensor nodes. The multi-node charging scheme is employed when an MC stops at a sensor node and charges all neighboring nodes within its charging range. In conclusion, while the multi-node charging scheme offers greater adaptability and efficiency compared to the single-node charging model, there is still room for the development of new algorithms to address the limitations imposed by fixed charging locations. By enabling MCs to dynamically adjust their charging destinations in response to real-time network conditions and energy demands, it is possible to further enhance the flexibility of the multi-node charging model. This approach can lead to the creation of more optimized and robust charging algorithms that effectively meet the energy needs of the network.

\textbf{Online and offline schedulers}

According to \cite{huong2020genetic}, charging algorithms for MCs are divided into two main groups:

\begin{itemize}
    \item Offline scheduler: These algorithms offlineally produce a long sequence of steps for the MC based on the current information of the WRSN. Each step can be interpreted as moving to a destination and charging neighboring sensors for an amount of time. As soon as the MC finishes all demands in the previous sequence, the scheduler yields a new one. In the study conducted by Lin et al. \cite{lin2019minimizing}, a novel energy transfer concept based on distance and angle was introduced to minimize the overall charging delay time. They proposed an optimal solution using linear programming techniques. Similarly, Jiang et al. \cite{jiang2017joint} aimed to minimize the number of MCs while ensuring that no sensor nodes experience energy depletion. They achieved this by jointly determining the charging paths of the MCs and the locations of base stations (BSs). In another work by Lyu et al. \cite{lyu2019periodic}, a combination of Genetic Algorithm (GA) and Particle Swarm Optimization (PSO) was employed to find an optimal path for the MC that minimizes the idle time spent at the BS for energy replenishment. On the other hand, Ma et al. \cite{ma2018charging} proposed approximation algorithms for determining the charging path of the MC with the objective of maximizing the overall accumulative charging utility gain or minimizing the energy consumption of the MC during its travel.
    \item Online scheduler: Instead of yielding a sequence step for the MC, these schedulers only output a next step. Regarding conventional algorithms of this type, most of them adopt a queue of charging requests sent from sensors. As soon as the MC finishes the previous steps, algorithms implement a mechanism to map the most important charging request to the next step of the MC. For instance, the paper \cite{zhu2018adaptive} presents a new probability-based selection strategy to minimize the number of depleted nodes over time. In this work, the authors set up a threshold energy level for sensors. If the energy of a sensor decrease below the threshold, it will send a charging request. The sooner the request is sent, the more important it is. To improve the quality of online requests, the works in \cite{7924317,LIN201972} propose a multi-level warning threshold for each sensor node depending on its remaining energy. 
\end{itemize}

It is important to acknowledge that the use of offline scheduling in charging algorithms is not widely applicable in practical scenarios. This is because nodes in the network often deactivate unexpectedly, resulting in changes to the network topology. Consequently, predetermined charging schedules may no longer be optimal or feasible in the altered topology, as highlighted by He et al. \cite{he2013demand}. On the other hand, online charging algorithms offer a more reasonable approach for practical applications. These algorithms employ a step-by-step controller approach, responding to charging requests as they arise, rather than relying on pre-determined schedules. However, the effectiveness of online charging algorithms, which operate based on request queues, heavily relies on the selection of the energy threshold for charging requests. If a low threshold is set, mobile chargers (MCs) will have limited time to fulfill charging requests from critical nodes. In situations where numerous requests occur simultaneously, it may become infeasible to select an optimal path to satisfy all these requests, thereby putting the network at risk \cite{zhu2018adaptive}. On the contrary, setting a high threshold for the energy threshold may result in an inefficient charging scheme, as an influx of charging requests may occur simultaneously. This can lead to congestion and suboptimal utilization of the MCs' charging capabilities \cite{zhu2018adaptive}.

To address the limitations of online algorithms based on queues, researchers have recently turned to reinforcement learning (RL) techniques for controlling the operation of mobile chargers (MCs). RL-based algorithms have shown significant improvements over queue-based approaches in addressing various charging problems in WRSNs. The superiority of RL algorithms can be attributed to the nature of WRSNs, which can be seen as a robot-controlling problem. Indeed, RL algorithms have demonstrated superhuman performance in such interacting agents applications as Atari games \cite{mnih2015human}, mastering the game of Go \cite{silver2016mastering}, chess \cite{silver2017mastering}, robotic \cite{kober2013reinforcement}, making them well-suited for optimizing MCs' operations in WRSNs. Moreover, queue-based algorithms often focus on optimizing short-term objective functions and may overlook long-term performance. In contrast, RL-based algorithms are designed to create interactive agents that maximize cumulative rewards, leading to a long-term commitment to network performance. By considering the overall network performance over time, RL-based algorithms can offer more effective and sustainable solutions for charging problems in WRSNs.  

Current applications of RL in charging problems primarily focus on WRSNs utilizing a single mobile charger (MC). For example, in \cite{cao2021deep}, the authors formulate the charging problem as a Markov Decision Process (MDP) and propose an advanced deep reinforcement learning algorithm to solve the MDP. Their algorithm optimizes the charging path while considering constraints on the MC's energy, sensors' energy levels, and time windows. Similarly, in \cite{yang2020dynamic}, the authors introduce an actor-critic reinforcement learning algorithm that incorporates gated recurrent units (GRUs) to capture the temporal relationships of charging actions. In the paper \cite{gong2023deep}, the Double Dueling DQN algorithm is employed to jointly optimize the scheduling of both the charging sequence of the MC and the charging amount of individual nodes. Furthermore, Bui et al. \cite{bui2022deep} propose a policy-based RL algorithm to determine the next sensor to charge based on the current state of the network.

The application of RL in the context of multi-mobile charger (MC) multi-node charging models is relatively limited. This can be attributed to the complexity of Multi-Agent Reinforcement Learning (MARL), which has shown promise in solving sequential decision-making problems involving multiple autonomous agents aiming to optimize long-term returns. However, applying the problem model of Decentralized Partially Observable Markov Decision Process (Dec-POMDP) \cite{oliehoek2016concise}, a framework commonly used in MARL, to charging problems in WRSNs is inappropriate. The primary issue lies in the synchronous nature of the time step definition in the Dec-POMDP framework. It assumes synchronized action execution among agents, while in WRSNs, MCs complete charging steps at different time points. If we were to clip the steps of MCs using a time interval, it would pose a significant problem in selecting an appropriate interval value. A short time interval may prevent MCs from completing their steps entirely, as it may not allow enough time for them to move from their current location to the destination or provide sufficient charging time. On the other hand, a long time interval would result in inactive MCs that wait excessively long periods before committing to the next step. Consequently, traditional MARL algorithms based on Dec-POMDP, such as MAPPDG \cite{lowe2017multi}, QMIX \cite{rashid2020monotonic}, and MAPPO \cite{yu2021surprising}, would prove inefficient in cooperative charging scenarios. 

A naive approach to remedy this challenge is considering each MC as an independent agent, so it becomes possible to apply traditional single-agent RL algorithms to each MC's MDP. One notable example in the literature is the proposal presented in \cite{long2023q}, which utilizes the Q-learning technique. In this work, each MC maintains a Q-table to assess the quality of different actions. This approach simplifies the problem by breaking it down into multiple single-agent problems, allowing each MC to make decisions based on its own observations and goals. To separate the operation of MCs, it is important for each MC's state space to include information about the other agents, treating them as part of the environment. However, in this proposal, the state space is limited to the set of charging locations, which means that it fails to capture the information of other agents. As a result, the coordination and cooperation between MCs are adversely affected, impacting the overall effectiveness of the system. It is worth mentioning that this study shares a common objective with the thesis, focusing on achieving target coverage and connectivity in a network model involving multiple MCs and a multi-node charging mechanism. This work stands out as one of the few that addresses these specific aspects and aligns closely with the goals of the thesis. There are also other works that try to remedy the complexity of MARL. They are not discussed here due to the difference in the objective of the thesis. 

In conclusion, almost works suffer from common drawbacks listed as follows:
\begin{itemize}
    \item Incidentally decreasing the generalization and the scalability of WRSNs by adopting either model using one MC or single-node charging mechanism.

    \item Fixing the destinations of MCs which potentially produce sensors definitely running out of energy.

    \item Failing to emphasize the cooperation between MCs.

    \item Requiring retraining reinforcement learning models if the flee of MCs is brought to a new network.

\end{itemize}

\section{Contributions and organization of the thesis}

Overall, the thesis makes significant contributions to addressing the limitations of previous works in the following ways:
\begin{itemize}
    \item The thesis thoroughly studies the concept of WSN, and WRSN specifically focusing on the target coverage and connectivity problem. It provides a comprehensive analysis of existing charging models for WRSNs, including an in-depth examination of offline and online charging strategies.
    
    \item  The thesis proposes an effective Decentralized Partially Observable Semi-Markov Decision Process (Dec-POSMDP) model that mathematically formulates the multi-node charging model. This model promotes cooperation among Mobile Chargers (MCs) and avoids restricting the set of destinations for MCs. Furthermore, the proposed MDP model allows reinforcement algorithms to be applied to different networks without requiring extensive retraining.
    
    \item A revised version of the Proximal Policy Optimization algorithm (PPO) is introduced in the thesis, specifically tailored to accommodate the asynchronous nature of MCs' steps in WRSNs.
    
    \item Finally, extensive experiments are conducted in various practical scenarios, showing that our approach outperforms other state-of-the-art methods. 
\end{itemize}

The structure of this thesis is as follows:
\begin{itemize}
    \item Chapter \ref{chap:intro} introduces some basic concepts of \acrshort{WSN}s, \acrshort{WRSN}s, NLM-CTC problem, and its related works.

    \item Chapter \ref{chap:preliminaries} introduces some fundamental theories about Reinforcement Learning, Cooperative Multi-Agent Reinforcement Learning, and Convolutional Neural Networks. 
    
    \item Chapter \ref{chap:network-model} establishes the network model and provides a mathematical formulation (Dec-POSMDP) for the NLM-CTC.
    
    \item Chapter \ref{chap:proposed} presents the proposed algorithm.
    \item Chapter \ref{chap:experiments} conducts extensive experiments and analysis.
\end{itemize} 


\newpage
\chapter{PRELIMINARIES}
\label{chap:preliminaries}
In recent years, the application of reinforcement learning (RL) and multi-agent reinforcement learning (MARL) has gained significant attention and demonstrated remarkable impact across various sectors. Chapter \ref{chap:preliminaries} provides a concise introduction to fundamental terminologies, models, and popular approaches in RL and MARL. Additionally, this chapter explores the utilization of convolutional neural networks (CNNs) and their variant, U-net, highlighting their underlying support and relevance in the context of RL.

\section{Reinforcement learning}
\subsection{Markov decision process}
\begin{figure} [!ht]
    \centering
    \includegraphics[width=\textwidth]{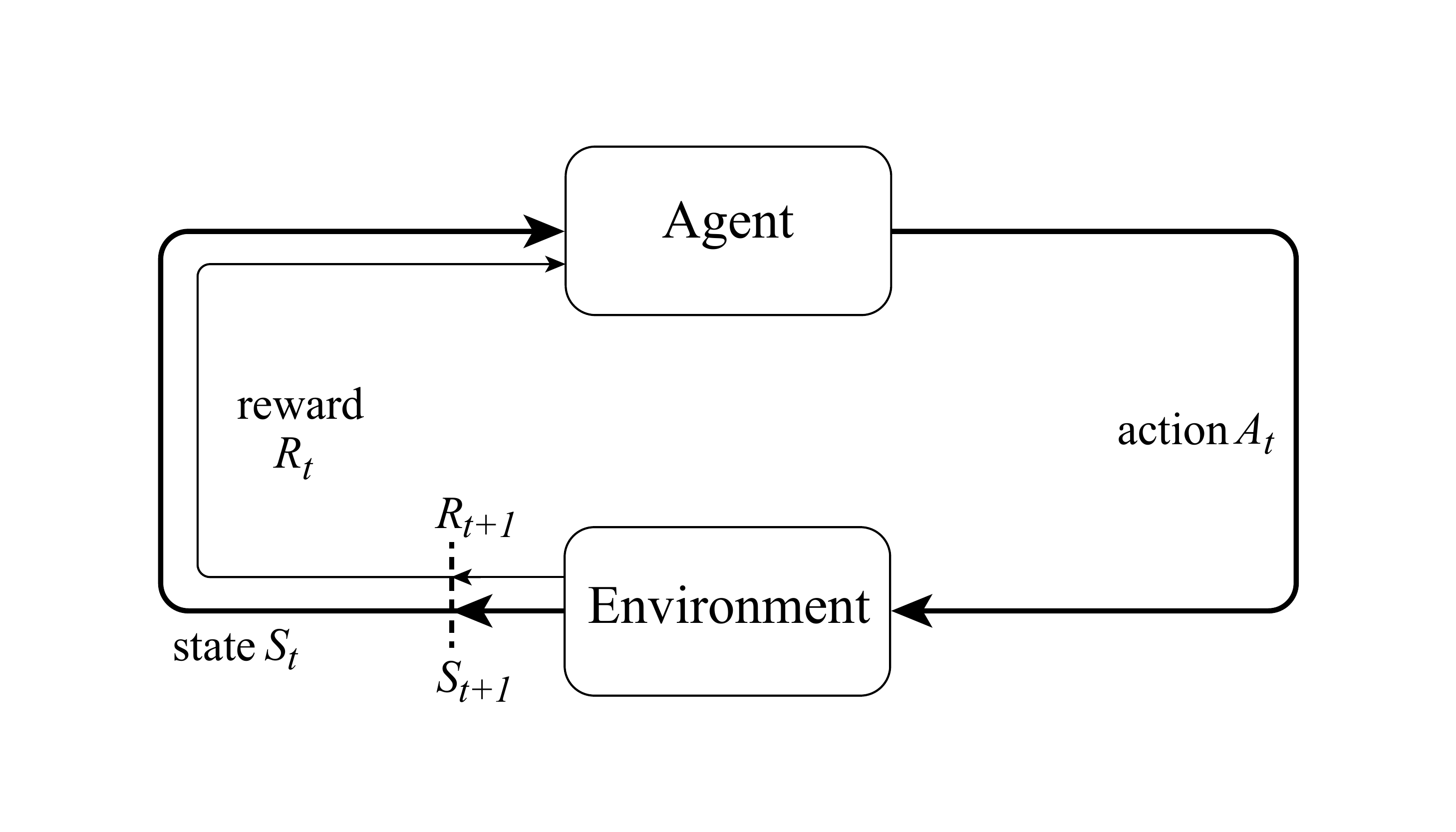}
    \caption[Markov Decision Process]{In a Markov Decision Process (MDP), an agent interacts with an environment over time. At each timestep $t$, the agent observes the current state of the environment, denoted as $S_t$, and takes an action, denoted as $A_t$. As a result of the agent's action, the environment transitions from state $S_t$ to a new state $S_{t+1}$ and provides the agent with a reward of $R_{t+1}$. This interaction process continues in a loop, with the agent repeatedly observing states, taking actions, and receiving rewards, as represented by the sequence: $S_0, A_0, R_1, S_1, A_1, R_2, S_2, \dots$
    \label{fig:mdp}}
\end{figure}

Markov Decision Processes (MDP) \cite{bellman1957markovian} serve as the conventional method for modeling reinforcement learning (RL) problems, focusing on the interaction between the agent and the environment at each timestep $t$. An MDP is defined by a triplet $\mathcal{M} = (\mathcal{S},\mathcal{A}, \mathcal{P}_0)$, where $\mathcal{S}$ represents the set of states, $\mathcal{A}$ denotes the set of actions available to the agent, and $\mathcal{P}_0$ is the transition probability kernel. In particular, given a state $s$ and an action $a$, $\mathcal{P}_0(\cdot|x, a)$ is a probability map from $\mathcal{S}\times\mathbb{R}\rightarrow\mathbb{R}$. For instance:

\begin{equation}
    \label{prob-kernel}
    \mathcal{P}_0(y, R) = \Pr(y, R | S_t=x, A_t=a)
\end{equation}

The probability kernel, as depicted by the equation~\eqref{prob-kernel}, quantifies the likelihood of the environment transitioning from its current state $S_t=x$ to the next state $S_{t+1}=y$, with an associated intermediate reward $R_{t+1}=R$, given that the agent takes action $A_t=a$.

The \textit{return} is defined as follows:
\begin{equation}
    \label{return}
    \mathcal{R} = \sum_{t=0}^{+\infty}\gamma^tR_{t+1}
\end{equation}

where $\gamma \in [0, 1]$ is the discount factor.

When the \textit{discount factor} $\gamma$ is set to 0, the expected return $\mathcal{R}$ only takes into account the immediate reward $R_0$, and the agent is considered near-sighted as it ignores potential long-term gains. Conversely, a higher value of $\gamma$ indicates a more far-sighted agent, as it considers long-term rewards into its decision-making. The stochastic nature of $\mathcal{R}$ arises from the fact that $R_t$ is determined by a probability map. The \textit{expected return} of $\mathcal{R}$ can be calculated as follows:

\begin{equation}
    \label{expected-return}
    \mathbb{E}\left[\mathcal{R}\right] = \mathbb{E}\left[ \sum_{t=0}^{+\infty}\gamma^tR_{t+1}\right]
\end{equation}

The primary objective of an \acrshort{MDP} is to maximize the expected return, as shown in \eqref{expected-return}. The most common approach to achieve this objective is to estimate the expected return and use it to optimize the agent's behavior. The \textit{optimal value} for each state $x$ is denoted by $V^*(x) = E[R|S_0=x]$, and $V^*: S\to\mathbb{R}$ represents the so-called \textit{optimal value function}. An \textit{optimized} behavior refers to a behavior that leads to the optimal value for each state $x \in \mathcal{S}$, which corresponds to the optimal value function \cite{szepesvari2010algorithms}. The behavior being referred to is commonly known as a \textit{policy}. There are two types of policies:
\begin{itemize}
    \item Deterministic stationary policies are a type of policy that directly maps a state $x$ to an action $a$. For instance, if we have a deterministic stationary policy $\pi: \mathcal{S} \to A$, then the agent will consistently choose the action $A_t = \pi(x) = a$ whenever it encounters the state $S_t = x$. In other words, the agent selects the same action every time it faces the same state.
    \item In contrast, stochastic stationary policies are policies that represent probability distributions over the set of possible actions $A$, given a particular state $x$. If $\pi$ is a stochastic stationary policy and the agent is in state $S_t=x$, then the policy will randomly choose an action from the distribution:
    \[ A_t \sim \pi(\cdot | S_t=x).\]
    In other words, the agent may select different actions each time it encounters the same state, as the policy is probabilistic.
\end{itemize}
The value function based on policy $\pi$ is denoted as $V^{\pi}(x)$ and can be defined as:

\[ V^{\pi}(x) = \mathbb{E}\left[ \sum_{t=0}^{+\infty}\gamma^tR_{t+1} \mid S_0=x\right] \]

So far, we have only considered the expected return of a state $x$; oftentimes, this estimation provides insufficient information to select the next action $a$. Similarly, the \textit{action-value function} is defined as:

\[ Q^{\pi}(x, a) = \mathbb{E}\left[ \sum_{t=0}^{+\infty}\gamma^tR_{t+1} \mid S_0=x, A_0=a\right] \]

We use $Q^*$ to denote the optimal action-value function. According to \cite{sutton2018reinforcement}, the following relation holds:
\begin{equation*}
     V^{\ast}(x) = \sup_{a\in\mathcal{A}} Q^{\ast}(x, a), \forall x\in\mathcal{S}
\end{equation*}
This implies that if we know the optimal action-value function $Q^*$, we can always choose the action that maximizes the expected cumulative reward for any given state $x$. We simply need to select the action $a$ that yields the maximum value of $Q^{\ast}(x,a)$. The Q-learning algorithm is designed to optimize the action-value function $Q_{\pi}$, which can be used to derive a near-optimal policy $\widetilde{\pi}$.

\subsection{Actor-critic approaches}

Reinforcement learning (RL) is a subfield of machine learning that focuses on training agents to make sequential decisions in an environment to maximize a long-term reward signal. Actor-Critic approaches have gained significant attention in RL due to their ability to address the challenges of high-dimensional and continuous action spaces.

\subsubsection*{Actor-Critic Architecture}
The Actor-Critic architecture is a hybrid approach that combines elements of both value-based methods and policy-based methods. It consists of two main components: the actor and the critic.

The actor is responsible for selecting actions based on a policy. It receives observations from the environment and maps them to a probability distribution over actions. The policy can be deterministic or stochastic, depending on the specific implementation. Stochastic policies are commonly used to explore the action space more effectively.

The critic, on the other hand, evaluates the actions taken by the actor and provides feedback in the form of value estimates. It estimates the expected return or value function associated with a particular state or state-action pair. The critic's role is to guide the actor by providing a signal that indicates the quality of the actor's actions.

\subsubsection*{Training Procedure}
The training procedure in actor-critic methods involves an iterative process of updating both the actor and the critic based on the received feedback. During each iteration, the agent interacts with the environment, receiving observations, taking actions, and receiving rewards. The actor selects actions according to the current policy, while the critic estimates the value of each state or state-action pair encountered.

To train the actor, the advantages of each action are computed. The advantage function represents the advantage of taking a specific action compared to the average value of actions in that state. The actor's parameters are updated using the gradient descent method. The update equation for the actor can be expressed as:
\[
\theta_{\text{actor}} \leftarrow \theta_{\text{actor}} + \alpha \nabla_{\theta_{\text{actor}}} J(\theta_{\text{actor}}),
\]
where $\theta_{\text{actor}}$ represents the parameters of the actor, $\alpha$ is the learning rate, and $\nabla_{\theta_{\text{actor}}} J(\theta_{\text{actor}})$ is the gradient of the objective function $J(\theta_{\text{actor}})$ with respect to the actor's parameters.

Similarly, the critic's parameters are updated using the gradient descent method. The updated equation for the critic can be expressed as:
\[
\theta_{\text{critic}} \leftarrow \theta_{\text{critic}} + \alpha \nabla_{\theta_{\text{critic}}} L(\theta_{\text{critic}}),
\]
where $\theta_{\text{critic}}$ represents the parameters of the critic, $\alpha$ is the learning rate, and $\nabla_{\theta_{\text{critic}}} L(\theta_{\text{critic}})$ is the gradient of the loss function $L(\theta_{\text{critic}})$ with respect to the critic's parameters.

The specific form of the objective function and the loss function depends on the chosen algorithm and the problem.

\subsubsection*{Advantages of Actor-Critic Approaches}
Actor-Critic approaches offer several advantages in RL:
\begin{enumerate}
    \item They combine the benefits of both value-based and policy-based methods. The critic provides valuable feedback to guide the actor's actions, while the actor explores the action space more effectively.
    \item They can handle high-dimensional and continuous action spaces, which are challenging for traditional value-based methods.
    \item They have the potential for online learning, as the actor can continuously update its policy based on the received feedback.

    \item They can achieve faster convergence and improved sample efficiency compared to other RL methods.
\end{enumerate}

Actor-Critic approaches have been successfully applied to various domains, including robotics, game-playing, and autonomous driving. Their ability to handle continuous control and learn directly from raw sensor inputs makes them suitable for real-world applications requiring precise and continuous actions.

\subsection{Proximal policy optimization algorithm}

Proximal Policy Optimization (PPO) aims to optimize the policy parameters, denoted by $\theta$, to maximize the expected cumulative reward, or return, over time. The agent interacts with the environment, and at each time step t, the agent observes the state $s_t$, takes an action $a_t$ according to the policy $\pi(a_t|s_t, \theta)$, receives a reward $r_t$, and transitions to the next state $s_{t+1}$.

The goal is to find the optimal policy parameters $\theta^{\ast}$ that maximize the expected return:

\[ J(\theta) = \mathbb{E}_{\tau \sim \pi_\theta} \left[ \sum_{t=0}^\infty \gamma^t r_t \right], \]

where $\tau = (s_0, a_0, r_0, s_1, a_1, r_1, \ldots)$ is a trajectory sampled from the policy $\pi_{\theta}$, and $\gamma$ is the discount factor ($0 < \gamma < 1$) that discounts future rewards.

To update the policy parameters, PPO introduces a surrogate objective function, denoted as $L^{CLIP}(\theta)$, which estimates the expected advantage of the new policy over the old policy. The advantage function $A(s_t, a_t)$ measures how much better an action $a_t$ is than the average action in state $s_t$.

The clipped surrogate objective function $L^{CLIP}(\theta)$ is given by:

\[ L^{CLIP}(\theta) = \mathbb{E}_{t} \left[ \min(r_t(\theta) \cdot A_t, \text{clip}(r_t(\theta), 1 - \varepsilon, 1 + \varepsilon) \cdot A_t) \right], \]
where $r_t(\theta) = \frac{\pi_\theta(a_t|s_t)}{\pi_{\theta_{\text{old}}}(a_t|s_t)}$ is the ratio of the probabilities of selecting action $a_t$ under the new policy and the old policy, and $ \varepsilon$ is a hyperparameter that controls the clipping range.

The clipped surrogate objective function applies the min operator between two terms: one term that represents the policy improvement, and the other term that includes a clipping factor to limit the update size. If the ratio $r_t(\theta)$ is within the range $1 - \varepsilon$ to $1 + \varepsilon$, the update is not clipped. Otherwise, it is clipped to be within this range. This ensures that the update is not too large and maintains stability during training.

To optimize the policy, the algorithm collects trajectories using the current policy, computes the advantage estimates, and performs multiple epochs of training. Each epoch consists of updating the policy and possibly the value function to improve performance.

Overall, Proximal Policy Optimization (PPO) combines the advantages of policy gradient methods and trust region methods. By introducing the clipped surrogate objective function, PPO achieves more stable and efficient learning compared to its predecessors. The algorithm has been widely used in various reinforcement learning applications and continues to be an active area of research.

\section{Multi-agent reinforcement learning}

Similarly to single-agent RL, multi-agent reinforcement learning (MARL) also aims to resolve sequential decision-making problems, but in settings involving more than one agent involved. In this setting, the next state, and reward received by each agent depend on both its behavior and other agents' behavior.

\subsection{Terminologies and taxonomies}

To avoid misunderstanding and confusion potentially raised by terminologies, the definition of a multi-agent system in this section would be kept intact when applied to the thesis' problem. 

A general multi-agent system denoted as a tuple $(\mc N, \mc S, \mc A, R, P, \mc O, \gamma)$ in which $\mc N = \{1, 2, \dots, N\} $ is the set of $N$ agents, $\mc S$ is the general system space, $A = \{\mc A_1, \mc A_2, \dots, \mc A_N\}$ is the set of actions for agents, $R$ is the reward function, $P$ is the state action transition probability, and $ \mc O = \{\mc O_1, \mc O_2, \dots, \mc O_N\}$ is the set of all observations of all agents.  $R$ and $P$ in this model are ambiguous and so depend on the specific setting of the system. An approach to generalize the multi-agent system is utilizing the definition of the Markov game (MG), also known as stochastic games in \cite{shapley1953stochastic} that incorporate the synchronous and fully observable properties of agents' behaviors. The formal definition is established as below:

\begin{definition}
    \label{def:markov_game}
     A MG is defined by a tuple $(\mc N, \mc S, \{ \mc A^i\}_{i \in \mc N}, \{R^i\}_{i \in \mc N}, P, \gamma)$, where $\mc N = \{1, \dots, N\}$ denotes the set of $N > 1$ agents, $\mc S$ denotes the state space observed by all agents, $\mc A^i$ denotes the action space of agent $i$. Let $\mc A:= \mc A^1 \times \dots, \times \mc A^N$, then $P: \mc S \times \mc A \rightarrow \Delta(\mc S) $ denotes the transition probability from any state $s \in \mc S$ to any state $s' \in \mc S$ for any joint action $a \in \mc A$; $R^i: \mc S \times \mc A \times \mc S \rightarrow \R$ is the reward function that determines the immediate reward received by agent $i$ for a transition from $(s, a)$ to $s'$; $\gamma \in [0, 1)$ is the discount factor.
\end{definition}

At time $t$, each agent $i \in \mathcal{N}$ takes an action $a_{t}^{i}$ based on the system state $s_{t}$. The system then transitions to state $s_{t+1}$ and rewards each agent $i$ with $R^{i}\left(s_{t}, a_{t}, s_{t+1}\right)$. The objective of agent $i$ is to optimize its long-term reward by finding a policy $\pi^{i}: \mathcal{S} \rightarrow \Delta\left(\mathcal{A}^{i}\right)$, where $a_{t}^{i}$ is sampled from $\pi^{i}\left(\cdot \mid s_{t}\right)$. As a result, the value function $V^{i}: \mathcal{S} \rightarrow \mathbb{R}$ of agent $i$ becomes dependent on the joint policy $\pi: \mathcal{S} \rightarrow \Delta(\mathcal{A})$, defined as $\pi(a \mid s):=\prod_{i \in \mathcal{N}} \pi^{i}\left(a^{i} \mid s\right)$. Specifically, for any joint policy $\pi$ and state $s \in \mathcal{S}$,

\[
V_{\pi^{i}, \pi^{-i}}^{i}(s):=\mathbb{E}\left[\sum_{t \geq 0} \gamma^{t} R^{i}\left(s_{t}, a_{t}, s_{t+1}\right) \mid a_{t}^{i} \sim \pi^{i}\left(\cdot \mid s_{t}\right), s_{0}=s\right],
\]
where $-i$ denotes the indices of all agents in $\mathcal{N}$ except agent $i$, the solution concept for a multi-agent setting deviates from that of a Markov Decision Process (MDP). This deviation arises due to the fact that the optimal performance of each agent is not solely determined by its own policy, but also influenced by the decisions made by all other players in the game. The Nash equilibrium (NE) stands as the prevailing solution concept in this scenario, and it is defined as follows \cite{filar2012competitive, bacsar1998dynamic}.

\begin{definition}
    A Nash equilibrium of the MG $\left(\mathcal{N}, \mathcal{S},\left\{\mathcal{A}^{i}\right\}_{i \in \mathcal{N}}, P,\left\{R^{i}\right\}_{i \in \mathcal{N}}, \gamma\right)$ is a joint policy $\pi^{*}=\left(\pi^{1, *}, \cdots, \pi^{N, *}\right)$, such that for any $s \in \mathcal{S}$ and $i \in \mathcal{N}$ 
    \[
    V_{\pi^{i, *}, \pi^{-i, *}}^{i}(s) \geq V_{\pi^{i}, \pi^{-i, \pi}}^{i}(s), \quad \text{for any } \pi^{i}.
    \]
\end{definition} 

Nash equilibrium indicates a set of points $\pi^{*}$ from which any agents can take advantage by deviating. Concretely, the policy $\pi^{i, *}$ of each agent $\i$ is the best policy conditioning on the policy of other agents $\pi^{-i, *}$. According to \cite{filar2012competitive}, NE always exists for finite-space infinite-horizon discounted MGs, but does not generally necessitate being unique. All MARL algorithms strive to obtain these equilibrium points if they exist.

In a multi-agent reinforcement learning (MARL) system, the settings can be broadly categorized into three types: fully cooperative, fully competitive, and mixed.
\begin{itemize}
    \item In cooperative settings, all agents engage in learning and collaboration to collectively optimize a shared objective. Comprehensive surveys about cooperative settings can be found in \cite{oroojlooy2023review, matignon2012independent}.

    \item  In contrast, in competitive settings, each agent strives to improve its own objective while simultaneously impeding the progress of its opponents. This concept is akin to a zero-sum game, where the sum of returns for all agents adds up to zero. Some particular settings are examined in \cite{bucsoniu2010multi}.

    \item The mixed settings, also known as the general-sum games \cite{littman2001friend}, involve both the cooperative and competitive relations between agents. A popular instance is that a group of agents collaborate to fight against other groups of agents.  
\end{itemize}

The charging problem in WRSNs is commonly recognized as a fully cooperative setting. Therefore, this section will primarily focus on former works that have developed models and solutions specifically tailored to cooperative settings. Consider a Markov game as defined in Definition \ref{def:markov_game}, where the reward function $R$ is the same for all agents, denoted as $R^1 = \dots = R^N = R$. Here, the reward function $R$ depends on the joint action $a \in \mc A$. In this scenario, the value function or Q-function is identical for all agents. This allows us to treat the group of agents as a single decision-maker, with the joint action space representing the actions of all agents. Consequently, we can leverage single-agent reinforcement learning algorithms to find the policy for the group of agents.

\subsection{Decentralized partially observable markov decision process}

It is worth noting that the Definition \ref{def:markov_game} assumes that every agent can observe the general state of the system including all parameters that affect the problem and also how the decision-makers of other agents work. However, in real-world applications, it is impossible to capture all impacting factors. Hence, each agent at each time step would receive its own local observation instead of the overall state of the system. This setting is formulated as a Decentralized partially observable Markov decision process (Dec-POMDP) \cite{bernstein2002complexity}. The definition of Dec-POMDP is presented as follows.

\begin{definition}
    \label{decpomdp}
    A Dec-POMDP is defined by a tuple $(\mathcal{N}, \mathcal{S}, \mathcal{A}, \mathcal{O}, P, R, Z, \gamma)$, where:
    \begin{itemize}
        \item $\mathcal{N}$ is the set of $N$ agents.
        \item $\mathcal{S}$ is the set of possible states.
        \item $\mathcal{A} = \mathcal{A}^1 \times \mathcal{A}^2 \times \dots \times \mathcal{A}^N$ is the joint action space, where $\mathcal{A}^i$ is the set of actions available to agent $i$.
        \item $\mathcal{O} = \mathcal{O}^1 \times \mathcal{O}^2 \times \dots \times \mathcal{O}^N$ is the joint observation space, where $\mathcal{O}^i$ is the set of observations available to agent $i$.
        \item $P: \mathcal{S} \times \mathcal{A} \times \mathcal{S} \rightarrow [0, 1]$ is the state transition function, which gives the probability of transitioning from state $s$ to $s'$ given joint action $a$.
        \item $R: \mathcal{S} \times \mathcal{A} \rightarrow \mathbb{R}$ is the reward function, which assigns a real-valued reward to each state-action pair.
        \item $Z: \mathcal{S} \times \mathcal{A} \times \mathcal{O} \rightarrow [0, 1]$ is the observation function, which gives the probability of receiving observation $o$ given state $s$ and joint action $a$.
        \item $\gamma \in [0, 1]$ is the discount factor, determining the trade-off between immediate and future rewards.
    \end{itemize}
\end{definition}

At each time step, each agent $i \in \mathcal{N}$ observes an observation $o^i \in \mathcal{O}^i$ based on the true state $s$ and its own action $a^i$. The agents collectively take a joint action $a \in \mathcal{A}$ based on their observation history, and the system transitions to a new state $s'$ based on the state transition function $P$. The difference between Definition~\ref{def:markov_game} and Definition~\ref{decpomdp} lies in the appearance of $\mc O$ and $Z$ which represent the local information of agents.

\subsection{Options in markov decision process}

In the context of MDPs, an option is a temporally extended action that allows an agent to pursue a specific subgoal or achieve a particular behavior. It provides a higher-level abstraction of actions, enabling the agent to plan and execute a sequence of actions to reach a desired state or achieve a specific task. Formally, an option in an MDP is defined as a tuple $(\mathcal{I}, U, B)$, where:

\begin{itemize}
    \item $\mathcal{I}$ is the initiation set, which specifies the states where the option can be initiated or selected by the agent.
    \item $U$ is the policy associated with the option, representing the action selection strategy for the agent while executing the option. It maps states to actions or sub-actions.
    \item $B$ is the termination condition, which determines when the option terminates and the control returns to the agent's default policy. It can be based on reaching a specific state, satisfying a certain criterion, or the passage of a certain amount of time.
\end{itemize}

Options provide a way to hierarchically decompose complex tasks into a sequence of subtasks, allowing the agent to efficiently explore and exploit the MDP's state-action space. By using options, an agent can learn and execute more sophisticated and goal-oriented behaviors, leading to improved performance and adaptability in complex environments.

\subsection{Dec-POSMDP: \textbf{Dec}entralized \textbf{P}artially \textbf{O}bservable \textbf{S}emi-\textbf{M}arkov \textbf{D}ecision \textbf{P}rocess}

 A Dec-POSMDP is defined by a set of elements $(\mathcal{N}, \mathcal{S}, \mathcal{A}, \mathcal{O}, P, R, Z, \mc U, B, \gamma)$. The common components of the decentralized partially observable Markov decision process (Dec-POMDP) are kept intact, while two settings of option are added, including $\mc U, B$. In addition, $\mc A$ is the atomic action space instead. $\mc {U} = \mathcal{U}^1 \times \mathcal{U}^2 \times \dots \times \mathcal{U}^N$ is the joint macro-action space. A macro action $u^{i} \in \mc U^i$ of an agent $i$ is a high-level policy that can generate a sequence of atomic actions $a^i_{t} \sim u^{i}\left(H_{t}^{i}\right)$ for any $t$ when $u^{i}$ is activated, where $H_{t}^{i}$ is the individual action-observation history till $t$. $B$ denotes the stop condition of macro actions and $B^{i}(u^{i})$ is represented as a set of action-observation histories of an agent $i$. If $H_{t}^{i} \in B^{i}(u^{i})$ holds, $u^{i}$ terminates and the agent generates a new macro action. $R^{\tau}$ is the macro joint reward function: $R^{\tau}(s, u)=\mathbb{E}\left[\sum_{t=0}^{\tau_{\text {end }}} \gamma^{t} R\left(s_{t}, a_{t}\right) \mid a_{t} \sim u\left(H_{t}\right)\right]$ where $\tau_{\text {end }}=\min _{t}\left\{t: H_{t}^{i} \in B^{i}\left(u^{i}\right)\right\}$.

The solution of a Dec-POSMDP is a joint high-level decentralized policy $\Theta=\left(\Theta^{1}, \cdots, \Theta^{N}\right)$ where each $\Theta^{i}$ produces a macro action $\Theta^{i}\left(H_{t}^{i}\right) \in \mc U^{i}$ given individual action-observation history $H_{t}^{i}$. At the beginning of an episode, an initial macro action is computed as $u_{t_{0}}^{i}=\Theta^{i}\left(H_{t_{0}}^{i}\right)$. At action-making step $k>0$, the agent generates a new macro action $u_{t_{k}}^{i}=\Theta^{i}\left(H_{t_{k}}^{i}\right)$ if the stop condition is met, i.e. $H_{t_{k}}^{i} \in B^{i}\left(u_{t_{k-1}}^{i}\right)$. Otherwise, the agent continues to use the previous MA: $u_{t_{k}}^{i}=u_{t_{k-1}}^{i}$. In the time range $\left[t_{k}, t_{k+1}\right)$, the agent interacts with the environment with atomic actions sampled from MA: $a_{t}^{i} \sim u^{i}\left(H_{t}^{i}\right)$. Finally, the goal of Dec-POSMDP is to maximize the accumulative discounted reward: $\mathbb{E}\left[\sum_{k=0}^{\infty} \gamma^{t_{k}} \bar{R}^{\tau}\left(\bar{s}_{t_{k}}, \bar{u}_{t_{k}}\right) \mid \bar{\Theta}, \bar{s}_{0}\right]$ where $t_{0}=0$ and $t_{k}=\min _{t}\left\{t>t_{k-1}: H_{t}^{i} \in B^{i}\left(u_{t_{k-1}}^{i}\right)\right\}$ for $k \geq 1$. A more detailed definition can be found in [31].

\section{Convolutional neural network}

Convolutional Neural Networks (CNNs) are a class of deep learning models specifically tailored for processing image-like data. Over the past decade, CNNs have established themselves as the dominant approach in various sub-fields of computer vision. They have demonstrated their efficacy in several key areas, including (i) object detection tasks as demonstrated by Redmon et al. in \cite{redmon2016you}, (ii) image segmentation tasks as showcased by Ronneberger et al. in \cite{ronneberger2015u}, (iii) image classification tasks as exemplified by Huang et al. in \cite{huang2017densely}, and so forth. These notable contributions highlight the versatility and effectiveness of CNNs in tackling diverse computer vision challenges.

\subsection{Architecture}

Like other deep learning models, Convolutional Neural Networks (CNNs) consist of multiple stacked layers. Over time, various kinds of layers have been added to CNNs to enhance their performance and cater to specific tasks. However, the principal layer that empowers CNNs to capture local spatial patterns and hierarchies of features is the convolution layer. Additionally, other popular supportive layers include activation layers, pooling layers, normalization layers, and fully connected layers. Here's a brief overview of the forward pass of these layers:

\begin{itemize}
    \item Convolution layer: In this layer, a set of learnable filters (also known as kernels) convolve over the input data, extracting local features and generating feature maps through element-wise multiplications and summations. This process allows the network to capture spatial patterns in the data. For example, consider an input volume with dimensions 6x6x3 (height, width, channels) and two convolutional filters with sizes of 3x3x3. Each kernel is initialized with learnable weights. Each filter slides across the input, performing element-wise multiplication with the corresponding input values in its receptive field (3x3x3 region), described in Fig. \ref{fig:conv}. The results are summed to produce a single value, creating feature maps. With no padding and a stride of 1, each filter generates an output feature map of size 4x4. The final output shape is 4x4x2.

    \begin{figure}[H]
    
        \centering   \includegraphics[width=0.99\textwidth]{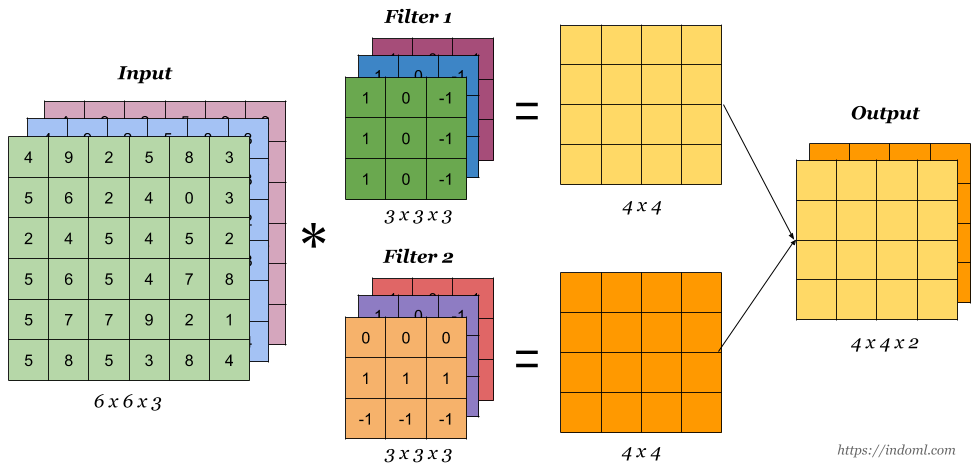}
        \caption[A convolution layer with two kernels with the shape of 3x3x3]{A convolution layer with two kernels with the shape of 3x3x3 (Source: Internet)}
        \label{fig:conv}
    \end{figure}
    
    \item Activation Layer: After the convolution operation, an activation function is applied element-wise to the feature maps. Common activation functions include ReLU (Rectified Linear Unit) and Tanh (Hyperbolic Tangent). The ReLU activation function is defined as $f(x) = \max(0, x)$, where $x$ represents the input value. The Tanh activation function is defined as $f(x) = \frac{{e^x - e^{-x}}}{{e^x + e^{-x}}}$, where $x$ represents the input value. These activation functions introduce non-linearity, enabling the network to model complex relationships and make the network more expressive.
    
    \item Pooling layer: This layer aims to downsample the spatial dimensions of the feature maps while retaining important information. The feature map is divided into non-overlapping regions (windows), and the maximum or minimum, or average value within each window is selected as the representative value for that region. Let's take an example of max pooling with the stride of 2 in Fig. \ref{fig:pooling}. This downsampling reduces the spatial size of the feature map by a factor of 2. The process of max pooling helps reduce computation and control overfitting by retaining the most prominent features while discarding less relevant information.

    \begin{figure}
        \centering
    \includegraphics[width=0.8\textwidth]{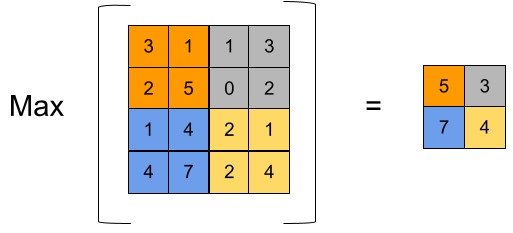}
        \caption[Maxing pooling layer with the stride of 2]{Maxing pooling layer with the stride of 2 (Source: Internet)}
        \label{fig:pooling}
    \end{figure}
    
    \item Normalization layer: Normalization layers, such as Batch Normalization, are often inserted after convolution or fully connected layers. They standardize the activations of a layer, which helps to stabilize and accelerate the training process by reducing internal covariate shifts.
    \item Fully Connected layer: This layer connects all the neurons from the previous layer to the neurons in the subsequent layer, allowing the network to learn complex combinations of features. Each neuron in the fully connected layer receives inputs from all the neurons in the previous layer and applies a weight to each input before passing it through an activation function.
\end{itemize}

The learning process involves computing gradients using backpropagation and updating the weights and biases of the CNN using optimization algorithms such as stochastic gradient descent (SGD) or Adam. 

\subsection{U-Net}
U-Net first introduced in \cite{ronneberger2015u} is a CNN architecture particularly solving the image segmentation problems which necessitate mapping from the original image to a mask with the same height and width. Fig. \ref{fig:segmentation_example} represents an input and its respective mask for low-grade glioma segmentation.

\begin{figure}[H]
    \begin{subfigure}{.5\textwidth}
  \centering
  \includegraphics[width=.8\linewidth]{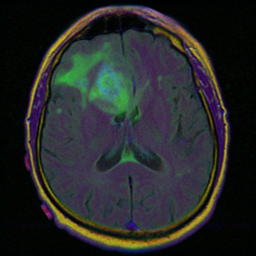}
  \caption{Input}
  \label{fig:sfig1}
    \end{subfigure}%
    \begin{subfigure}{.5\textwidth}
  \centering
  \includegraphics[width=.8\linewidth]{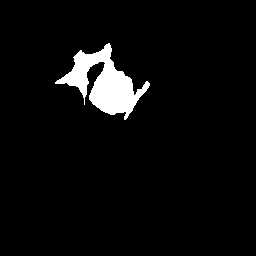}
  \caption{Mask}
  \label{fig:sfig2}
    \end{subfigure}
\caption[An input and its respective mask for low-grade glioma segmentation]{An input and its respective mask for low-grade glioma segmentation (Source: Internet)}
\label{fig:segmentation_example}
\end{figure}

The U-Net architecture consists of an encoding path and a decoding path, which are symmetrically connected. The encoding path captures the context and extracts high-level features from the input image, while the decoding path enables precise localization by upsampling the feature maps.

The encoding path of U-Net is composed of multiple convolutional and pooling layers, typically arranged in a contracting manner. Each convolutional layer applies a set of learnable filters to the input feature maps, followed by an activation function (e.g., ReLU) to introduce non-linearity. Pooling layers, commonly using max pooling, downsample the feature maps, reducing their spatial dimensions while retaining important information.

The decoding path of U-Net is symmetric to the encoding path. It consists of upsampling and concatenation operations, followed by convolutional layers. The upsampling operation increases the spatial resolution of the feature maps, while the concatenation operation combines the feature maps from the encoding path with the upsampled feature maps. This allows the network to make precise localizations by integrating both high-level and low-level features. 

The U-Net's last layer is a pixel-wise classification layer that produces a mask with the same size as the input image. This mask represents the segmentation map, where each pixel corresponds to a specific class or object. In certain binary classification cases, the final layer can be omitted, and the decoder's output can be used as the probability map. An illustration of the conventional U-Net architecture is shown in Figure~\ref{fig:unet_architecture}.

\begin{figure}
    \centering
    \includegraphics[width=\textwidth]{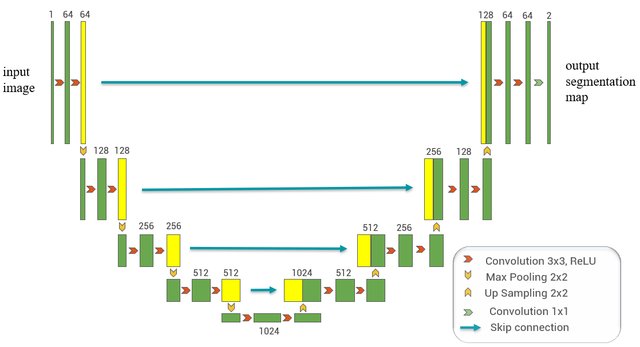}
    \caption[An example of the Unet architecture]{An example of the Unet architecture in the paper \cite{guo2020cloud}}
    \label{fig:unet_architecture}
\end{figure}

\newpage
\chapter{A GENERAL DEC-POSMDP FORMULATION FOR NLM-CTC}
\label{chap:mdp}
\label{chap:network-model}
\section{Network model}

The thesis employs a time-continuous system, where $t$ represents the elapsed time since the start of the network at $t = 0$. The Wireless Rechargeable Sensor Network (WRSN) in the thesis is composed of four fundamental components, including targets, a base station, sensors, and mobile chargers. An example of a WRSN is shown in Figure~\ref{fig:model} while the overall notations used in the thesis are presented in Table~\ref{tbl: tbl_notations}. A detailed description of the components is provided below.
\begin{figure}[h]
    \centering
    \includegraphics[width=\textwidth]{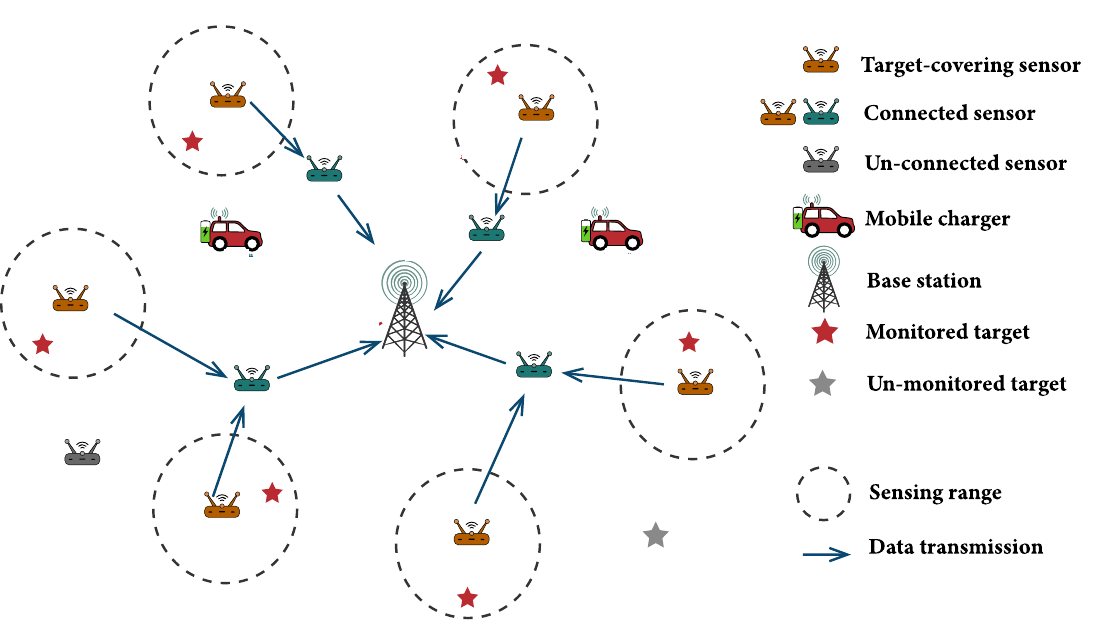}
    \caption{The network model}
    \label{fig:model}
\end{figure}
\begin{itemize}

    \item $N$ sensors indicated as $\mc S = \{S_1, S_2, \dots, S_N\}$ are scattered across the area of interest. These sensors are responsible for monitoring targets and forwarding information to the base station. The thesis considers a WRSN with homogeneous sensors, meaning that all sensors have similar physical specifications, which are listed in Table~\ref{tbl: tbl_notations}. Each sensor is equipped with a battery that has a maximum capacity of $e^{\text{max}}$ and a threshold of $e^{\text{th}}$. If the energy of a sensor decreases below $e^{\text{th}}$, the sensor ceases operating, and it cannot be reactivated even if it is charged by a mobile charger. Additionally, a sensor's energy cannot exceed $e^{\text{max}}$ while receiving energy from MCs. A sensor is responsible for monitoring all targets within a circle centered at its location, with a radius of $r_s$. A communication link is established between two sensors or between a sensor and the base station if their distance is at most $r_c$. In this work, we use the energy consumption model \cite{heinzelman2002application}, and the greedy geography routing protocol  \cite{karp2000gpsr} to simulate the network deployment process. Specifically, each sensor node continuously consumes energy for three essential tasks: environmental sensing, data receiving, and data transmitting. However, the sensing task's energy expenses are negligible compared to transmitting and receiving energy. Let $\Upsilon_r$ and $\Upsilon_t$ be the sensor node's energy consumed to receive and transmit $b$ data bits, respectively. They are defined as follows:
    \begin{equation}\label{ereceive}
\Upsilon_{r} = b\text{}\varepsilon_{\text{elec}}\text{          }(J/b)
        \end{equation}
    \begin{equation}\label{etransfer}
\Upsilon_{t} = 
  \begin{cases} 
   b \varepsilon_{\text{elec}} + b\varepsilon_{\text{fs}} \times d_{S_r, S_t}^{2} & \text{if  }  d_{S_r, S_t}<d_{0} \\
   b\varepsilon_{\text{elec}} + b\varepsilon_{\text{mp}} \times d_{S_r, S_t}^{4} & \text{if  }  d_{S_r, S_t} \geq d_{0},
  \end{cases}
\end{equation}
    where $\varepsilon_{\text{elec}}$ is a constant that represents the energy consumed by the transmitter or the receiver circuits for data per bit; $\varepsilon_{\text{fs}}$ and $\varepsilon_{\text{mp}}$ indicate amplifier energy in the free space, and multi-path communication, respectively; $d_{S_r, S_t}$ denotes the Euclidean distance between the transmitter ($S_t$) and receiver ($S_r$) while $d_{0}$ is the maximum threshold distance. We set our model based on \cite{heinzelman2002application} as follows: $\varepsilon_{\text{elec}} = 50 (nJ)$, $\varepsilon_{\text{fs}} = 10(pJ)$, and $\varepsilon_{\text{mp}}=1.3\times 10^{-3}(pJ)$. 
    The greedy geographic routing protocol is used in the model, where each sensor node selects the next hop as the nearest neighboring node to the BS. The protocol assumed that every node always knows its location and its one-hop neighbors. In order to find the next node, the current node will look up its neighbor table and determine the node with the smallest distance to the BS. Then, the BS will calculate each sensor node's average energy consumption rate based on their energy consumption and an estimation method following real-time introduced in \cite{zhu2018adaptive}.
    
    \item $L$ targets denoted as $\mc T = \{T_1, T_2, \dots, T_L\}$ necessitate surveillance. Each target $T_i$ is monitored by a set of sensors, indicated as $\mc {MS}_i$. These sensors collect information about the target $T_i$, which is then transmitted through intermediate sensors to the base station. It is important to note that failure to monitor the target $T_i$ can occur if either all sensors in the set $\mc {MS}_i$ run out of energy or there is no communication link between the surviving sensors in $\mc {MS}_i$ and the base station.
    
    \item A base station denoted as $BS$, resides at the center of the network. The base station serves as an information center which is the ending point of all data transmission lines. The base station maintains a special array, denoted as $\mc {MNT}_t = \{MNT_{t, 1}, MNT_{t, 2}, \dots, MNT_{t, L}\}$. The element $MNT_{t, i}$ is an integer variable representing the status of the target $T_i$ at time $t$. The values of $MNT_{t, i}$ are 0 and 1. $MNT_{t, i} = 0$ means the disconnection of the target $T_i$ to the base station and vice versa. The thesis adopts a strict setting of targets' surveillance, hence, as long as any target disconnect from the base station, the network is dead. The variable $F_t$ represents the elapsed time from the deployment of the network till the death of the network. Our objective in this thesis is maximizing the elapsed time from the deployment of the network till its dead is denoted as the network lifetime that is equivalent to $F_0$. Mathematically, $F_0$ can be calculated as Formula~\ref{eq:lifetime}, while $F_t = F_0 - t$.
    \begin{equation}
        \label{eq:lifetime}
        F_0 = \min_{t} \quad t \quad \text{s.t.} \quad t > 0, \quad \prod_{i=0}^L MNT_{t, i} = 0.
    \end{equation}
    However, it is intractable to calculate $F_0$ without a simulation or a real-world operation of the network due to the uncertainty during the network's operation and the ambiguous form of $MNT_{t, i}$. Besides, the base station is also the charging station for mobile chargers (MCs). In this thesis, the base station employs the battery replacement mechanism. Hence, as long as an MC arrives at the BS, it replenishes fully its battery without any waiting time.

    \item $M$ mobile chargers denoted as $\mc {MC} = \{MC_1, MC_2, \dots, MC_M\}$. All mobile chargers are identical concerning physical specifications which are presented in Table~\ref{tbl: tbl_notations}. The movement of each $MC_i$ around the network is in the ideal condition in which obstructions and terrain characteristics are omitted. This work extends the multi-node charging energy model presented in \cite{he2013demand}. Assuming that the distance between an MC and a sensor is $d$, the energy per second the sensor receives from that MC follows the formulas~\eqref{eq:friis}. It is worth noting that an MC only charges sensors within the distance of $r_{\text{charge}}$.
    \begin{equation}
    \label{eq:friis}
    P^{\text{charge}}(d) = 
  \begin{cases} 
   0 & \text{if  }  d > r_{\text{charge}}  \\
   \frac{\alpha}{(d + \beta)^2} & \text{if  }  d \leq r_{\text{charge}},
  \end{cases}
\end{equation}
    where $\alpha$ and $\beta$ are constants depending on the hardware of the charger and energy-received device, respectively. 
\end{itemize}

\begin{table}[]
\caption{Description of notations}
\label{tbl: tbl_notations}
\begin{tabular}{|l|l|l|}
\hline
Subjects & \textbf{Notation} & \textbf{Description} \\ \hline
\multirow{3}{*}{Network} & $ d_{A, B} $ & \begin{tabular}[c]{@{}l@{}}The Euclidean distance between \\ two locations $A$ and $B$\end{tabular} \\ \cline{2-3} 
 & $t$ & \begin{tabular}[c]{@{}l@{}}The elapsed time since \\ the deployment of the network\end{tabular} \\ \cline{2-3} 
 & $F_t$ & \begin{tabular}[c]{@{}l@{}}The remaining time from the time t\\ till the dead of the network\end{tabular} \\ \hline
\multirow{2}{*}{Target} & $\mc T = \{T_1, T_2, \dots, T_L\}$ & The set of all targets \\ \cline{2-3} 
 & $\mc {MS}_i$ & \begin{tabular}[c]{@{}l@{}}The set of all sensors\\  monitoring the target $T_i$\end{tabular} \\ \hline
\multirow{2}{*}{\begin{tabular}[c]{@{}l@{}}Base\\ station\end{tabular}} & BS & The base station \\ \cline{2-3} 
 & $\mc {MNT}_t = \{MNT_{t, 1}, \dots, MNT_{t, L}\}$ & \begin{tabular}[c]{@{}l@{}}The special array representing\\  the status of targets at time $t$\end{tabular} \\ \hline
\multirow{8}{*}{Sensors} & $\mc S = \{S_1, S_2,\dots ,S_N\}$ & The set of all sensors \\ \cline{2-3} 
 & $r_c$ & The sensor's communication radius \\ \cline{2-3} 
 & $r_s$ & The sensor's sensing radius \\ \cline{2-3} 
 & $e^{\text{th}}, e^{\text{max}}$ & The sensor's threshold and capacity \\ \cline{2-3} 
 & $\mc {MT}_j$ & \begin{tabular}[c]{@{}l@{}}The set of all targets \\ monitored by the sensor $S_j$\end{tabular} \\ \cline{2-3} 
 & $L^{\text{sensor}}_j$ & The location of the sensor $S_j$ \\ \cline{2-3} 
 & $e_{t, j}$ & \begin{tabular}[c]{@{}l@{}}The remaining energy of \\  sensor $S_j$ at time $t$\end{tabular} \\ \cline{2-3} 
 & $p_{t, j}$ & \begin{tabular}[c]{@{}l@{}}The 100-second average \\ consumption rate of\\  sensor $S_j$ at time $t$\end{tabular} \\ \hline
\multirow{12}{*}{MCs} & $\mc {MC} = \{MC_1, MC_2, \dots, MC_M\}$ & The set of all mobile chargers \\ \cline{2-3} 
 & $\mc {MC}^{\text{charge}}_t$ & \begin{tabular}[c]{@{}l@{}}The set of charging mobile chargers\\ at time $t$\end{tabular} \\ \cline{2-3} 
 & $\mc {MC}^{\text{move}}_t$ & \begin{tabular}[c]{@{}l@{}}The set of moving mobile chargers\\ at time $t$\end{tabular} \\ \cline{2-3} 
 & $t^{\text{charge}}_{t, k}$ & \begin{tabular}[c]{@{}l@{}}The remaining charging time of\\ $MC_k$ at its location at time $t$\end{tabular} \\ \cline{2-3} 
 & $L^{\text{target}}_{t, k}$ & \begin{tabular}[c]{@{}l@{}}The destination of  $MC_k$ \\ at time $t$\end{tabular} \\ \cline{2-3} 
 & $E^{\text{max}}$ & The energy capacity of each MC \\ \cline{2-3} 
 & $r^{\text{charge}}$ & The charging range of each MC \\ \cline{2-3} 
 & $P_M$ & \begin{tabular}[c]{@{}l@{}}The per-second consumption rate \\ of an MC when traveling\end{tabular} \\ \cline{2-3} 
 & $V$ & The velocity of an MC \\ \cline{2-3} 
 & $\alpha$, $\beta$ & Parameters of charging model \\ \cline{2-3} 
 & $L^{MC}_{t, k}$ & The location of $MC_k$ at time $t$ \\ \cline{2-3} 
 & $E_{t, k}$ & \begin{tabular}[c]{@{}l@{}}The current energy of $MC_k$ \\ at time $t$\end{tabular} \\ \hline
\end{tabular}
\end{table}
\section{Dec-POSMDP formulation}
\label{rl-model}
In this thesis, the charging problem for maximizing the lifetime of WRSNs is modeled as a Decentralized partially observable Semi-Markov decision process which is defined by a set of elements $(M, \mathcal{S}, \mathcal{A}, \mathcal{O}, P, R, Z, \mc U, B, \gamma)$. Of all components, $P, Z$ is provided by the simulation or real-world conditions. The element $\gamma$ is only a problem's definition if the final objective coincides with the cumulative reward. However, in this problem, the reward function only encourages the increase of the final objective function, so $\gamma$ is treated as a parameter of the formulation. A detailed description of other elements is presented below:

\begin{itemize}
    \item $M$ is the number of agents that is equivalent to the number of mobile chargers.
    
    \item \textit{State space} ($\mc S$): At the time $t$, the state is the information of the whole network including all targets, all sensors, all mobile chargers, data transmission process, uncertain events affecting the operation of the network. The high-level generalization of the state leads to the intractability of capturing this ambiguous information.

    \item \textit{Joint atomic actions space} ($\mc A$): $\mc A = \mathcal{A}^1 \times \mathcal{A}^2 \times \dots \times \mathcal{A}^M$ is the joint action space, where $\mathcal{A}^k$ is the set of atomic actions available to mobile charger $k$. All mobile chargers are identical, therefore $\mathcal{A}^1 = \mathcal{A}^2 = \dots = \mathcal{A}^M$. $\mathcal{A}^k$ consists of two types of atomic actions including charging action and moving action. While the charging action means that the MC stays still and charges sensors around, the moving action incorporates a parameter $\theta$ which regulates the moving direction of the MC as $[\sin{\theta}, \cos{\theta}]$.

    \item \textit{Joint macro actions space} ($\mc U$): $\mc U = \mathcal{U}^1 \times \mathcal{U}^2 \times \dots \times \mathcal{U}^M$ is the joint action space, where $\mathcal{U}^k$ is the set of macro actions available to mobile charger $k$. With the same setting of atomic actions, all mobile chargers also have the same macro action space. The thesis adopts the macro actions space of $\R ^ 2 \times \R^+$. A macro action $u^i = [a, b, c]$, where $[a, b]$ denotes the location that the MC $i$ visits and $c$ is the charging time the MC spends at $[a, b]$. Assuming that at the time $t$, the MC $i$ starts committing the macro action $u^i$, and so on, the MC will follow a deterministic policy until finishing the task. Specifically, if committing this macro action makes the MC run out of energy, the MC firstly returns to the base station to replenish energy. Then the MC commits the macro action as normal. The MC would practice the moving action toward the charging location until reaching this point and then charge sensors around. Due to this deterministic definition of the macro action, it is nonessential to indicate explicitly the mathematical form of $B$ because the mobile charger would move to $[a, b]$ and then charge surrounding sensors in $c$ seconds, hence as soon as the MC finishes charging, the macro action cease. 

\item \textit{Joint observations space} ($\mc O$): $\mc O = \mathcal{O}^1 \times \mathcal{O}^2 \times \dots \times \mathcal{O}^M$ is the joint observations space, where $\mathcal{O}^k$ is the set of observation of mobile charger $k$. As mentioned in Section~\ref{chap:network-model}, each MC has full access to the information garnered at the BS. However, the observation space of each MC is still different because each MC is capable of self-recognizing its policy. A common approach is to flatten the information about sensors, mobile chargers into one vector. However, this approach suffers from scaling problems because the observation changes along with the deviation of the number of sensors and mobile chargers. The thesis proposes a novel approach to designing the observation for MC so that RL algorithms trained in this designed environment can adapt to the deviation of network scenarios. At the time $t$, the observation of the $MC_k$ composes of four functions whose input is a 2D vector, each representing an aspect of the network. The base function to build these functions is the 2D Gaussian kernel function, represented in Equation~\eqref{eq:Kx}, because it is a common $[0, 1]$ bounded function able to represent the shrink of function values along the square of the distance. The illustration of this function is in Figure~\ref{fig: gaussian-kernel}.
    \begin{equation}
        \mc K(x, x^{\prime}) = \exp\big( -\frac{1}{2 h_{\mc X}^2} \| x - x^{\prime} \|_2^2 \big),
        \label{eq:Kx}
    \end{equation}   
    where $x, x^{\prime} \in \R^2$, $h_{\mc X} > 0$ is the kernel width which a formulation parameter. At the time $t$, the observation of $MC_k$ is four functions as follows:
    \begin{figure}
        \centering
        \includegraphics[width=0.75\textwidth]{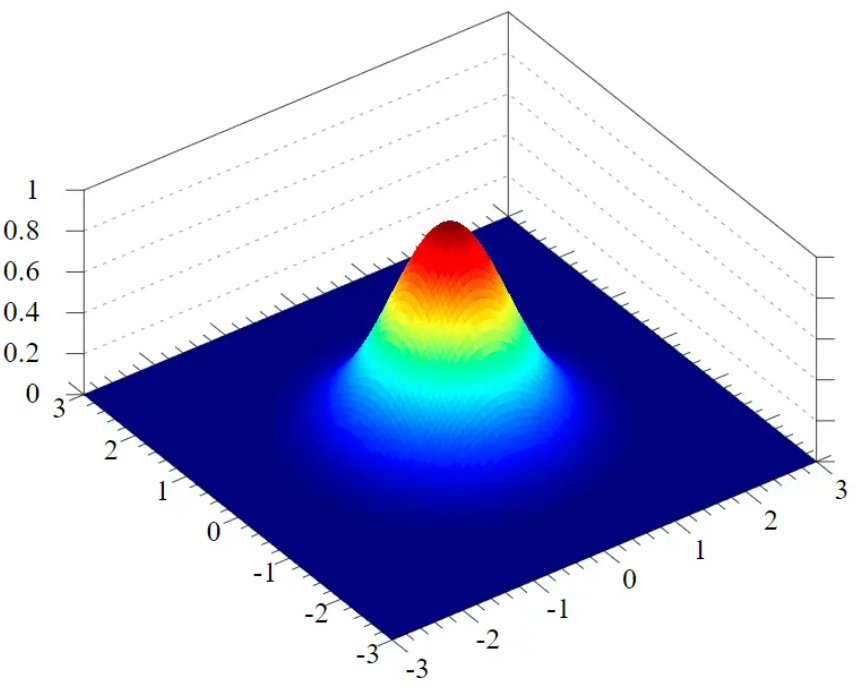}
        \caption{The visual representation of the 2D Gaussian kernel function }
        \label{fig: gaussian-kernel}
    \end{figure}
    \begin{itemize}
        \item The first function given in Equation \eqref{func:f1} captures the recent energy consumption rate and residual energy of each sensor. The function consists of multiple subfunctions, each corresponding to a sensor. A subfunction of the sensor $S_j$ assigns a higher value to the location of the sensor $L^{\text{sensor}}_j$, which gradually decreases as we move away from it. The subfunction of a sensor is inversely proportional to its estimated lifetime, represented by $\frac{{e_{t,j} - e^{\text{th}}}}{p_{t,j}}$. The lifetime is normalized in the function to aid the learning process. At time $t$, we observe that higher values are found near critical sensors that consume more energy and have lower energy levels compared to others. Suitable charging algorithms are likely to encourage mobile chargers to visit these locations because they are near sensors responsible for multiple network tasks. This function also guides MCs to avoid areas where MCs can not charge any sensors.
        \begin{equation}
            \label{func:f1}
            F_{1}(x) = \sum_{j = 1}^{N} \frac{p_{t,j}}{\frac{\alpha}{\beta^2}} \times \frac{e^{\text{max}} - e^{\text{th}}}{e_{t,j} - e^{\text{th}}} \times K(x, L^{\text{sensor}}_j), \text{where } x \in \R^2.
        \end{equation}
        Noted that $t$ is a constant because the observation of $MC_k$ at the time $t$ is currently considered.
        
        \item The second function, represented by Equation \eqref{func:f2}, captures the current location and remaining energy of the mobile charger currently making decisions. It assigns a higher value to locations near the MC's current position, encouraging the MC to prioritize visiting neighboring areas to minimize movement costs. Additionally, the function provides information about the MC's energy status, indicating when it is running low on energy and needs to make more strategic decisions.
        \begin{equation}
            \label{func:f2}
            F_{2}(x) = \frac{E_{t, k}}{E^{\text{max}}}  \times K(x, L^{MC}_{t, k}), \text{where } x \in \R^2.
        \end{equation}
        Noted that $t \in \R^+, MC_k \in \mc MC$, and $t, k$ are constants because the observation of $MC_k$ at the time $t$ is currently considered.
        
        \item The third function, shown in Equation~\eqref{func:f3} supports an MC to capture the macro action of other MCs in charging. This function suggests an MC avoids areas severed by too many MCs and is directed to vital sensors without energy supplement from any MCs.  

        \begin{equation}
            \label{func:f3}
            F_{3}(x) = \sum_{MC_{k^{\prime}} \in \mc {MC}^{\text{move}}_t} 
            \frac{t^{\text{charge}}_{t, k^{\prime}}}{\frac{e^{\text{max}} - e^{\text{th}}}{\frac{\alpha}{\beta^2}}} \times K(x, L^{MC}_{t, k^{\prime}}), \text{where } x \in \R^2.
        \end{equation}
        Noted that $t \in \R^+,$ and $t$ is a constant because the observation of $MC_k$ at the time $t$ is currently considered. 
        
        \item The fourth function, represented by Equation \eqref{func:f4}, assists an MC in capturing the macro actions of other MCs in terms of movement. This function provides information to the MC about areas that are likely to be served by other MCs in the near future. By avoiding these areas, the MC can make more efficient decisions and prioritize visiting locations that are not being actively covered by other MCs. However, if other MCs are expected to take a significant amount of time to arrive at a critical location, this function may suggest that nearer MCs should support this area.

        \begin{equation}
            \label{func:f4}
            F_{4}(x) = \sum_{MC_{k^{\prime}} \in \mc {MC}^{\text{move}}_t} 
            \frac{d(L^{MC}_{t, k}, L^{\text{target}}_{t, k})}{d^{\text{avg}}} \times K(x, L^{MC}_{t, k^{\prime}}), \text{where } x \in \R^2.
        \end{equation}
        Noted that $t \in \R^+,$ and $t$ is a constant because the observation of $MC_k$ at the time $t$ is currently considered. The term $d^{\text{avg}}$ is the average distance between sensors in the network.
    \end{itemize}

    Four function captures essential information for each MC to make its own decision. Consider $H_0, H_1$ as the lower bound and the upper bound of the x-axis location of all sensors while $W_0, W_1$ as these of the y-axis. It is only important to consider the region $[W_0, W_1] \times [H_0, H_1]$ instead of $\R^2$ for the value of $x$. However, it is still intractable to represent these functions as input for any models. Hence, each range is divided into $T$ equal segments, and the center point of each segment represents this segment. For example, if dividing the range $[H_0, H_1]$ into $T$ segments, it is possible to represent the range as a set of $T$ discrete points ${h_1, h_2, ..., h_T}$, where each $h_i$ represents the center point of the $i$-th segment, specifically $h_i = (i - 0.5) \times \frac{H_1 - H_0}{T}$. For each function $F_{i}$, a 2D vector of shape $T \times T$ is formed that the element at $j$, $k$ is denoted $RF_{i}[j][k] = F_{i}([w_j, h_k])$. An example of four 2D vectors respective to four functions is illustrated in Figure~\ref{fig:four_images}. These four vectors are stacked together to create the observation of $MC_k$ at time $t$. Each observation has the shape of $4 \times T \times T$. In this formulation, $T$ is a parameter.
    
    Finally, each function is visualized as a heat map, which is cropped to match the border of the network. An example of four heat maps is shown in Figure \ref{fig:four_images}. The resolution of each heat map is set to 100x100, meaning that the observation for each MC at time $t$ has a shape of 4x100x100. This representation allows the MCs to efficiently process and interpret the spatial information of the network.

    \begin{figure}
    \centering

    \begin{subfigure}[b]{0.49\textwidth}
        \includegraphics[width=\textwidth]{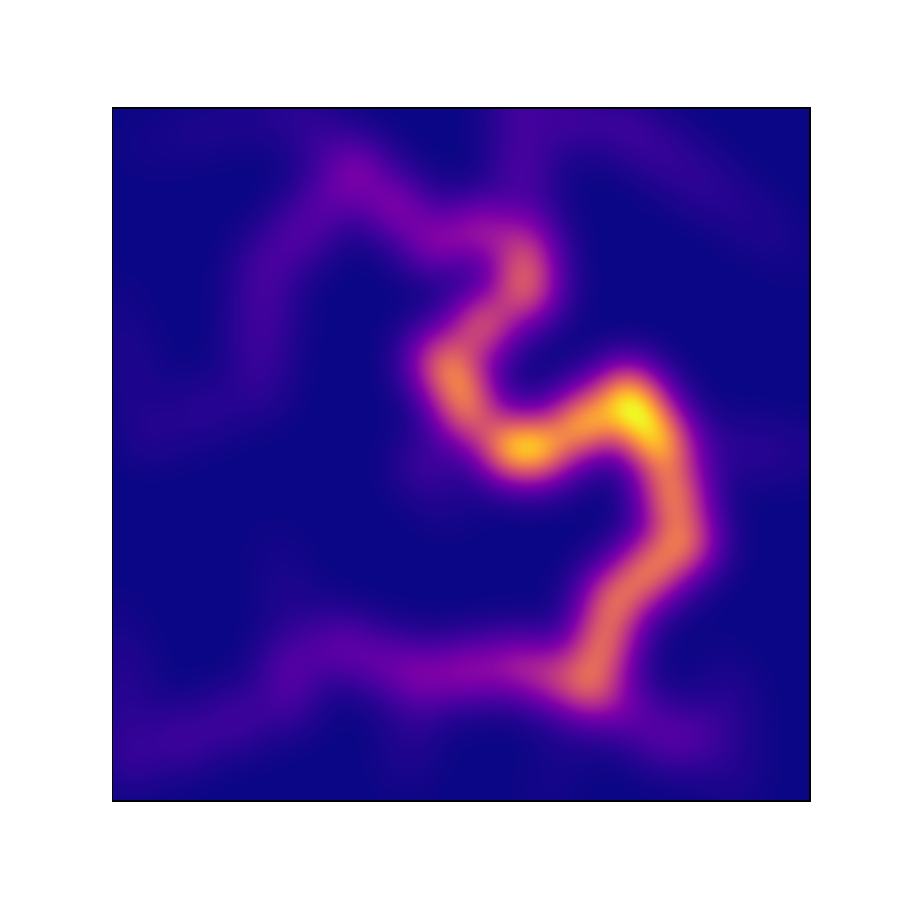}
        \caption{The first layer is clear to see which regions contain critical sensors and the data transmission in the network.}
        \label{fig:layer1}
    \end{subfigure}
    \hfill
    \begin{subfigure}[b]{0.49\textwidth}
        \includegraphics[width=\textwidth]{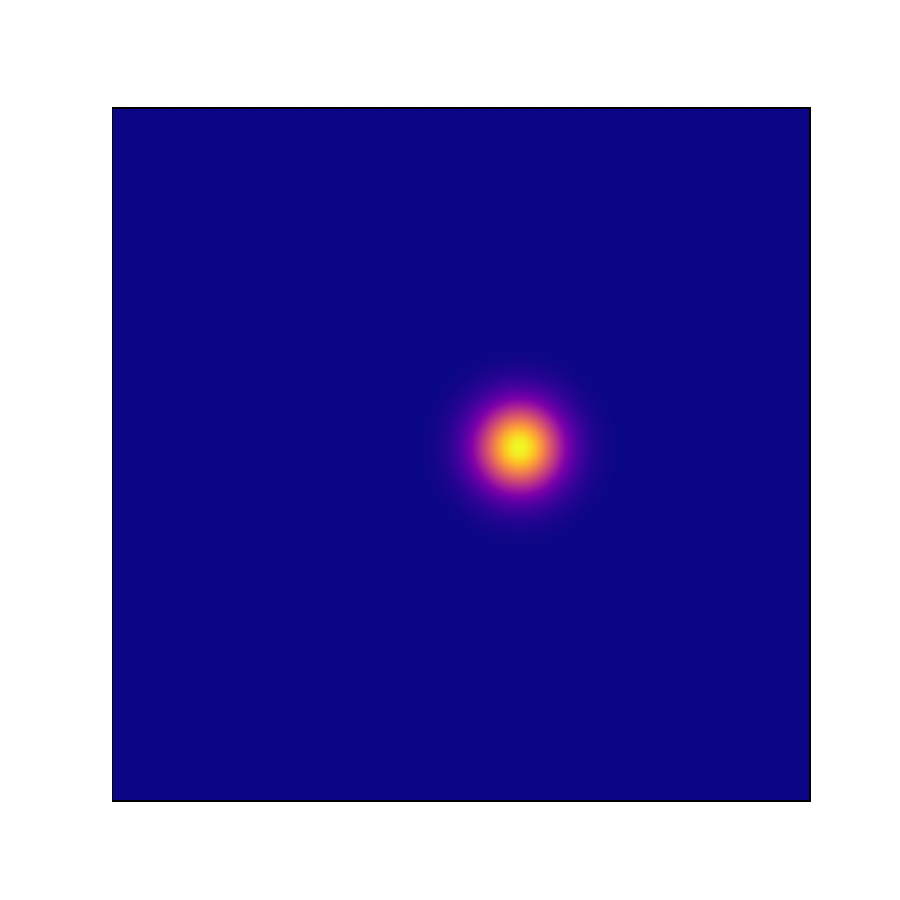}
        \caption{The third map is about mobile chargers currently charging, helping the mobile charger to avoid visiting the same region.}
        \label{fig:layer2}
    \end{subfigure}

    \begin{subfigure}[b]{0.49\textwidth}
        \includegraphics[width=\textwidth]{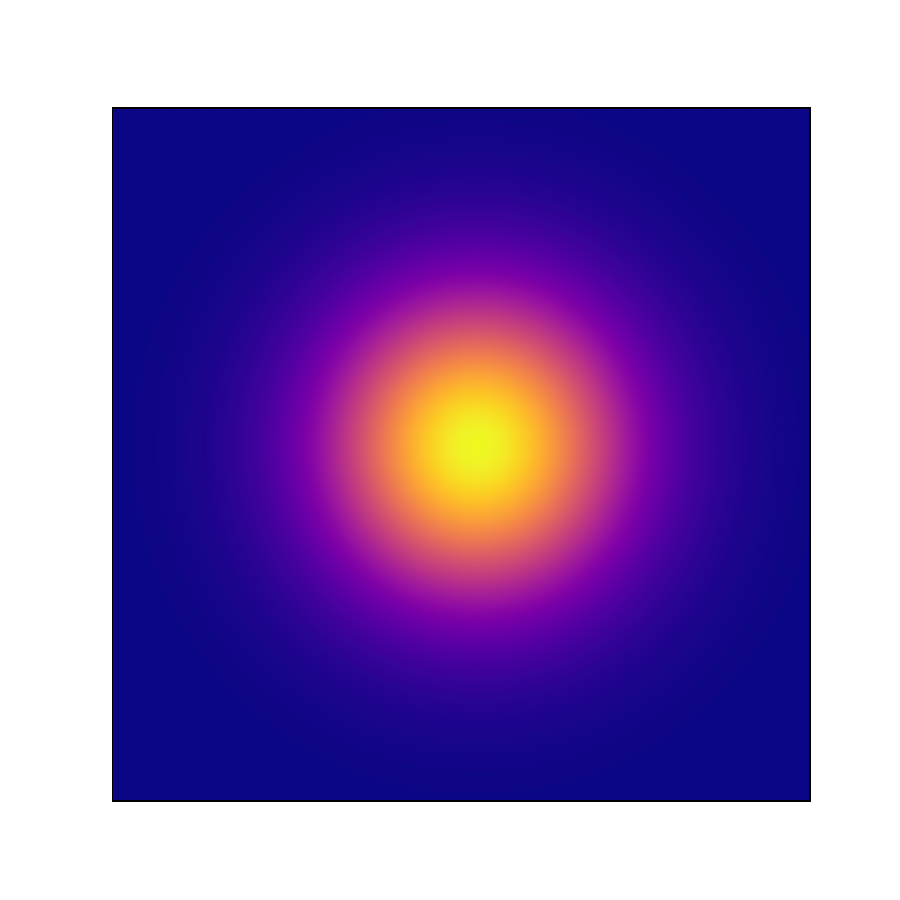}
        \caption{The second layer illustrates the location of the MC, with brighter regions indicating proximity to the MC. This map highlights areas that may require excessive energy expenditure for movement.}
        \label{fig:layer3}
    \end{subfigure}
    \hfill
    \begin{subfigure}[b]{0.49\textwidth}
        \includegraphics[width=\textwidth]{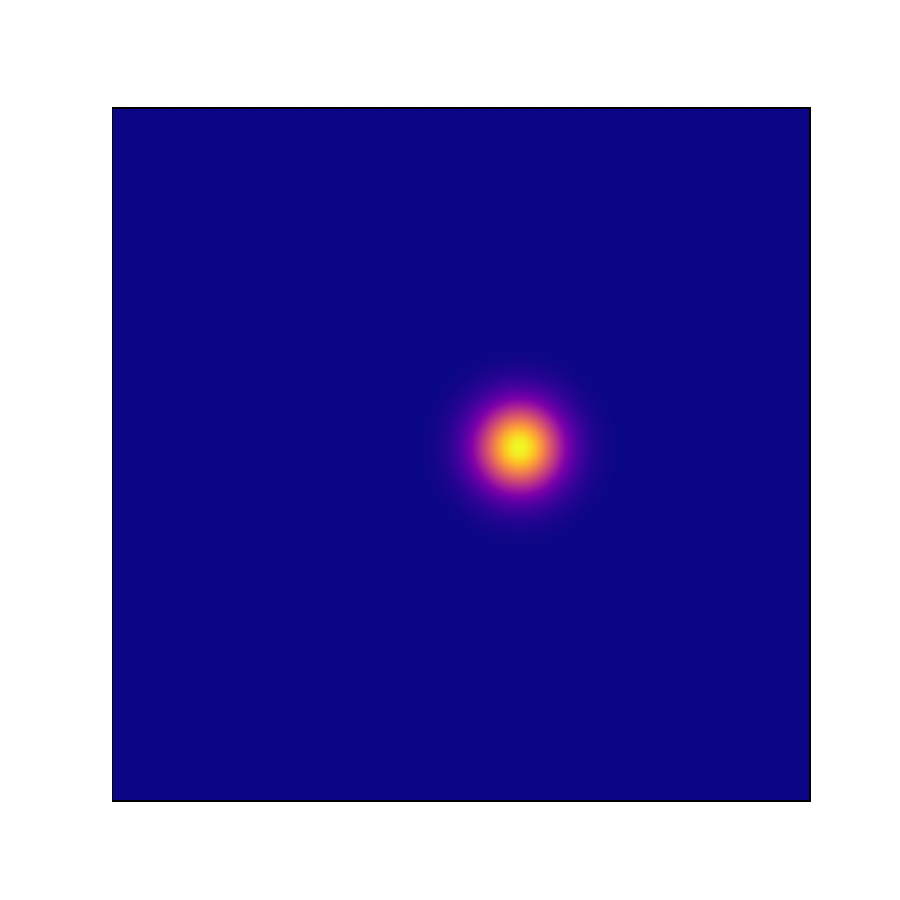}
        \caption{The fourth map indicates the target locations of mobile chargers currently in motion. This map encourages the mobile charger to collaborate and support each other when the moving distance of other MCs is too long.}
        \label{fig:layer4}
    \end{subfigure}

    \caption[An example of an observation]{An observation in a system with three mobile chargers. At this moment, one mobile charger is making a decision, while another is currently charging, and the third is in motion, moving to its destination}
    \label{fig:four_images}
\end{figure}

    \item \textit{Objective} ($Obj$): This is an additional term that is used to define the reward function. The final objective of the thesis is maximizing the lifetime of the network determined by the formula~\eqref{eq:lifetime}. However, this formula is intractable to calculate without waiting until when the network is dead. Therefore, the thesis proposes a method to estimate the remaining lifetime of a network at time $t$. 

    The estimation begins by constructing an undirected graph $\mathcal{G} = (\mathcal{V}, \mathcal{E})$ that represents the network's geometry at time $t$. Here, each node $x_j \in \mathcal{V} = BS \cup \mathcal{S}$ corresponds to either a sensor or the base station. An edge $(x_j, x_{j^{\prime}}) \in \mathcal{E}$ exists if the distance between $x_j$ and $x_{j^{\prime}}$ falls within the communication range of the sensor $r_c$. To each node $x_j$ in $\mathcal{V}$, the method assigns a weight $w(x_j)$, which represents its estimated lifetime. For the base station $x_0$, $w(x_0) = \infty$, while for the sensor $x_j$, $w(x_j) = \frac{e_{t, j} - e^{\text{th}}}{p_{t, j}}$. Each path in $\mathcal{G}$ from the node $x_0$ (base station) to any node $x_j \in \mathcal{V}$ represents a connection line between the base station and sensor $S_j$. The weight of each path is determined by the minimum weight among the nodes present in the path. This weight concept intuitively signifies the lifetime of the corresponding connection line in the network, as a connection line ceases to exist if any node along the path runs out of energy. Therefore, the remaining connection time of sensor $S_j$ to the base station, denoted as $CT(x_0, x_j)$ , is represented by the largest weight of all paths from $x_0$ to $x_j$. To compute $CT(x_0, x_j) \forall x_j \in \mc V$, the thesis proposes Algorithm~\ref{al:connection time estimates} with the idea, that the system will maintain an estimate $d[x_j] \forall x_j \in \mc V$ of the remaining connection time. 

   \begin{algorithm}
        \SetKwInput{Input}{Input}
        \SetKwInput{Output}{Output}
        $\forall x \in \mc V$, $d[x] \leftarrow 0 //$ set initial values
        
        $d[x_0] \leftarrow \infty //$ set initial value of the base station node
        
        $F = \mc V //$ $F$ is the set of nodes $x$ that are not yet to achieve $CT(x_0, x)$
        
        $D \leftarrow \emptyset //$  $D$ is the set of nodes $x$ that have achieved $CT(x_0, x)$

        \While{$F \neq \emptyset$}{
        $x \leftarrow$ element in $F$ with maximum $d[x]$

        \For{$(x, y) \in \mc E$} {
            $d[y] \leftarrow \max \{d[y], \min\{d[x], w(y)\}\}$
            }

        $F \leftarrow F \setminus x$ 
        
        $D \leftarrow D \cup x$
        }
        \caption{Calculating the remaining connection time of sensors}
        \label{al:connection time estimates} 
    \end{algorithm}

    We will prove the correction of Algorithm~\ref{al:connection time estimates} through the two following propositions.

    \begin{proposition}
        \label{prop:lower bound}
        At any iteration, $d[x_j] \leq CT(x_0, x_j).$
    \end{proposition}

    \begin{proof}[Proof of Proposition~\ref{prop:lower bound}]
        If $d[x] = 0$, the proposition is true because the cost of a path is always non-negative. If $d[x]>0$ at any point in time, we prove that this $d[x]$ corresponds to a path from $x_0$ to $x$ whose weight is at most the largest weight of all paths from $x_0$ to $x$. Consequently, $d[x] \leq CT(x_0, x)$. We proceed with formal proof by induction:
        
        \begin{itemize}
            \item The base case is when $x=x_0$, where $d[x_0]=CT(x_0, x_0)=+\infty$, and all other $d[x] \forall x \in \mc V \setminus x_0$ are set to $0$. Thus, the claim holds initially.
            
            \item We assume that the claim holds for all vertices up to a certain step in the algorithm. When the algorithm updates $d[y]$ to $max \{d[y], \min\{d[x], w(y)\}\}$ for some vertex $x$, by the induction hypothesis, there exists a path from $x_0$ to $x$ with a weight of $d[x]$, and an edge $(x, y)$. This implies that there is a path from $x_0$ to $y$ with a weight of $\min\{d[x], w(y)\}$. Since we have previously established that $d[x]$ is at at most the largest weight of all paths from $x_0$ to $x$, the induction argument is complete.
        \end{itemize}
    \end{proof}

    \begin{proposition}
        \label{prop: fixed after append}
        $d[x]$ does not change after x is added to $D$.
    \end{proposition}
    
    \begin{proof}[Proof of Proposition~\ref{prop: fixed after append}]

        It is worth noting that proving this proposition suffices to indicate the correctness of the Algorithm~\ref{al:connection time estimates}. Because $d[x]$ is kept intact after $x$ is appended to $D$, which is proved as follows:

        \begin{itemize}
            \item The assertion $d[x] \geq d[y] \forall y \in F$ is true at all points after $x$ is inserted into $D$. To prove this, we assume for a contradiction that at some point there exists an $y \in F$ such that $d[y]>d[x]$ and before $d[y]$ was updated, $d[y^{\prime}] \leq d[x]$ for all $y^{\prime} \in F$. But then when $d[y]$ was changed, it was due to some neighbor $y^{\prime}$ of $y$ in $F$. However, $d[y^{\prime}] \leq d[x] < d[y]$, so $\max \{d[y], \min\{d[y^{\prime}], w(y)\}\} = d[y]$, which means $d[y]$ is not updated. In conclusion, we get a contradiction.
            
            \item $d[x]$ is never changed again after $x$ is added to $D$: the only way it could be changed is if for some node $y \in F$, $\min\{d[y], w(x)\} < d[x]$ but this can not happen due to the assertion stated above that $d[x] \geq d[y] \forall y \in F$ at all points.
        \end{itemize}
    \end{proof}
    \begin{proposition}
        \label{prop: optimal solution}
        When node x is placed in D, $d[x] = CT(x_0, x)$.
    \end{proposition}

    \begin{proof}[Proof of the Preposition~\ref{prop: optimal solution}]
        This claim is proved by induction on the order of placement of nodes into $D$. For the base case, $x_0$ is placed into $D$ where $d[x_0]=CT(x_0, x_0)=+\infty$, so initially, the claim holds.

            For the inductive step,  assume that for all nodes $y$ currently in $D, d[y]=CT(x_0, y)$. Let $x$ be the node that currently has the maximum $d[x]$ in $F$ (this is the node about to be moved from $F$ to $D$). We will show that $d[x]=CT(x_0, x)$, and this will complete the induction.
            
            Let $p$ be the path with maximum weight from $x_0$ to $x$. Suppose $z$ is the node on $p$ closest to $x$ for which $d[z]=CT(x_0, z)$. We know $z$ exists since there is at least one such node, namely $x_0$, where $d[x_0]=CT(x_0, x_0)$. By the choice of $z$, for every node $y$ on $p$ between $z$ to $x$, $d[y]<CT(s, y)$ based on Proposition~\ref{prop:lower bound}. Consider the following options for $z$.
            \begin{itemize}
                \item If $z=x$, then $d[x]=CT(x_0, x)$ and we are done.
                \item Suppose $z \neq x$. Then there is a node $z^{\prime}$ after $z$ on $p$. Note that $z^{\prime}$ can be $x$. We know that $d[z] = CT(x_0, z) \geq CT(x_0, x) \geq d[x]$. The first $\geq$ inequality holds because subpaths of a path with maximum weight are also paths with maximum weight so that the prefix of $p$ from $s$ to $z$ has weight $CT(s, z)$. In addition, following the definition of the path's weight, the weight of a subpath of a path is always greater or equal to the weight of this path. The subsequent $\geq$ holds by Proposition~\ref{prop:lower bound}. We know that if $d[z]=d[x]$ all of the previous inequalities are equalities and $d[x]=CT(x_0, x)$ and the claim holds.
            \end{itemize}
            
            Finally, towards a contradiction, suppose $d[z]>d[x]$. By the choice of $x \in F$ we know $d[x]$ is the maximum estimate of nodes in $F$. Thus, since $d[z]>d[x]$, we know $z \notin F$ and must be in $D$, the finished set. This means the edges out of $z$, and in particular $\left(z, z^{\prime}\right)$, were already relaxed by our algorithm. But this means that $d[z^{\prime}] \geq \min\{CT(x_0, z), w(z^{\prime})\}=CT(s, z^{\prime})$, because $z$ is on the path $p$ from $x_0$ to $z^{\prime}$, and the $d[z^{\prime}]$ of $z^{\prime}$ must be correct. However, this contradicts $z$ being the closest node on $p$ to $x$ meeting the criteria $d[z]=CT(x_0, z)$. Thus, our initial assumption that $d[z]>d[x]$ must be false and $d[x]$ must equal $CT(x_0, x)$.
    \end{proof}

    The remaining connection time of sensor $S_j$ is $CT(x_0, x_j)$. The target $T_i$ is out of surveillance only when all sensors in the set $\mc {MS}_i$ are disconnected from the base station. Moreover, the network is dead as long as the disconnection of a target, so it is possible to estimate the remaining lifetime of a WRSN at time $t$ as follows:
    \begin{equation}
        \hat{F_t} = \min\limits_{x_i \in \mc T}~{\max\limits_{x \in \mc {MS}_i}~\left\{{d_t}(x_0, x)\right\}}.
    \end{equation}
    
    \item \textit{Reward} ($\mc R$): The reward of a macro action lasting from $t_1$ to $t_2$ ($t_2 > t_1$) is defined as the improvement of the estimate of the network lifetime:
    \begin{equation}
        \label{reward}
        R^{\text{GE}}_{t_1, t_2} = (\hat{F}_{t_1} - \hat{F}_{t_2}) - (t_2 - t_1).
    \end{equation}
    However, because during the time between $t_1$ and $t_2$, other agents also operate, it is uncertain to confirm whether a macro action of an agent brings benefit or not. Hence, if the system just adopts the common reward for all agents, it is likely to raise the problem of wrong evaluation. For instance, if an agent commits a really bad macro action but other agents practice so effectively, the general reward still be high, leading to a still high evaluation for the bad action. For this reason, it is essential to create an individual reward term that evaluates the effectiveness of an agent's action. The thesis proposes an individual term that promotes the charging energy for more critical sensors as follows:
    \[
        R^{\text{EX}} = \int_{t=t_1}^{t=t_2} \sum_{j=1}^{N} \left( \frac{{P^{\text{charge}}(d(L^{\text{sensor}}_j, [a, b]))}}{\frac{e_{t, j} - e^{\text{th}}}{p_{t, j}}} \right) \, dt,
    \]
    where $[a, b]$ is the charging location of the agent. The final reward of a macro action is the sum of two reward terms.
\end{itemize}

\newpage
\chapter{ASYNCHRONOUS MULTI-AGENT  PROXIMAL POLICY OPTIMIZATION}
\label{chap:methodology}
\label{chap:proposed}

\section{Asynchronous multi-agent sampling mechanism}

In this section, the thesis presents a learning framework that enables state-of-the-art single-agent and synchronous multi-agent reinforcement learning algorithms to function effectively in an asynchronous multi-agent environment. The key challenge in the asynchronous setting is the sampling mechanism, which is addressed with an asynchronous multi-agent sampling mechanism (Figure~\ref{fig:asynchronous framework}). In this setup, each mobile charger (MC) may have to decide its next macro action while others are still executing theirs. Each MC maintains its own buffer, which includes frames, each containing the observation at the start of the previous macro action $o_{i-1}$, the previous macro action $u_{i-1}$, the reward $r_{i-1}$, and the observation at the end of the previous macro action $o_i$. The training method varies based on the algorithm. In PPO, all agents' buffers are combined into a single buffer for training an actor and a critic. In IPPO, each agent's actor and critic are trained independently using its own buffer. After training, if the finishing condition is not met, the process repeats; otherwise, the final model is obtained. The asynchronous multi-agent sampling mechanism is respective to Line 4 to Line 29 of Algorithm~\ref{al:amappo}.
\begin{figure}[H]
    \centering
    \includegraphics[width=0.95\textwidth]{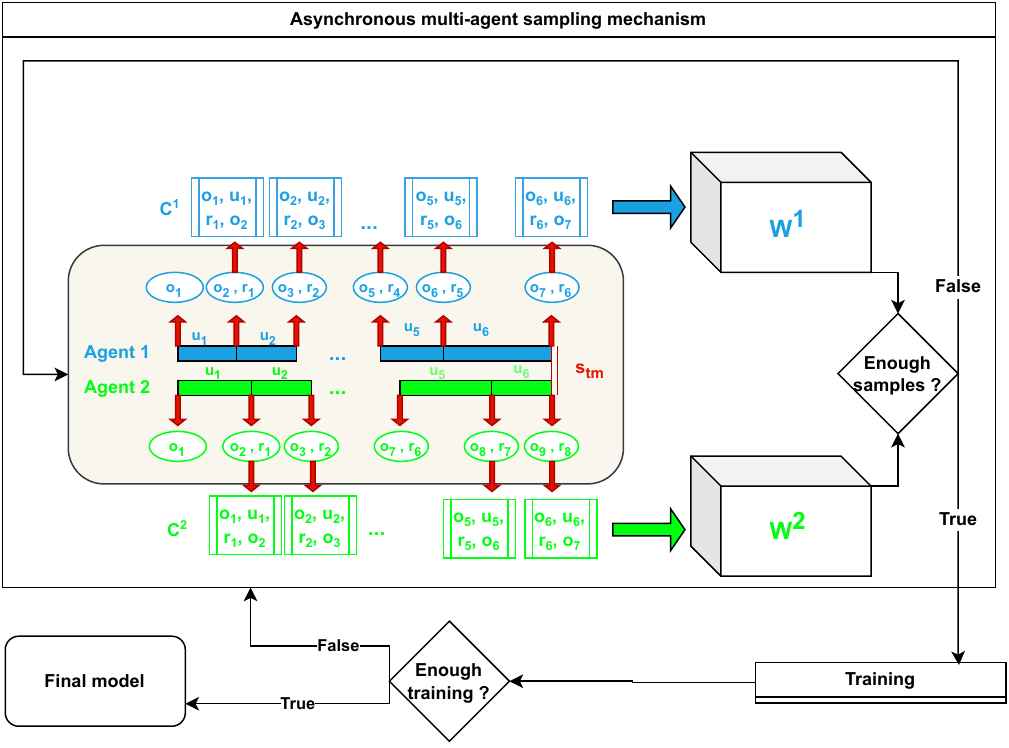}
    \caption{The asynchronous multi-agent sampling mechanism for two agents}
    \label{fig:asynchronous framework}
\end{figure}

\section{Asynchronous multi-agent proximal policy optimization algorithm}

The thesis proposes an algorithm called asynchronous multi-agent proximal policy optimization (AMAPPO) to optimize the high-level policy of MCs. The proposal originates from the proximal policy optimization (PPO) but is adapted to asynchronous multi-agent settings. Macro actions of the $MC_k$ are yielded by its own policy $\Theta^k$ with parameters of $\theta^k$. All MCs incorporate a joint-critic model $V_{\phi}(o)$, where $\phi$ is the set of parameters of the model, and $o$ is an observation of any MC. The $V_{\phi}(o)$ is used for variance reduction and is only utilized during training. The training algorithm is presented in Algorithm~\ref{al:amappo}. 

\begin{algorithm}

Initialize actor models $\Theta^1, \dots, \Theta^M$, and critic model $V_{\phi}$.

\While{$step < step_{max}$\tcp*{Loop through training steps}} {

    Initialize data buffers for actors $W^1, \dots, W^M$.

    \While{Any buffer $W^k$ does not reach a predefined size}{
        Create $M$ empty caches $C^1, \dots, C^M$

        Reset the environment;
        
        \While{$t < t_{sm}$\tcp*{Avoid an infinite loop}}{
        
            \For{$k = 1$ \KwTo $M$ \tcp*{Loop through agents}}{

                $u_{t}^{k} \leftarrow$ agent $k$’s macro action at $t$;

                $H_{t}^{k} \leftarrow$ agent $k$’s action-observation history at $t$;
                
                \uIf {$H_t^k \in B^{k}(u_{t}^{k})$}{
                
                     $t^{\prime} \leftarrow$ the beginning time of $u_{t}^{k}$;

                     $o^k_{t^{\prime}} \leftarrow$ the observation of the agent $k$ at time $t^{\prime}$;

                     $o^k_t \leftarrow$ the observation of the agent $k$ at time $t$;

                     $R \leftarrow$ the reward of $u_{t}^{k}$;

                    $\hat{A} = R + \gamma \times V_{\phi}(o^k_t) - V_{\phi}(o^k_{t^{\prime}})$
                     
                     Append ($o^k_{t^{\prime}}$, $u_{t}^{k}$, $R$, $\hat{A}$, $o^k_t$) to $C^k$;

                    $u_{t}^{k} \leftarrow \Theta^k(o^k_t)$\tcp*{Create a new macro action}               
                }
            }
            Execute atomic action $a_{t}^{k} \sim u_t^{k}(H_{t}^{k}) \forall k \in \{1, 2, \dots, M\}$
        }
        \For{$k = 1$ \KwTo $M$ \tcp*{Loop through agents}}
        {
            Append $C^k$ into the batch $W^k$.  
        }
    } 
    \For{$k = 1$ \KwTo $M$ \tcp*{Stop sampling, start training}}
        {
            Compute loss $J_{\Theta ^ k}(\theta^k)$ following Formula~\eqref{eq:ppo_formula} on the batch $W^k$.  
        }
Compute loss $J_{V_{\phi}}(\phi^k)$ on the batch $W = W^1 \cup W^2 \cup \dots W^M$ based on Formula~\eqref{eq:mse_loss}.

    Update $\phi$, and $\theta^1, \theta^2, \dots, \theta^M$ using Adam and gradient clipping.
}
\caption{AMAPPO}
\label{al:amappo} 
\end{algorithm}

The update of the actor models is identical to the policy improvement step in PPO, which is formally presented in the Formula~\eqref{eq:ppo_formula}:

\begin{equation}
\label{eq:ppo_formula}
J_{\Theta ^ k}(\theta^k) = \frac{1}{|W^k|} \sum_{i=1}^{|W^k|} \min\left( \frac{\pi_{\theta^k}(u_i|o_i)}{\pi_{\theta^k_{\text{old}}}(u_i|o_i)} \hat{A}_i, \text{clip}\left( \frac{\pi_{\theta^k}(u_i|o_i)}{\pi_{\theta^k_{\text{old}}}(u_i|o_i)}, 1 - \epsilon, 1 + \epsilon\right) \hat{A}_i\right),
\end{equation}
where $i$ indexes the samples in the batch,
$\frac{\pi_{\theta^k}(u_i|o_i)}{\pi_{\theta^k_{\text{old}}}(u_i|o_i)}$ is the importance sampling ratio for each sample $i$. Specifically, the $\pi_{\theta^k}(u_i|o_i)$ is the current policy's probability density of choosing the macro action $u_i$ when observing an observation $o_i$. The $\pi_{\theta^k_{\text{old}}}(u_i|o_i)$ is the probability density of choosing the macro action $u_i$ when observing an observation $o_i$ by the policy at the decision-making time. The $\hat{A}_i$ is the Generalized Advantage Estimation (GAE) advantage for each sample $i$ that is calculated for each cache in Line 19 of Algorithm~\ref{al:amappo}. The $\text{clip}(\cdot, a, b)$ is the function that clips the value to be within the range $[a, b]$.

The critic is trained based on the Mean Squared Error (MSE) loss, which is shown in Formula~\eqref{eq:mse_loss}:

\begin{equation}
\label{eq:mse_loss}
J_{V_{\phi}}(\phi^k) = \frac{1}{|W|} \sum_{i=1}^{|W|} r_i + \gamma \times V_{\phi}(o_i^{\prime}) - V_{\phi}(o_i),
\end{equation}
where $i$ indexes the samples in the batch, $V_{\phi}(o_i^{\prime})$ is the evaluation of the critic for the observation when an actor begins the macro action $u_i$, while $V_{\phi}(o_i)$ is the evaluation of the critic for the observation when an actor terminates the macro action $u_i$. $r_i$ is the reward for the macro action $u_i$.

\section{Critic model}
\label{sec: critic_model}
The role of the critic model $V_{\phi}$ is to evaluate the goodness of an MC's observation. Hence, the input of $V_{\phi}$ is an observation of MC whose shape is 4x100x100. The output is a float number to express evaluation for the respective input observation. With these image-like inputs, the CNN architecture has shown its superior performance in many similar tasks such as Dota in \cite{berner2019dota}, Atari in \cite{mnih2013playing}, and so on. The convolutional layers capture spatial features and exploit local patterns in the input, while the fully connected layers capture high-level representations and produce goodness estimates. The Table~\ref{tab:critic-network-layers} provides detailed information about each layer in the design of  $V_{\phi}$ in the proposal.

\begin{table}[h]
\centering
\caption{Summary of Actor-Network Layers}
\label{tab:critic-network-layers}
\begin{tabular}{|c|c|c|c|c|}
\hline
Layer & Type & Input Size & Output Size & Nb parameters \\
\hline
conv1 & Conv & (4, H, W) & (16, H/2, W/2) & 1616 \\
\hline
relu1 & ReLU & (16, H/2, W/2) & (16, H/2, W/2) & 0 \\
\hline
conv2 & Conv & (16, H/2, W/2) & (32, H/4, W/4) & 12832 \\
\hline
relu2 & ReLU & (32, H/4, W/4) & (32, H/4, W/4) & 0 \\
\hline
conv3 & Conv & (32, H/4, W/4) & (64, H/8, W/8) & 51264 \\
\hline
relu3 & ReLU & (64, H/8, W/8) & (64, H/8, W/8) & 0 \\
\hline
flatten & Flatten & (64, H/8, W/8) & (4096) & 0 \\
\hline
fc1 & FC & (4096) & (100) & 409700 \\
\hline
relu4 & ReLU & (100) & (100) & 0 \\
\hline
fc2 & FC & (100) & (1) & 101 \\
\hline
\end{tabular}
\end{table}

The network's input has 4 channels while "H" and "W" represent the height and width. It is worth noting that, Table~\ref{tab:critic-network-layers} indicates the size of the input and output of each layer when passing one observation through the network. In the training process, the batch training strategy is adopted, which means multiple observations are forwarded through the network at the same time. The number of parameters indicates the total trainable parameters in each layer.

\section{Actor model}
\label{sec: actor_model}
\begin{figure}
    \centering
    \includegraphics[width=\linewidth]{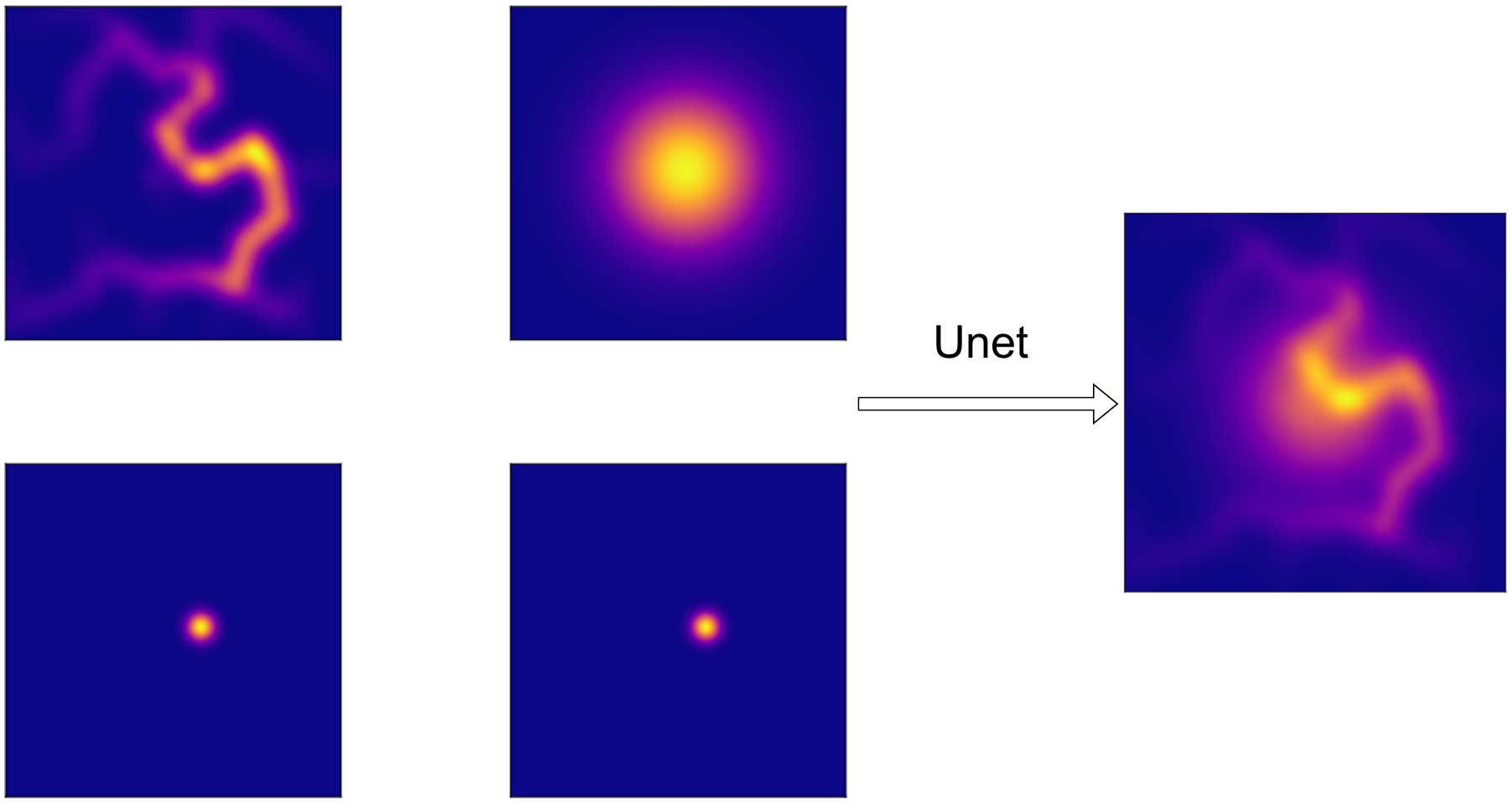}
    \caption{An example of a charging probability map produced from an observation by an Unet}
    \label{fig:input to prob}
\end{figure}
The actor is designed to output a macro action for the MC, which is a 3-component vector. However, the macro action space is too large for the MCs to effectively learn feasible charging locations where at least one sensor can receive energy. To address this challenge, the thesis proposes an alternative approach utilizing the U-net architecture. This approach aims to generate a map of visiting probabilities across the area of interest. An example of a charging probabilities map is illustrated in Figure~\ref{fig:input to prob}. Subsequently, an optimizing algorithm is employed to identify the charging points surrounding the areas with the highest probability. A concrete instance of this optimizing phase is presented in Figure~\ref{fig: optimization}, where the red rectangle indicates the finding region and the purple point is the charging location with the highest charging rate to two nearby sensors. The comprehensive optimization procedure is detailed in the final part of this section. By leveraging the U-net architecture, the proposed method can effectively capture and represent the spatial distribution of feasible charging locations, enhancing the MCs' decision-making process. 

Table~\ref{tab:actor-layers} provides a detailed summary of each layer in the U-Net network, including information about input size, output size, number of parameters, and kernel sizes.

\begin{table}
\centering
\caption{Summary of Network Layers}
\label{tab:actor-layers}
\begin{tabular}{|c|c|c|c|c|}
\hline
Layer & Type & Input Size & Output Size & Params \\ \hline
inc & Conv. Block & (4, H, W) & (64, H, W) & 2368 \\ \hline
down1 & Down. Block & (64, H, W) & (128, H/2, W/2) & 131584 \\ \hline
down2 & Down. Block & (128, H/2, W/2) & (256, H/4, W/4) & 525312 \\ \hline
up1 & Up. Block & \begin{tabular}[c]{@{}c@{}}(384, H/4, W/4), \\ (128, H/2, W/2)\end{tabular} & (128, H/2, W/2) & 675584 \\ \hline
up2 & Up. Block & \begin{tabular}[c]{@{}c@{}}(192, H/2, W/2), \\ (64, H, W)\end{tabular} & (64, H, W) & 270656 \\ \hline
out\_mean & Output Conv. & (64, H, W) & (1, H, W) & 577 \\ \hline
log\_std & Parameter & N/A & (1, H, W) & 1 \\ \hline
\end{tabular}
\end{table}

\begin{figure}[H]
    \begin{minipage}[c]{0.5\linewidth}
        \centering
        \includegraphics[width=\linewidth]{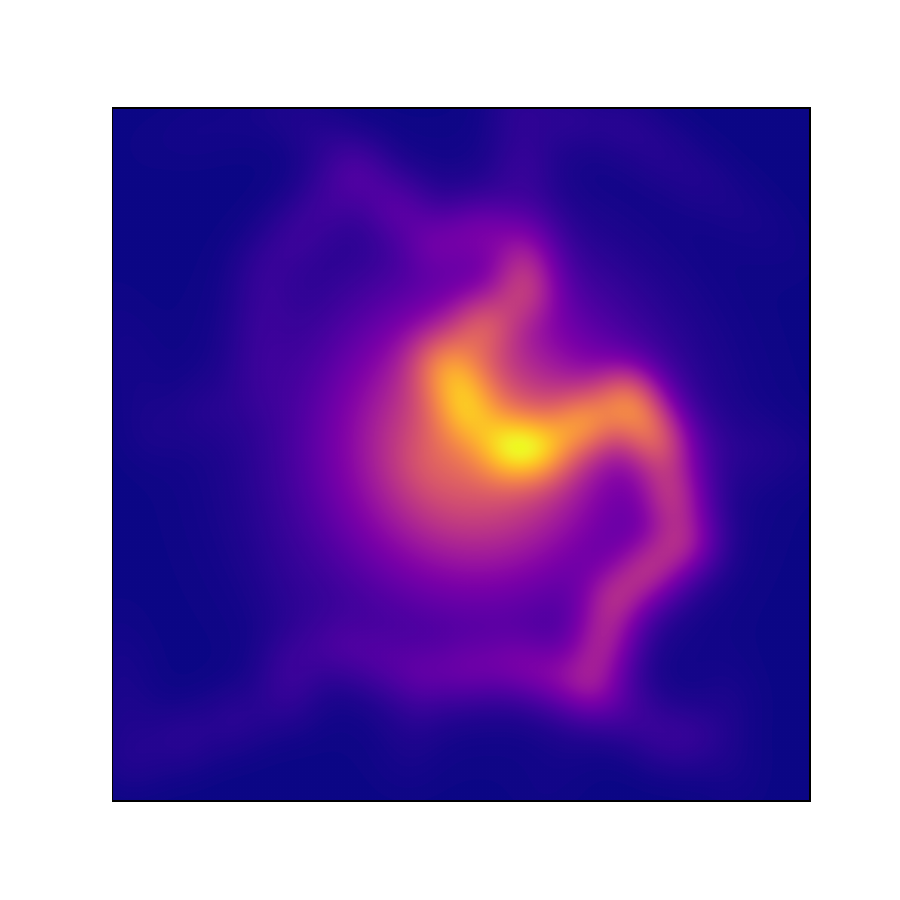}
    \end{minipage}%
    \begin{minipage}[c]{0.5\linewidth}
        \centering
        \includegraphics[scale=0.5]{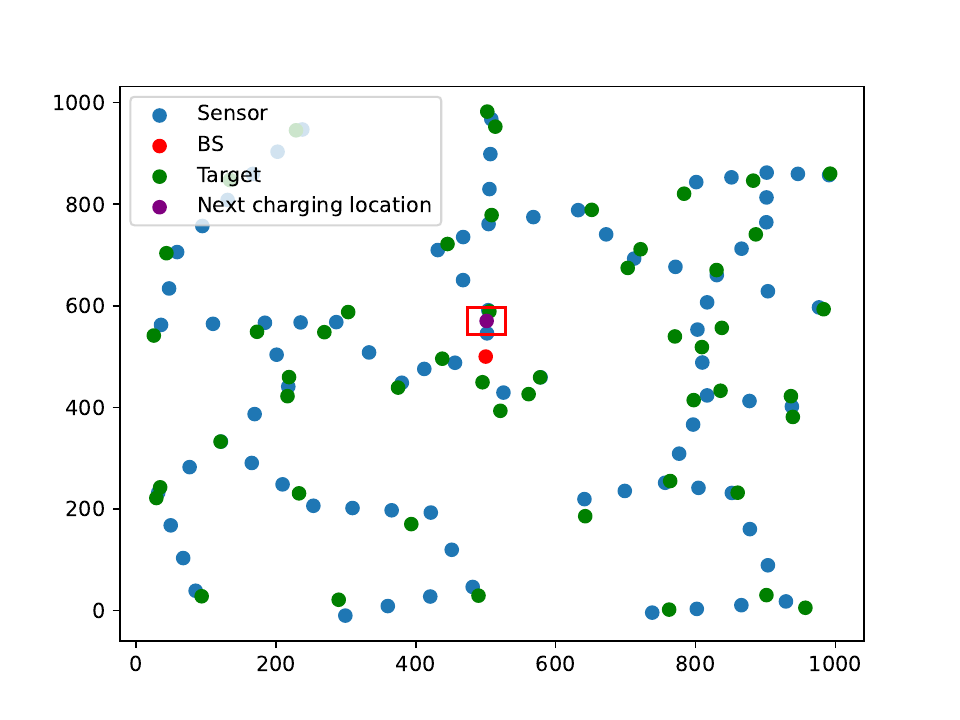}
    \end{minipage}
    \caption{The optimization problem to detect the next macro action from a probability map}
    \label{fig: optimization}
\end{figure}
The input of the actor network is also observations of an MC same as the critic network. The symbols "H" and "W" represent the height and width of the input, respectively. The number of parameters indicates the total trainable parameters in each layer. The kernel size specifies the size of the convolutional kernel used in each layer. 

As previously discussed, the Unet is required to produce a probability map with dimensions HxW, where each element in the map is represented as \(Pr[u][v]\) for convenience. However, this approach lacks the incorporation of the exploration term in action selection. To address this, the thesis employs the Gaussian distribution strategy commonly used in continuous action space reinforcement learning. In this strategy, each element \(Pr[u][v]\) is modeled as a Gaussian distribution with a mean of \(mean[u][v]\) and a standard deviation of \(std[u][v]\). Consequently, the Unet module does not output a single map; rather, it generates two components, \(mean\) and \(log(std)\), both with dimensions HxW.

The final map is then constructed by sampling each element as follows:
\[
Pr[u][v] = \frac{e^{x[u][v]}}{\sum_{u=1}^{H}\sum_{v=1}^{W} e^{x[u][v]}},
\text{where } x[u][v] \sim \mathcal{N}(mean[u][v], std[u][v]^2).
\]
In this approach, each \(x[u][v]\) is drawn from a Gaussian distribution with mean \(mean[u][v]\) and standard deviation \(std[u][v]\). By employing this sampling technique, the Unet can better handle continuous action spaces. Assume that $Pr[u_{\text{max}}][v_{\text{max}}]$ is the maximum element of $Pr$. The charging time for the next macro action is determined by the Formulas:

\begin{equation}
    c = Pr[u_{\text{max}}][v_{\text{max}}] \times \frac{e^{\text{max}} - e^{\text{th}}}{\frac{\alpha}{\beta^2}},
\end{equation}
where $\frac{e^{\text{max}} - e^{\text{th}}}{\frac{\alpha}{\beta^2}}$ is a standard term for the charging time. This formula indicates that if the charging probability is high, the mobile charger should charge for a longer time. Otherwise, it means that there are also other regions with a high probability, and the mobile charger is not certain that its decision is optimal so the mobile charger only should charge sensors for a small time and then make a decision again. The charging location $[a, b]$ is determined by the following procedure.
Consider $H_0, H_1$ as the lower bound and the upper bound of the x-axis location of all sensors while $W_0, W_1$ as these of the y-axis. The finding region defined by the inequalities $H_0 + \frac{(H_1 - H_0) \times (u_{\text{max}} - 0.5)}{\text{H}} \leq a \leq H_0 + \frac{(H_1 - H_0) \times (u_{\text{max}} + 0.5)}{\text{H}}$ and $W_0 + \frac{(W_1 - W_0) \times (v_{\text{max}} - 0.5)}{\text{W}} \leq b \leq W_0 + \frac{(W_1 - W_0) \times (v_{\text{max}} + 0.5)}{\text{W}}$ referred to as region $D$. Region D can be characterized as the subset of the 2D space where the coordinates  $a$ and $b$ lie within the specified bounds, determined by the relationships between the dimensions and the central values. These constraints are essential in the context of the study, providing a well-defined region of interest for a good finding charging location. The optimization problem aimed at maximizing the total charging rate for sensors involves finding the values of \(a\) and \(b\) to optimize the charging strategy. This objective is achieved while considering the estimated remaining lifetime of the sensors, where those with a lower remaining lifetime are given higher weights in the optimization process. The optimization problem can be formally defined as follows:

\[
\max_{a, b} \sum_{j=1}^{N} \left( \frac{{P^{\text{charge}}(d(L^{\text{sensor}}_j, [a, b])}}{\frac{e_{t, j} - e^{\text{th}}}{p_{t, j}}} \right)
\]
subject to:
\[
H_0 + \frac{(H_1 - H_0) \times (u_{\text{max}} - 0.5)}{\text{H}} \leq a \leq H_0 + \frac{(H_1 - H_0) \times (u_{\text{max}} + 0.5)}{\text{H}}
\]

\[
W_0 + \frac{(W_1 - W_0) \times (v_{\text{max}} - 0.5)}{\text{W}} \leq b \leq W_0 + \frac{(W_1 - W_0) \times (v_{\text{max}} + 0.5)}{\text{W}},
\]
where $d(L^{\text{sensor}}_j, [a, b])$ indicate the Eucidean distance between $S_j$ and the location [a, b]. This optimization problem is solved by the L-BFGS-B method initialized in many popular libraries. The details about this method can be found in \cite{byrd1995limited}.

\newpage
\chapter{EXPERIMENTAL RESULTS}
\label{chap:experiments}
This chapter presents ablation studies that highlight the critical role of each component in the proposed Dec-POSMDP formulation and the proposed algorithm, named "\textbf{A}synchronous \textbf{M}ulti-\textbf{A}gent \textbf{P}roximal \textbf{P}olicy \textbf{O}ptimization" (AMAPPO). The numerical results demonstrate the effectiveness and generalizability of the Dec-POSMDP formulation in addressing target coverage and connectivity problems of WRSNs. Furthermore, the comparison to state-of-the-art approaches underscores the superior performance of AMAPPO across various real-world scenarios.

\section{Experimental settings}

The experiments were conducted using the Python programming language and executed on the DGX Station A100 server, which is equipped with 512GB DDR4 RAM and 4 x NVidia A100 GPUs. The Simpy framework, initially introduced in \cite{matloff2008introduction}, was utilized to simulate the network operation, following the mechanism described in \cite{zhu2018adaptive}. The charging model's parameters were inherited from the paper \cite{lyu2019periodic} and are detailed in Table~\ref{tab:model parameter}. The parameters of the proposed AMAPPO algorithm are presented in Table~\ref{tab: gsadql parameter}.

\begin{table}[!ht]
\begin{minipage}[h]{0.45\textwidth}
\caption{Base parameters
of the energy model}
\label{tab:model parameter}
\begin{tabular}{|l|l|}
\hline
\textbf{Parameter} & \textbf{Value} \\ \hline
$r_c$ & 80 (m) \\ \hline
$r_s$ & 40 (m) \\ \hline
$e^{\text{th}}$ & 540 (J) \\ \hline
$e^{\text{max}}$ & 10800 (J) \\ \hline
$E^{\text{max}}$ & 108000 (J) \\ \hline
$r^{\text{charge}}$ & 27 (m) \\ \hline
$P_M$ & 1 (J/m) \\ \hline
$V$ & 5 (m/s) \\ \hline
$\alpha$, $\beta$ & 4500, 30 \\ \hline
\end{tabular}
\end{minipage}
\hfill
\begin{minipage}{0.45\textwidth}
\caption{Parameters of NLM-CTC's Dec-POSMDP model and AMAPPO}
\label{tab: gsadql parameter}
\begin{tabular}{|l|l|}
\hline
\textbf{Parameter} & \textbf{Value} \\ \hline
$\epsilon$ & 0.2 \\ \hline
$\gamma$ & 0.99 \\ \hline
Buffer size & 512 \\ \hline
\begin{tabular}[c]{@{}l@{}}Update iteration\\ per a buffer\end{tabular} & 5 \\ \hline
Learning rate & 0.0003 \\ \hline
$t_{sm}$ & 604800 (s) \\ \hline
$T$ & 100 \\ \hline
\end{tabular}
\end{minipage}
\end{table}

\section{Baselines and performance metric}

 To show off the effectiveness and the generalization of proposals, the thesis considers three baselines as follows:
\begin{itemize}
    \item PPO, which stands for Proximal Policy Optimization algorithm \cite{schulman2017proximal}, is among the most popular and effective techniques for single-agent reinforcement learning. While the original PPO is not originally designed for the multi-agent setting and asynchronous decision-making environments, this thesis has adapted it to suit the Dec-POSMDP formulation. The data collection and parameter update process remains the same as in AMAPPO. However, in the PPO algorithm, only one actor and critic are employed to govern the operations of multiple mobile chargers.
    
    \item IPPO: The Independent Proximal Policy Optimization algorithm \cite{yu2022surprising} is a modified version of PPO specifically designed to address the multi-agent setting. In this algorithm, each mobile charger is treated as an independent learner, with its own actor and critic. Consequently, each mobile charger collects its own samples, which are then utilized to train its respective PPO model. This approach allows for individualized learning and decision-making by each mobile charger
    
    \item DTCM: The Distributed charging scheme under Target coverage and Connectivity constraint for Multiple chargers is a reinforcement learning-based algorithm designed to solve the target coverage and connectivity in WRSNs. The algorithm implements a Q-learning algorithm on their proposed MDP model whose action space is traveling to a predefined charging location. The cooperation between mobile chargers is ensured by a consensus mechanism.

\end{itemize}
The objective of the thesis is to maximize the network lifetime which is calculated through Formula~\ref{eq:lifetime} and the network simulation. However, due to the physical geometry characteristic of each network scenario being different, the thesis  adopts an alternate metric called network lifetime improvement, calculated as follows:

\begin{equation}
     \text{Network lifetime improvement}  =  \dfrac{F_0}{F_B},
\end{equation}
where $F_0$ is the network lifetime when implementing mobile chargers with a certain controlling algorithm, while $F_B$ is the network lifetime without using mobile chargers. With this metric, an inferior algorithm would result in the network lifetime improvement of 1 while a good algorithm aims to increase the improvement as high as possible.

\section{Dataset}
\begin{figure}
    \centering
    \includegraphics[scale=0.65]{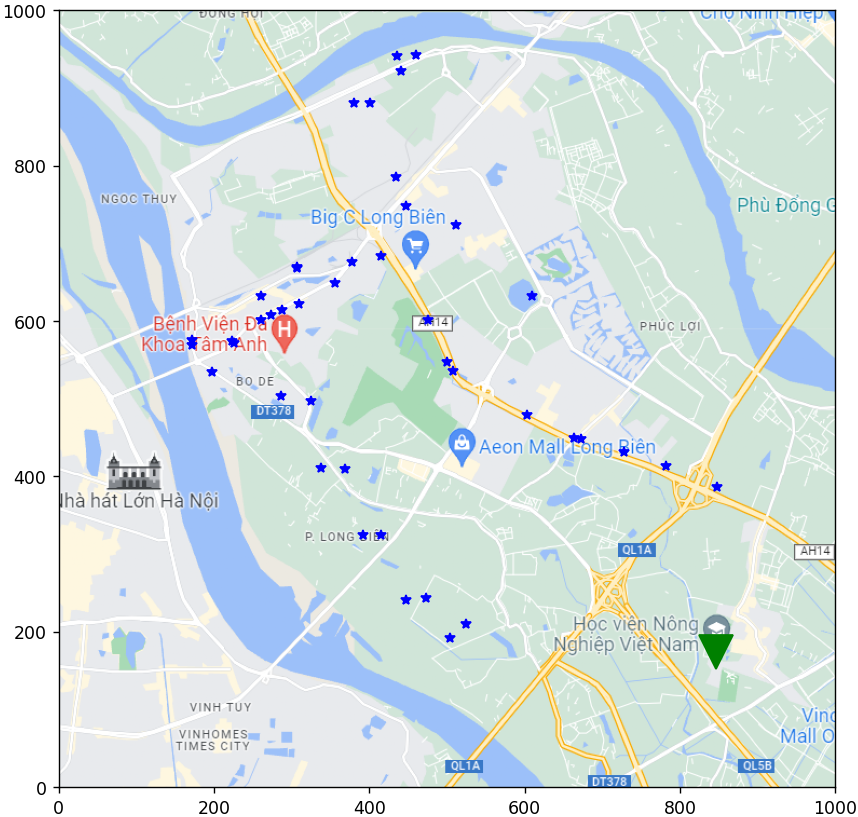}
    \caption{The instance "hanoi\_50"}
    \label{fig: hanoi_50_target}
\end{figure}
\begin{figure}
    \centering
    \includegraphics[scale=0.28]{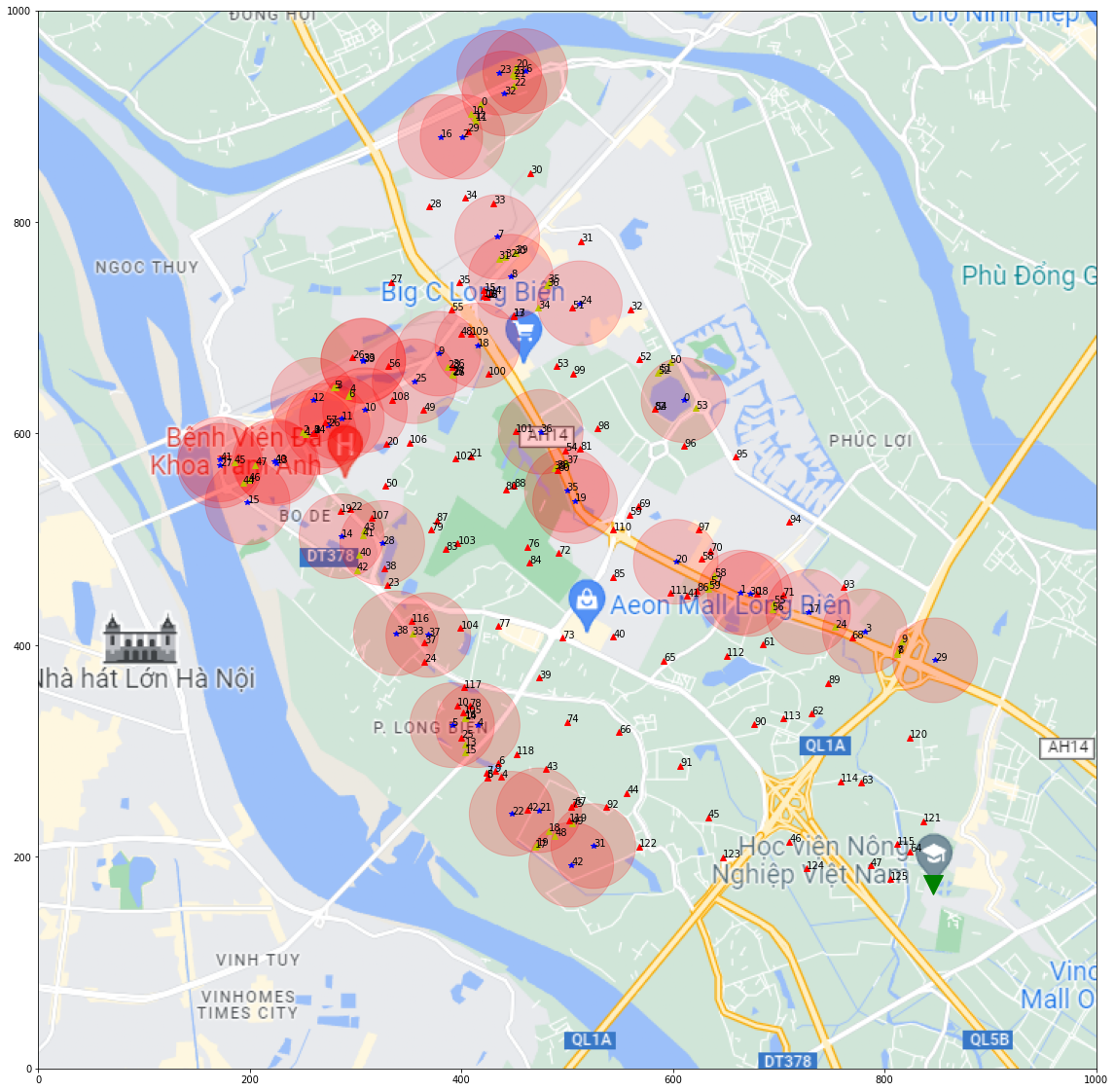}
    \caption{The instance "hanoi\_50" with sensor placement}
    \label{fig: hanoi_50_target_sensor_placement}
\end{figure}
The thesis introduces a real-world dataset comprising 9 instances, derived from the geographical information of Vietnam. Each instance consists of the location of a base station and a set of targets. The base station serves as an information center, while the targets are strategically placed objects of interest. The instances are named following a convention of "province\_no. targets", where "province" belongs to the set \{hanoi, bacgiang, sonla\}, representing three provinces in Vietnam where the network is implemented. The "no. targets" taken from the set \{50, 100, 150, 200\}, denotes the number of targets to be monitored in each instance. For example, the instance "hanoi\_50" represents a scenario where the base station is situated at the Vietnam National University of Agriculture, and there are 50 targets that serve as transportation stations, illustrated in Figure~\ref{fig: hanoi_50_target}.

To ensure effective target monitoring and connection to the base station, sensors are strategically positioned within each instance, following the approach outlined in the paper \cite{HANH2023103578}. Each instance comprises the locations of sensors, targets, and the base station. The instance "hanoi\_50" with the sensor placement is illustrated in Figure~\ref{fig: hanoi_50_target_sensor_placement}. Throughout the training and testing phases, the number of mobile chargers is fixed at 3.

The training of AMAPPO and other reinforcement learning algorithms is conducted using the "hanoi\_50" instance in a simulation environment, while the other instances serve as evaluation maps to assess the testing performance.

\section{Ablation study}

To showcase the importance of each component in the Dec-POSMDP formulation and the AMAPPO algorithm, this section provides a comprehensive set of ablation experiments. These experiments involve various ablation variants of the proposed methods, which are carefully designed to isolate the effects of different components. The performance of all ablation variants is summarized in Table~\ref{tab: ablation1}. Additionally, in the following subsections, we provide detailed insights into the performance of each group of variants throughout the training progress. This analysis highlights the significance of different components in achieving efficient and cooperative charging strategies for the MCs, ultimately improving the network lifetime and overall system performance.

\begin{table}
\centering
\begin{tabular}{|c|cccc|c|}
\hline
\multirow{2}{*}{\begin{tabular}[c]{@{}c@{}}Dec-POSMDP\\ type\end{tabular}} & \multicolumn{4}{c|}{The number of targets} & \multirow{2}{*}{Overall} \\ \cline{2-5}
 & \multicolumn{1}{c|}{50} & \multicolumn{1}{c|}{100} & \multicolumn{1}{c|}{150} & 200 &  \\ \hline
FULL & \multicolumn{1}{c|}{\textbf{3.13}} & \multicolumn{1}{c|}{\textbf{2.35}} & \multicolumn{1}{c|}{\textbf{1.69}} & \textbf{1.43} & \textbf{2.15} \\ \hline
NO\_1 & \multicolumn{1}{c|}{1.05} & \multicolumn{1}{c|}{1.07} & \multicolumn{1}{c|}{1.02} & 1.00 & 1.04 \\ \hline
NO\_2\_3\_4 & \multicolumn{1}{c|}{2.02} & \multicolumn{1}{c|}{1.51} & \multicolumn{1}{c|}{1.23} & 1.14 & 1.48 \\ \hline
NO\_EX & \multicolumn{1}{c|}{2.19} & \multicolumn{1}{c|}{1.72} & \multicolumn{1}{c|}{1.34} & 1.21 & 1.56 \\ \hline
NO\_GE & \multicolumn{1}{c|}{2.24} & \multicolumn{1}{c|}{1.82} & \multicolumn{1}{c|}{1.22} & 1.04 & 1.58 \\ \hline
NO\_PM & \multicolumn{1}{c|}{1.13} & \multicolumn{1}{c|}{1.05} & \multicolumn{1}{c|}{1.06} & 1.00 & 1.06 \\ \hline
\end{tabular}
\caption{The network lifetime improvement on training maps of AMAPPO trained on different variants of the Dec-POSMDP model}
\label{tab: ablation1}
\end{table}
\subsection{The importance of observation space's components}

\begin{figure}[H]
    \centering
    \includegraphics[width= \linewidth]{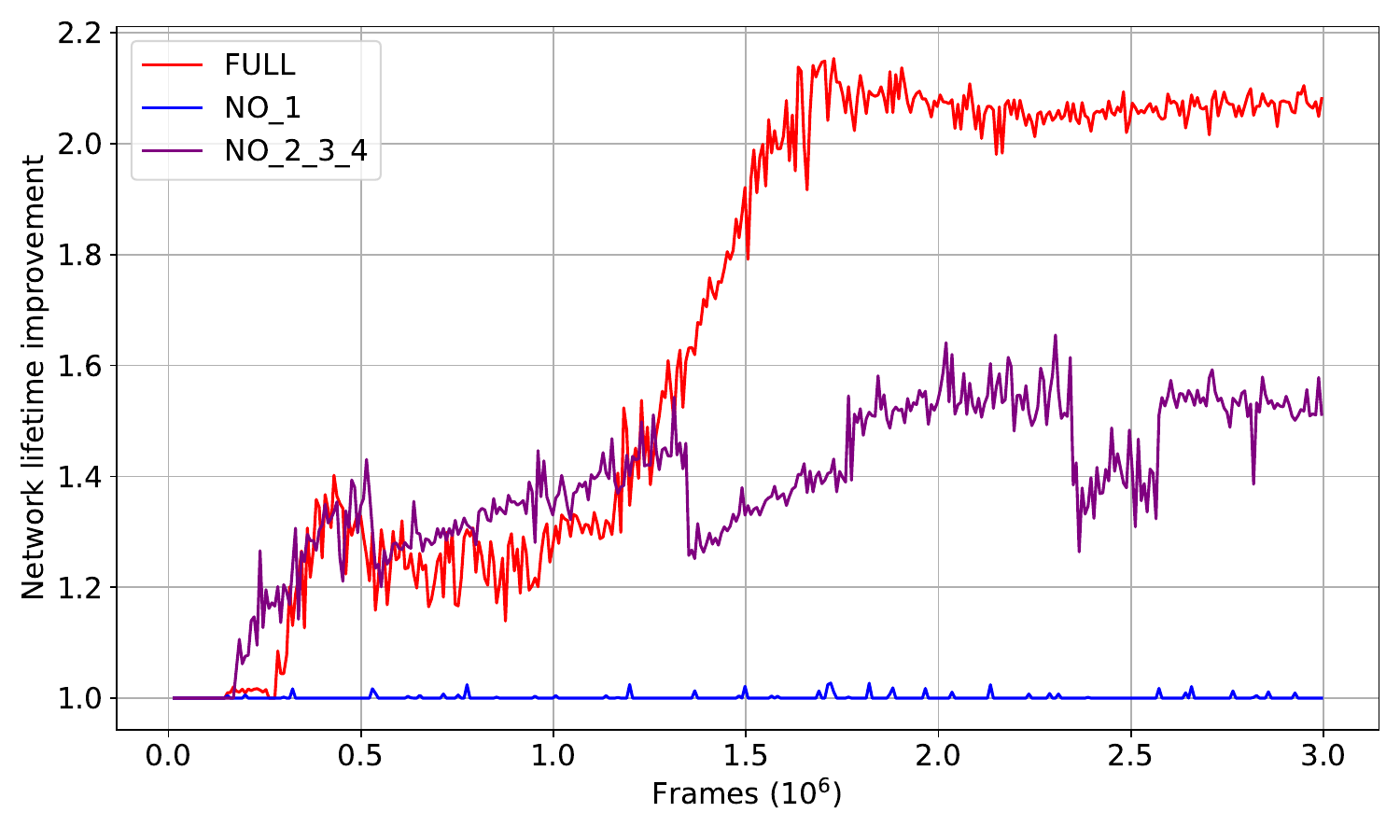}
    \caption{Ablation study on components of observation space}
    \label{fig:ob_ablation_study}
\end{figure}

This experiment highlights the significance of different components in the observation space on the learning process and performance of mobile chargers. The experiment evaluates three variants of the observation space: "FULL," which represents the proposed observation space in the thesis, "NO\_1" excluding the first layer, and "NO\_2\_3\_4" containing only the first layer as the observation space. The network lifetime improvement for these variants over the training phase is illustrated in Figure~\ref{fig:ob_ablation_study}. Overall, the AMAPPO trained on the "FULL" variant outperforms that trained on the "NO\_2\_3\_4" and "NO\_1" variants. 

The first layer of the observation space places emphasis on areas with multiple sensors in critical conditions, particularly those with low levels of energy and high consumption rates. Sensors consuming substantial energy are responsible for numerous tasks in the network. Hence, when these sensors cease operation, it often leads to the disconnection of important network components. Consequently, removing the first layer significantly hampers the learning process, as mobile chargers are unable to recognize critical network conditions. This is evident from the fixed network lifetime improvement around 1.0, indicating no enhancement during training.

In the "NO\_2\_3\_4" variant, where the information about the status of mobile chargers in the second, third, and fourth layers is ignored, the learning process still shows some positive results. However, the performance is much less competitive compared to the full version. Moreover, this variant exhibits oscillations and multiple sudden decreases during training. This phenomenon can be attributed to the lack of information about other mobile chargers in the decision-making process. Specifically, the first layer still guides mobile chargers to critical areas, primarily contributing to the enhancement of network lifetime. However, since mobile chargers are unaware of each other's intentions, they might inadvertently visit the same area, resulting in insufficient energy replenishment for other critical areas. This demonstrates that incorporating the information of other mobile chargers through the three layers facilitates cooperation between mobile chargers, leading to more efficient network lifetime improvement.

\subsection{The impact of reward components}
This section presents a detailed analysis of the general reward and the exclusive reward of mobile chargers (MCs). We implement three variants to compare the separate impact of each reward term: the "FULL" variant represents the proposed combined reward function, the "NO\_EX" variant indicates the reward function without the exclusive term, and the "NO\_GE" variant represents the version excluding the general term. The change in network lifetime improvement for these three variants during the training progress is illustrated in Figure~\ref{fig:ab_reward}.

Overall, the performance of all three variants improves during training. However, the "FULL" and "NO\_GE" variants demonstrate better performance when considering the cutoff at three million training frames. The "NO\_GE" version achieves a comparable enhancement in network lifetime to the full version, despite lacking the general term in its reward function. This could be attributed to the fact that the exclusive term already captures a portion of the general term when distributing the total exclusive rewards among MCs charging the same sensors. This encourages MCs to avoid visiting the proximity of others, leading to efficient resource utilization. However, the "NO\_GE" variant does not as competitive in the testing instances as the "FULL" variants. This highlights the vital role of the general reward components in creating a general multi-agent system.

On the other hand, when removing the exclusive term in the "NO\_EX" variant, the phenomenon of "lazy agents" becomes apparent, which is a common occurrence in multi-agent systems with joint reward functions. This phenomenon arises from the overvaluation of certain actions by one or more agents, even though the overall benefit is contributed by other agents. As a result, this leads to high oscillations in performance during training, hindering cooperation and overall performance improvement.

\begin{figure}[!ht]
    \centering
    \includegraphics[width = \linewidth]{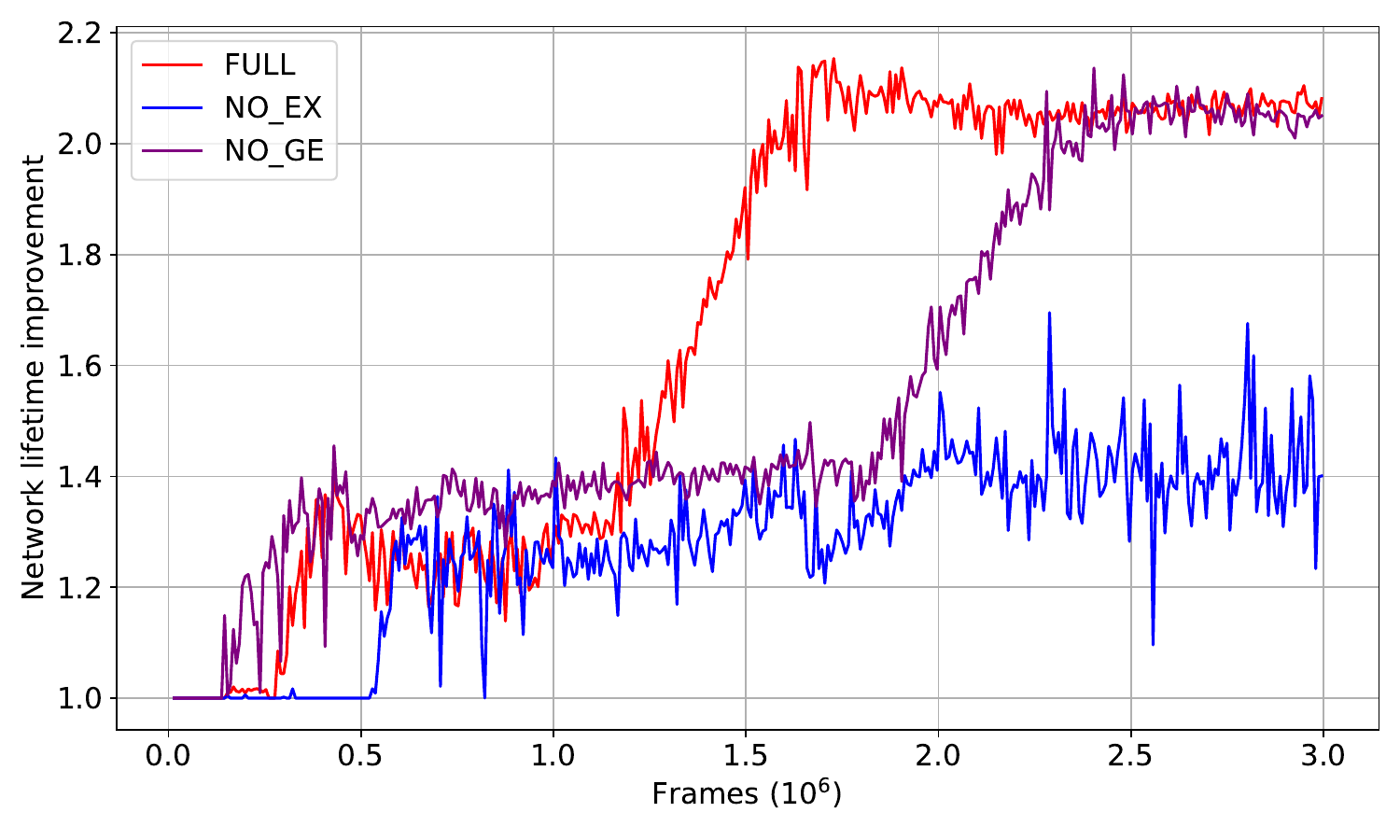}
    \caption{Ablation study on reward components}
    \label{fig:ab_reward}
\end{figure}

\subsection{The necessity of combining charging probability map and optimization procedure}

\begin{figure}[H]
    \centering
    \includegraphics[width=\linewidth]{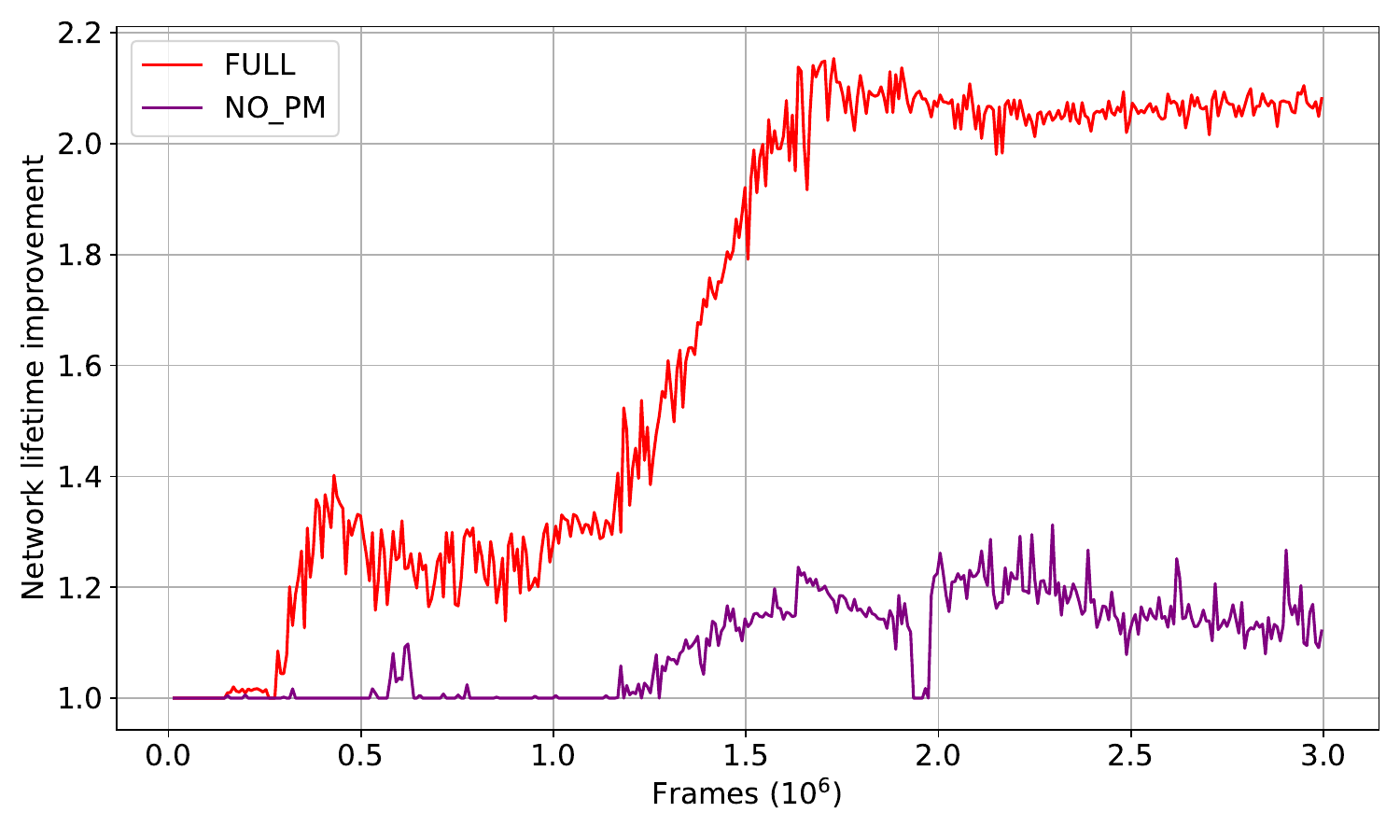}
    \caption{Comparing two action-making strategies}
    \label{fig: comparing}
\end{figure}

In the design of the actor model in Section~\ref{sec: actor_model}, the thesis proposes a U-net model to output a charging probability map, then applies an optimization algorithm to find the charging location. To highlight the vital role of this component, this experiment adopts a variant of AMAPPO named NO\_PM which substitutes the U-net with a CNN model with the same architecture as the critic one, but the output layer is a 3-component vector respective to a macro action. The network life improvement on the training map is presented in Figure~\ref{fig: comparing}. In essence, the space of the charging location is large and includes multiple no-reward signal areas. Hence, it is demanding to train agents to make decisions directly in this space. The exploration just brings agents to a proximal area that is impossible to escape from the no-reward signal area. This challenge is evident in the training progress of NO\_PM and its inferior performance on both the training map and testing maps.

\section{Comparison to baselines}

\subsection{Training phase}
In this experiment, the thesis aims to clarify the vital role of the architecture of AMAPPO through comparison with two other architectures, including IPPO and PPO. Regarding IPPO, this algorithm considers each mobile charger as an independent learner, so implementing an actor and a critic for each mobile charger. The PPO algorithm adopts only one actor and critic to regulate the operation of multiple mobile chargers. For a fair comparison, all three algorithms utilize the same architecture of actor in Section~\ref{sec: actor_model}, critic in Section~\ref{sec: critic_model}, and also the learning process in Algorithm~\ref{al:amappo}. The only differences are the weight update and action-making that depend on which actor and critic is responsible for each actor. 

Overall, the proposed architecture showcases a better performance than the two counterparts. Specifically, after three million training frames, the final models of AMAPPO show a higher gap of 0.21 and 0.59 in the network lifetime improvement than IPPO, and PPO, respectively. It is worth noting that the IPPO algorithm performance is unstable during training which is confirmed by multiple sudden drops in network lifetime improvement and extremely high oscillation. This would be explained by the architecture separating critics, which slower the training of the critic models, and even leads to the significant difference between should-be-the-same critic evaluation. The slow update of the critic model is evident when actors reach some highly better policies. In this context, the critic model is not prompt to shift the parameters to have a good evaluation for these better behaviors that have not been seen before. 

\begin{figure}[H]
    \centering
    \includegraphics[width=\linewidth]{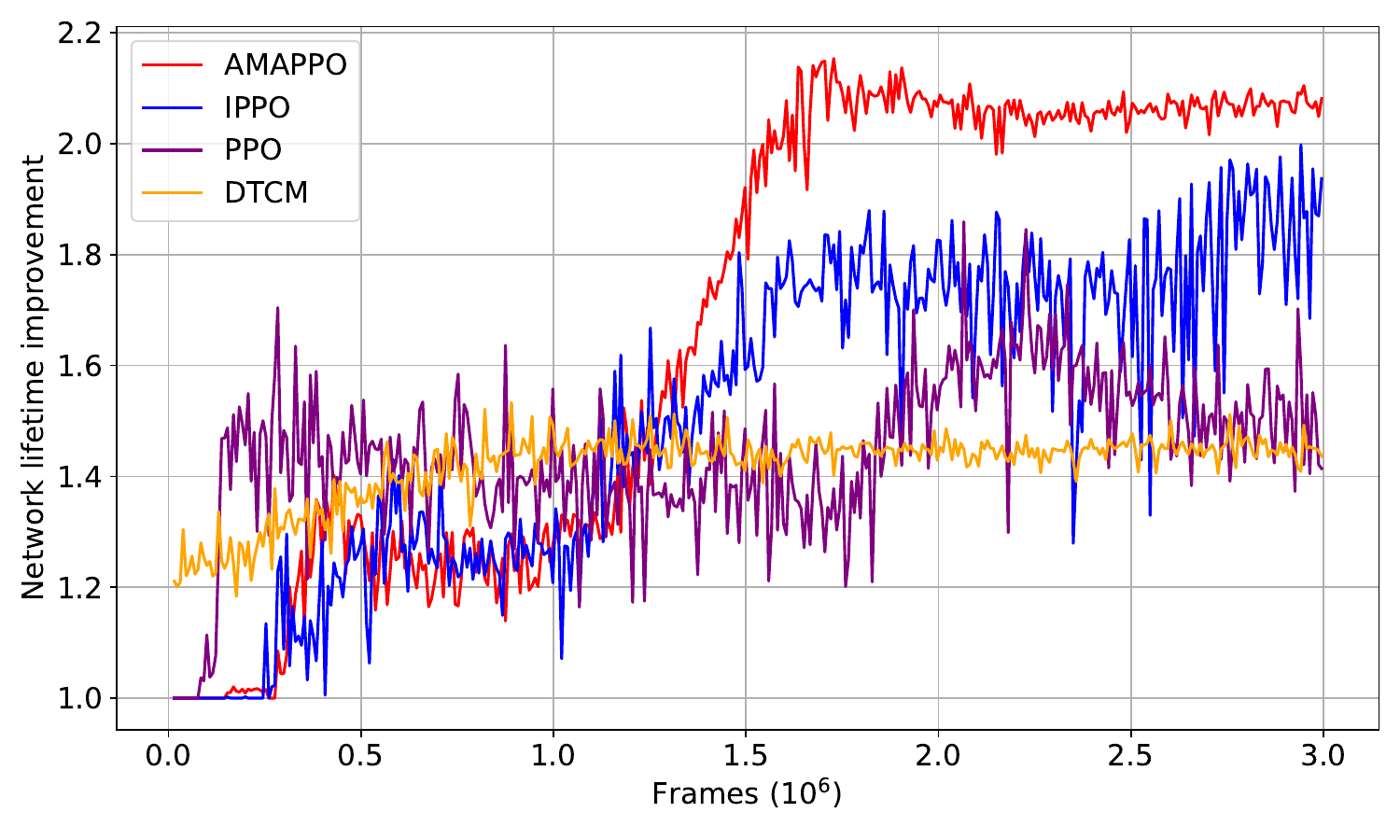}
    \caption{Comparison between algorithms on network lifetime improvement}
    \label{fig:ab_ac_critc}
\end{figure}

\begin{figure}[H]
    \centering
    \includegraphics[width =    \linewidth]{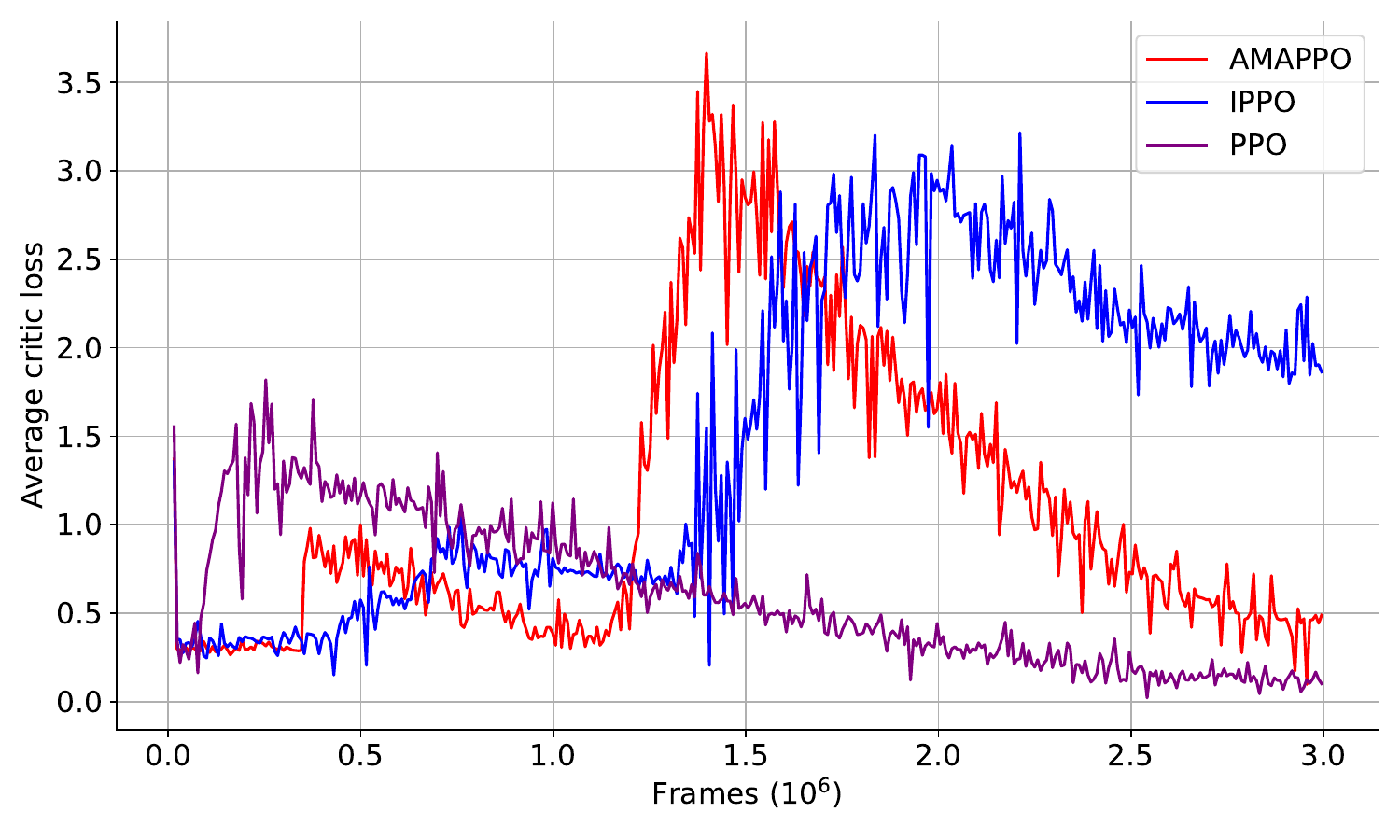}
    \caption{The average critic loss function of actor-critic algorithms}
    \label{fig:critic loss}
\end{figure}

\subsection{Testing phase}

This experiment aims to evaluate the network lifetime improvement of pre-trained models on testing instances. To enhance clarity, instances with the same number of targets are grouped together, and the average value is computed for each group. The results for each model group are presented in Figure~\ref{fig: testing_performance}. As a general trend, the network lifetime improvement shows a decreasing pattern with an increase in the number of targets. This phenomenon is expected, as a higher number of targets necessitates more sensors for surveillance and data transmission, leading to higher energy demand on the network. Additionally, the increased number of targets poses a greater risk of losing connection to one or more targets, further contributing to the decrease in network lifetime improvement.

The DTCM algorithm stands out as the only one using a distinct formulation, while IPPO, PPO, and AMAPPO all rely on the proposed DEC-POSMDP formulation. Due to the specific nature of the DTCM formulation, it cannot be directly applied to different instances apart from the training instance. Therefore, for the testing phase, it becomes crucial to retrain DTCM on a new formulation specific to each instance. In contrast, the proposed DEC-POSMDP formulation exhibits generalization capabilities, enabling pre-trained models of AMAPPO, IPPO, and PPO to effectively operate in various network instances, even if they were trained on a different instance. Moreover, the performance of AMAPPO and IPPO is still much more remarkable in testing instances. Specifically, AMAPPO can averagely lengthen the network lifetime up to more than three times for 50-target instances, and 2.15 times in overall. The performance of IPPO reaches the rate of 1.85 in overall. As mentioned above, DTCM must retrain its model, which likely causes the death of the network before the model can learn something significant, which is indicated by its inferior network lifetime improvement of 1.1 overall.

To highlight the performance of AMAPPO and IPPO, an additional model of DTCM, named DTCM(1M), is trained for each testing instance in 1 million frames. Indeed, this version of DTCM has a more favorable condition to reach an excellent performance due to its already access to the data from testing instances. However, this version only has performance better than the pre-trained PPO model but is less competitive with AMAPPO and IPPO, which do not know any information about testing instances. This comparison showcases the significant generalization and effectiveness of the proposed DEC-POSMDP formulation. Regarding the proposed algorithm, AMAPPO outperforms its counterparts in all instances, showing a network lifetime improvement up to 3.13, 2.35, 1.69, and 1.43 for instances with 50, 100, 150, and 200 targets, respectively.

\begin{figure}[H]
    \centering
    \includegraphics[width =\linewidth]{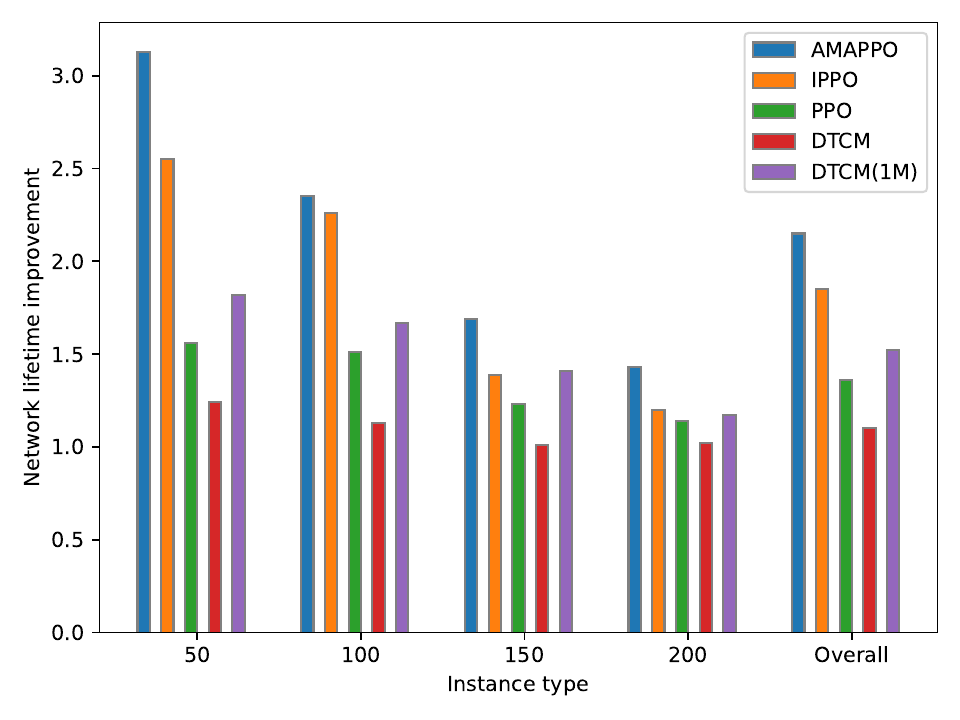}
    \caption{The network lifetime improvement of algorithms on testing instances}
    \label{fig: testing_performance}
\end{figure}

\newpage
\chapter*{CONCLUSION}
\phantomsection\addcontentsline{toc}{chapter}{CONCLUSION}
\label{chap:conclusions}
\noindent \textbf{Summary}

The thesis focuses on maximizing the network lifetime for connected target coverage (NTM-CTC) in wireless rechargeable sensor networks (WRSNs). It introduces a fleet of mobile chargers with a multi-node charging model to achieve this goal. A key contribution is the development of a general Dec-POSMDP formulation for NTM-CTC, allowing mobile chargers to perform effectively without the need for retraining across diverse network scenarios. To address the synchronous nature of the Dec-POSMDP formulation, an asynchronous multi-agent learning framework is proposed, enabling the application of state-of-the-art reinforcement learning algorithms. The final proposed algorithm, AMAPPO, is a modified version of the PPO algorithm tailored to the Dec-POSMDP formulation of NTM-CTC.

To evaluate the performance, a novel training-testing procedure is introduced using a dataset with Vietnam's geometry information, representing various real-world scenarios for NTM-CTC. A series of ablation studies are conducted to demonstrate the individual contributions of each component in the proposed approaches. The numerical results highlight the superior performance of AMAPPO compared to other methods.

In conclusion, the thesis paves the way for a new research direction, focusing on the creation of general agents that accumulate knowledge over time and can adapt to various environments. The proposed framework and algorithm showcase promising advancements in maximizing network lifetime and target coverage in wireless sensor networks.

\noindent \textbf{Limitation}

Despite its significant contributions, the thesis does face two primary limitations. Firstly, there is a lack of investigation into the system's flexibility when the number of mobile chargers is dynamically increased or decreased. Understanding how the system adapts to changes in the number of agents is crucial for real-world applicability. Secondly, the proposal's robustness against environmental uncertainties, such as anomaly events impacting target behavior and sensor energy burden, remains unexplored. Evaluating the proposal's performance under varying and uncertain conditions is important for assessing its practicality and reliability.

\noindent \textbf{Future work}

For future work, the thesis aims to develop an extension mechanism for the multi-agent system, allowing seamless integration of new agents with reasonable parameters and proper regularization of existing agents when agents are removed. This mechanism is vital in real-world scenarios to ensure continuous performance despite changes in the agent composition. Additionally, the thesis plans to conduct experiments to analyze the proposal's robustness under different levels of environmental randomness. Understanding how the proposal performs in uncertain and dynamic environments will provide valuable insights into its applicability and potential areas of improvement. Addressing these limitations and exploring these future directions will further enhance the effectiveness and practicality of the proposed approach in real-world applications.

\noindent \textbf{Publication}

The thesis inherits the network model, the charging model, and the NLM-CTC problem from two author's publications:

\begin{itemize}
    
    \item \textbf{Long, N. T., Huong, T. T., Bao, N. N., Binh, H. T. T., Le Nguyen, P., \& Nguyen, K. (2023). Q-learning-based distributed multi-charging algorithm for large-scale WRSNs in Nonlinear Theory and Its Applications, IEICE, 14(1), 18-34.} 

    \item \textbf{Huong, T. T., Bao, N. N., Hai, N. M., \& Binh, H. T. T. (2021, June). Effective partial charging scheme for minimizing the energy depletion and charging cost in wireless rechargeable sensor networks in 2021 IEEE Congress on Evolutionary Computation (CEC) (pp. 217-224), rank B.}
\end{itemize}

Techniques used in the thesis such as the optimization method and the Gaussian function to model the observation space are also motivated by two publications of the thesis's author, including:
\begin{itemize}
    \item \textbf{Van Cuong, L., Bao, N. N., Phuong, N. K., \& Binh, H. T. T. (2022, July). Dynamic perturbation for population diversity management in differential evolution. In Proceedings of the Genetic and Evolutionary Computation Conference Companion (GECCO) (pp. 391-394), rank A}
    
\end{itemize}

\newpage

\renewcommand\bibname{REFERENCES}


\printbibliography
\phantomsection\addcontentsline{toc}{chapter}{REFERENCES}

\end{document}